\documentclass{article} 
\PassOptionsToPackage{table}{xcolor}
\usepackage{iclr2027_conference,times}

\usepackage{hyperref}
\hypersetup{hidelinks}
\usepackage{url}

\usepackage{amsmath,amssymb,bm}
\usepackage{graphicx}
\usepackage{booktabs}
\usepackage{multirow}
\usepackage{longtable}
\usepackage{xcolor}
\usepackage{subcaption}
\usepackage{enumitem}
\usepackage{amsthm}
\usepackage{float}
\newtheorem{proposition}{Proposition}
\graphicspath{{./figures/}}

\newcommand{\decisive}{\textsc{decisive}}
\newcommand{\directional}{\textsc{directional}}
\newcommand{\boundarylbl}{\textsc{boundary}}
\newcommand{\why}{} 

\title{Stable and Counterfactually Robust Physical World Models from Imposed Structure and Learned Physics}

\author{Yufeng Wang \thanks{These authors contributed equally to this work}\\
  Stony Brook University \\
  \And
  Parivesh Priye $^*$\\
  Georgia Institute of Technology \\
  \And
  Lu Wei $^*$\\
  Stony Brook University \\
  \And
  Haibin Ling \\
Westlake University \\}

\iclrfinalcopy
\makeatletter
\g@addto@macro\@maketitle{\lhead{}}
\makeatother

\begin{document}

\maketitle

\begin{abstract}
  A world model learns to forecast how a physical system evolves from recorded trajectories, yet the systems it imitates obey physical laws that are neither fully supplied nor reliably respected. The model may create energy, drift or diverge over long rollouts, and answer a changed-law query using the law observed during training. We ask how much general physical structure must be hard-coded into a world model, and how much system-specific physics can then be learned from data, for four properties to hold simultaneously: second-law-compatible dissipation, correct responses to interventions on physical parameters, stability out to one hundred times the training horizon, and robustness to disturbances. The imposed structure is general: dynamics are generated from the gradient of a learned energy through a fixed reversible operator, the energy is restricted to a confining class, a one-way port can remove energy but never inject it, the drive channel is known, and the intervened parameter enters through a separable map. The model learns the energy functional, constitutive relations, dissipation rate, and couplings. Across an electromagnetic cavity, a particle-in-cell grid, and a shallow-water fluid, models with roughly nine thousand parameters recover constitutive functions with unit slope, separate conserving from dissipating worlds by four orders of magnitude using a single set of weights, and transfer changes in sign, magnitude, rate, and gravity to unseen values, where equal-capacity models without the same structure perform at chance or worse. A nonlinear constitutive law is recovered with its curvature preserved and predicts a held-out intervention $2$--$17\times$ better than a converged linear model. The same structured model also refuses interventions for which no parameter map was supplied, producing errors $7\times$ to $10^{9}\times$ larger than in-distribution, while oracle models given the missing map succeed; its one-way port likewise cannot represent anti-dissipation by construction.
\end{abstract}

\section{Introduction}
\label{iclr:intro}

A world model is a learned simulator: it predicts the next state of a system and iterates that prediction into a trajectory. Such models now achieve impressive fidelity on their training distributions, from graph-network and mesh-based simulators
\citep{sanchezgonzalez2020gns,pfaff2021meshgraphnets} and neural operators
\citep{li2021fno} to video-scale world models that plan in latent space
\citep{assran2025vjepa2,nam2026causaljepa}. Fidelity alone, however, does not establish that a model obeys the laws of the world it simulates. Physical-realism benchmarks show that even strong video models violate conservation, create energy, and fail on ``anti-physics'' probes
\citep{gu2026phyworldbench,bansal2025videophy2}. The same failures recur across scales: rollouts drift or diverge beyond the training horizon, energy appears where the true system dissipates it, and a changed-law query is answered using the law observed during training.

A natural remedy is to encode physical structure directly into the model. Hamiltonian and Lagrangian neural networks
\citep{greydanus2019hnn,cranmer2020lnn}, their control and generalised-coordinate variants
\citep{zhong2020symoden,lutter2019delan}, dissipative and port-Hamiltonian extensions
\citep{zhong2020dissipative,sosanya2022dhnn,desai2021phnn}, compositional port-Hamiltonian learning
\citep{neary2023compositional}, and GENERIC or metriplectic networks
\citep{hernandez2021,lee2021gfinn,gruber2023} conserve or dissipate a learned energy, and benchmarks show the numerical integrator can matter as much as the model form
\citep{zhong2021benchmark}. These methods nevertheless fit one system at a time, do not determine from data whether dissipation is present, give no rule for which interventions are answerable, state guarantees for continuous-time flows rather than the implemented discrete update
\citep{hairer2006gni}, and do not separate what the modeller supplies from what the data identify. Symbolic regression instead recovers equations directly
\citep{brunton2016sindy,messenger2021weak} but requires a candidate dictionary and reliable derivatives. We therefore ask how much general physics must be \emph{imposed} on a learned world model, and how much system-specific physics can then be \emph{learned} from data, for four behaviours to hold simultaneously: second-law-compatible dissipation, correct responses to interventions on physical parameters, stability to one hundred times the training horizon, and robustness to disturbances. We impose a general form and learn the specific content, testing this division across an electromagnetic ring cavity, a particle-in-cell grid, and a shallow-water fluid, against equal-capacity baselines and controls that remove one structural element at a time.

This paper makes four contributions. First, a single model in each regime determines from trajectories whether its world dissipates and recovers constitutive functions within the supplied structural form, while making explicit what is identifiable and how estimation error propagates into downstream predictions. Second, we attribute each of the four target behaviours to the structural component whose removal destroys it, and formalise the two guarantees provided by the framework, one for the dissipative port and one for form selection, with their precise scope. Third, we give a simple criterion for when an intervention on a physical parameter is answerable by construction, and show that the same models fail, as predicted, on interventions outside that criterion. Fourth, supporting studies test whether these conclusions persist under increased scale, observation noise, and a broader collection of everyday physical systems.

\section{Related work}
\label{iclr:related}

\textbf{Structure-preserving learned dynamics.} Hamiltonian neural networks
\citep{greydanus2019hnn} learn a scalar energy whose symplectic gradient defines the
vector field, while Lagrangian networks \citep{cranmer2020lnn} learn the corresponding
dynamics through a Lagrangian formulation. Symplectic ODE-Net
\citep{zhong2020symoden} and Deep Lagrangian Networks
\citep{lutter2019delan} extend these ideas to controlled systems and generalised
coordinates, and differentiable contact models further incorporate impacts
\citep{zhong2021contact}. Dissipative SymODEN
\citep{zhong2020dissipative}, Dissipative HNNs
\citep{sosanya2022dhnn}, and port-Hamiltonian networks
\citep{desai2021phnn} introduce irreversible dynamics, while compositional
port-Hamiltonian learning assembles larger systems from learned components
\citep{neary2023compositional}. GENERIC and metriplectic networks impose the
reversible/irreversible decomposition together with its degeneracy conditions
\citep{hernandez2021,lee2021gfinn,gruber2023}, and benchmark studies show that the
numerical integrator and evaluation horizon can matter as much as the model form
\citep{zhong2021benchmark,hairer2006gni}. These approaches generally fit one system
at a time. When a dissipative component is present, its sign is either fixed by the
model class or left unconstrained, and none jointly determines from data whether
dissipation is present, carries a continuous-time guarantee to the implemented
discrete update, and explicitly separates what the modeller supplies from what the
data identify (App.~\ref{app:related}, Table~\ref{tab:related-compare}).

\textbf{Identification, counterfactuals, selection, and scale.} Sparse symbolic
regression identifies governing equations from a dictionary of candidate terms
\citep{brunton2016sindy}, while its weak Galerkin form reduces dependence on clean
derivative estimates \citep{messenger2021weak}. These methods are particularly strong
identifiers when the candidate dictionary contains the true dynamics, and we report
the regimes in which they outperform a learned simulator. Whether an intervention
outside the training support can be identified is closely related to the positivity
condition in causal inference \citep{pearl2009}, which motivates our distinction
between extrapolation, interpolation, and composition. Split-conformal prediction
\citep{lei2018conformal} provides the finite-sample coverage statement used when
selecting among candidate structural forms. At larger scale, graph-network and
mesh-based simulators \citep{sanchezgonzalez2020gns,pfaff2021meshgraphnets}, Fourier
neural operators \citep{li2021fno}, and video world models that predict and plan in
latent space \citep{assran2025vjepa2,nam2026causaljepa} define the current competence
frontier. Compute-optimal scaling \citep{hoffmann2022chinchilla} motivates the broader expectation that increased capacity improves generalisation, while physical-realism benchmarks test how far that expectation extends to consistency with physical laws \citep{gu2026phyworldbench,bansal2025videophy2}. Our scaling study probes this question using the same intervention applied to the structured models.

\section{Methods: regimes, models, guarantees, and the criterion}
\label{iclr:program}

\begin{figure}[h]
  \centering
  \includegraphics[width=\linewidth]{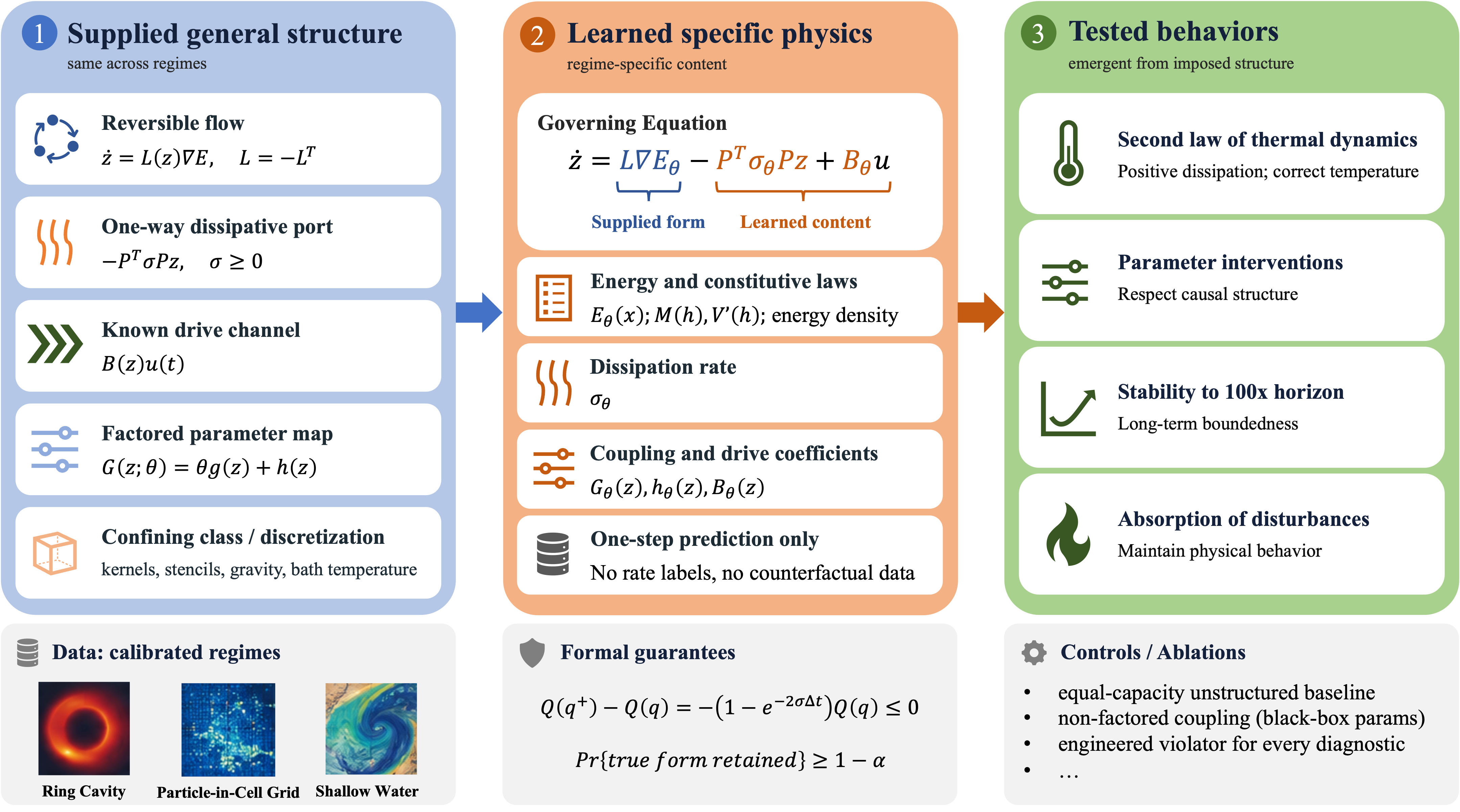}
  \caption{Overview of the approach: the general physical structure supplied by the modeller (left), the system-specific physics learned from data (centre), and the four behaviours supported by each structural element together with the
  interventions the structure refuses (right).}
  \label{fig:overview}
\end{figure}

This section follows a simple division of labour: the modeller specifies the general
form of the dynamics, the data determine the system-specific physics within that form, and each target behaviour is evaluated against a control that removes one structural component, so the resulting capability can be attributed to the component it depends on (Figure~\ref{fig:overview}).

\subsection{Three physical test systems, audited before any training} \label{sec:regimes} 

A structural claim is most convincing on systems where it could fail and where the ground truth can be checked independently. We therefore evaluate three deliberately different physical systems, each small enough to audit yet each stressing the imposed structure in a distinct way. The \emph{ring cavity} is a chaotic particle-field system in which charged particles on a periodic ring exchange energy with $K=16$ electromagnetic modes through a fixed smoothing kernel, integrated with a palindromic Strang splitting that is time-reversible and energy-preserving. The \emph{particle-in-cell grid} is a genuine field theory with a conserved charge through a Gauss-law invariant, implemented on a $64\times64$ periodic Yee finite-difference grid with Boris particle pushes and charge-conserving Esirkepov current deposition. The \emph{shallow-water fluid} is non-electromagnetic and exactly conserves mass, with water height $h(x)$ and momentum evolved on a periodic line using a spectral implicit-midpoint integrator. Together, these systems span a chaotic particle-field model, a conserved-charge field theory, and a fluid, reducing the chance that a shared result is specific to one physical setting. Training episodes contain $500$ steps for the cavity and grid and $200$ steps for the fluid. Every stability claim is evaluated for $50{,}000$ steps, corresponding to $100\times$ and $250\times$ the training horizons, respectively. 

Before any model is trained, each reference simulator must pass conservation tests over the full evaluation horizon, because long-horizon model stability cannot be interpreted below the numerical drift of the reference itself. All three simulators pass with margin, with energy drift at the $10^{-5}$ level or better and structural invariants preserved near machine precision. This calibration stage also exposed two integrator defects and a static self-force that would otherwise have appeared as learned-model instability at long horizon (App.~\ref{sec:systems}).

\subsection{The imposed form and what is learned}
\label{sec:form}

Every structured model evolves a state $z$ through a supplied antisymmetric operator acting on the gradient of a learned energy, a scalar one-way port on a designated coordinate block, and a known drive channel:
\begin{equation}
  \dot z=L(z)\nabla E_\theta(z)-P^\top\sigma_\theta Pz+B_\theta(z)u(t),
  \qquad L=-L^\top,\quad \sigma_\theta=\mathrm{raw}_\theta^2\ge0 .
  \label{eq:form}
\end{equation}
Here $P$ selects the dissipated coordinates, with $q=Pz$ and $PP^\top=I$. Antisymmetry implies $\dot E_\theta=0$ for the isolated continuous reversible dynamics; interpreting $L$ as a Poisson bracket additionally requires the Jacobi identity. The port is a scalar contraction of $q$, not an entropy gradient. For a quadratic diagnostic energy on the dissipated coordinates,
\begin{equation}
  Q(q)=\tfrac12\,q^\top Kq,\quad K\succ0,\qquad
  \dot Q\big|_{\rm port}=-\sigma_\theta\,q^\top Kq\le0,
  \label{eq:portQ}
\end{equation}
and the frozen-rate exponential step used by the implemented integrator contracts $Q$ exactly (Prop.~\ref{prop:discrete2law}). This is the guarantee, and its scope is restricted to this diagnostic and channel. It is not a statement about the learned energy $E_\theta$, since coordinate contraction can increase a non-quadratic energy (App.~\ref{sec:coverage}); nor does it apply to drives or thermal noise, which inject energy by design, or to an arbitrary damping matrix, which does not inherit this property merely from being positive semidefinite.

A physical parameter $\theta$ that may be intervened on, such as charge, gravity, or a rate, enters through a supplied factored map,
\begin{equation}
  G_\vartheta(z;\theta)=\theta\,g_\vartheta(z)+h_\vartheta(z),
  \label{eq:factored2}
\end{equation}
so a new value of $\theta$ evaluates a prescribed continuation of the learned content rather than requiring a new function. Each modal energy is confining. When the total energy, including its coupling terms, is coercive, its level sets are compact and exact conservative trajectories remain bounded. We state this condition explicitly as an assumption and distinguish it from the finite-horizon boundedness observed experimentally (Prop.~\ref{prop:A}, App.~\ref{app:theory}).

The division of labour between modeller and data is explicit throughout the paper. The modeller supplies the bracket $L$, the discretization and field kernels for each regime, the confining energy class, the location and sign of the port, the drive channel $B_\theta$, the factored map in Eq.~\eqref{eq:factored2}, the fluid gravity $g$, the bath temperature, and, for form selection, a finite menu of candidate structures. The data determine the system-specific content: the energy functional and constitutive functions, the dissipation rate, the drive-response and coupling coefficients, and the parameters within each candidate form. Each learned quantity is evaluated separately in the results, with the concrete learned content for each regime reported in App.~\ref{sec:models}.

\subsection{Two guarantees, with their scope}

The imposed structure provides two guarantees, stated and proved formally in App.~\ref{sec:coverage}. The first concerns the dissipative port at the level of the implemented discrete update. For a quadratic diagnostic energy on the port channel, the exact exponential update contracts that energy at every step size and never injects it (Prop.~\ref{prop:discrete2law}); under fluctuation-matched stochastic forcing, the isolated port also has the corresponding Gibbs distribution as its invariant law. Composition with the reversible step and the use of state-dependent rates fall outside this guarantee, so the full-model temperature results reported below remain empirical. The second is a finite-sample coverage guarantee for the form selector. Under exchangeability, the true in-family form is retained with probability at least $1-\alpha$ (Prop.~\ref{prop:fourway}). This is a marginal coverage statement. It does not bound the error of an individual selective decision or guarantee coverage after the data-generating process changes, both of which we evaluate separately.

\subsection{Controls that isolate the structure, and how claims are measured}

Every claim is evaluated against a control that removes one structural component, so the resulting behaviour can be attributed to the component whose removal destroys it. An equal-capacity \emph{unstructured baseline} removes all imposed structure and tests whether parameter count alone explains the result. A \emph{non-factored control} retains locality and field mediation but couples the intervened parameter through a black box, so behaviours it preserves are credited to locality while failures isolate the separable parameter map. A \emph{sign-free port} allows the dissipation rate to take either sign, separating the guarantee of non-injective dissipation from the ability to represent its violation. An \emph{oracle} receives the intervened physics directly, showing that a refused intervention is learnable when the missing structure is supplied. All models remain small, with approximately $8$ to $12$ thousand parameters, and are trained only for next-state prediction. Architectures, exact parameter counts, and training settings are given in App.~\ref{sec:models}.

Three measurements recur throughout the results. Rollout accuracy is measured by normalised mean-squared error against the reference trajectory from the same initial state; energy behaviour is measured by the relative drift of the true energy along a rollout; and a learned rate is evaluated by its relative error against the true rate. Every experiment is repeated across independent random initialisations, and we report the resulting variation through error bars or standard deviations rather than a single run. Each diagnostic is also paired with an engineered case that it must detect. Formal metric definitions, the repetition protocol, and the three evidence labels used to grade findings in the appendices are provided in App.~\ref{sec:models} and App.~\ref{sec:program}.

\subsection{When is a counterfactual answerable by construction?}
\label{iclr:criterion}

Let the true generator be $G_*(z;\theta)=\theta\,g_*(z)+h_*(z)$ and the learned generator $G_\vartheta$ be defined as in Eq.~\eqref{eq:factored2}. If the learned content matches the truth on a state region $\mathcal R$, then for every admissible
$\theta'$,
\begin{equation}
  g_\vartheta=g_*\ \text{and}\ h_\vartheta=h_*\ \text{on}\ \mathcal R
  \quad\Longrightarrow\quad
  G_\vartheta(z;\theta')=G_*(z;\theta')\ \text{for all}\ z\in\mathcal R,
  \label{eq:transfer}
\end{equation}
so the intervention is answered exactly wherever the learned content remains accurate on the states visited by the rollout (Prop.~\ref{prop:B}; a known nonlinear parameter map would provide the same principle). Without a restriction on how the parameter enters the dynamics, observations confined to the training support cannot distinguish functions that agree on that support but differ outside it. No amount of additional in-support data can therefore determine the off-support response.

Two consequences organize our experiments. First, a counterfactual test is informative only when the queried parameter value leaves its marginal training support. In-support values and novel combinations of previously observed values instead test interpolation and composition, capabilities that locality and field mediation can already provide. Second, the criterion predicts its own failures. An intervention on a parameter for which no map is supplied, such as a kernel width inside a Gaussian or an exponent inside a power law, should fail, while an oracle supplied with the missing map should succeed. Such a failure therefore identifies missing intervention structure rather than insufficient data or model capacity.

\section{Results}
\label{iclr:results}

\subsection{One set of weights decides whether a world dissipates and recovers its functions}
\label{iclr:discovery}

The first question is whether a learned simulator can decide from trajectories whether its world loses energy and recover the constitutive functions that govern it. Existing structured models do not pose this directly: conservation is fixed by the class or a dissipative head may take either sign, so conservative and weakly dissipative worlds need not be distinguished. The one-way port in Eq.~\eqref{eq:form} makes ``conservative'' the $\sigma=0$ member of the dissipative family, turning the fitted rate into a measurable quantity. We train one model jointly on conservative, dissipative ($\sigma=0.5$), driven, and external-field cavity episodes, giving only the condition identity.

The resulting model separates the rates by four orders of magnitude: $\hat\sigma_{\rm cons}=1.2\times10^{-5}$, corresponding to an implied spurious loss of $0.6\%$ over $50{,}000$ steps, $\hat\sigma_{\rm diss}=0.481$, and exactly zero for the external-field condition. The fluid reproduces the same separation, $\hat\sigma_{\rm inviscid}=1.4\times10^{-4}$ versus $\hat\sigma_{\rm viscous}=0.500$. The controls isolate the source of this result. Without the port, the model cannot fit the dissipative world, with rate error $0.99$; the informed baseline fails outright, with rollout error ${\sim}10^{6}$; and the sign-free port learns a small \emph{negative} rate for the conservative condition, creating a spurious energy source that the one-way port forbids by construction. The estimated rate is not merely decorative. Clamping the dissipative rate to incorrect values and retraining the remaining components leaves the fit minimum at the true value, showing that the learned energy does not absorb the missing dissipation. On fresh data, the conservative and dissipative rates remain separated by orders of magnitude, and a label-free encoder recovers the same distinction from only the first $100$ steps of a rollout (App.~\ref{sec:discovery}).

\begin{table}[h]
  \centering
  \caption{Recovered constitutive functions. Learned against the reference over each
    operating range, the grid's vacuum-energy derivatives are linear, implying a
    quadratic energy, while the fluid's mass density $M(h)$ and potential force $V'(h)$
    are recovered at unit slope. Gravity is supplied, so the fluid coefficient is an
    effective coefficient rather than discovered gravity, and $V'$ is identifiable only
    up to an additive constant.}
  \label{tab:discovery}
  \begin{tabular}{llcc}
    \toprule
    system & learned function & slope vs.\ reference & $R^2$ \\
    \midrule
    grid  & field constitutive derivative & $1.043$ & $1.0000$ \\
    grid  & field constitutive derivative & $1.001$ & $1.0000$ \\
    fluid & mass density $M(h)$           & $1.000$ & $1.00$ \\
    fluid & potential force $V'(h)$       & $1.000$ & $1.00$ \\
    \bottomrule
  \end{tabular}
\end{table}

The same weights recover functions rather than only scalar coefficients. On the grid, with no functional form assumed beyond confinement, the learned constitutive derivatives are linear, so the model recovers the quadratic form of the vacuum energy. A residual $4\%$ slope-ratio systematic remains, which we report as a constitutive diagnostic rather than interpret as an independent wave-speed estimate. In the fluid, the mass density $M(h)$ and potential force $V'(h)$ are recovered at unit slope (Table~\ref{tab:discovery}). Because gravity is supplied, the fitted coefficient is effective rather than an independently discovered gravity, and $V'$ is recovered only up to a Casimir gauge that leaves the trajectories unchanged. A genuinely nonlinear truth, $V'(h)=h^{2}$, is recovered with its curvature. Compared with an explicitly linear constitutive model trained to convergence on identical data, the nonlinear model reduces one-step, rollout, and held-out gravity-intervention errors by $6$--$17\times$ at small excitation and $2$--$5\times$ at larger excitation, while retaining the same long-horizon stability. The advantage appears primarily in downstream prediction rather than uniform function error because a linear law approximates $h^2$ closely over the sampled operating range (Table~\ref{tab:nonlinear-recovery}, App.~\ref{sec:discovery}).

At the next level, the model can also select among candidate structural forms. A Bayesian information criterion over structural channels identifies the true form on all $40$ stress-ladder worlds and never selects a non-nested incorrect form. On the cavity, it selects the dissipative and driven forms by $\Delta$BIC $2257$ and $1054$, while also recovering the dissipation rate. The only near-degenerate case is the conservative world, where the one-step objective cannot cleanly separate exact conservation from vanishingly small dissipation, even though those alternatives differ strongly over the full rollout horizon. The main selection result therefore uses a frozen conformal selector evaluated across four independently implemented generator families. It reaches four-way accuracy $0.84$--$0.89$ and true-candidate coverage $0.73$--$0.90$ (Table~\ref{tab:selector-full}). Coverage below the nominal $0.90$ on two families occurs under generator shift and therefore does not contradict the exchangeable coverage guarantee (App.~\ref{sec:coverage}). Taken together, one model per regime identifies whether its world dissipates, recovers the constitutive content up to stated non-identifiabilities, and carries its estimation error into the predictions that use that content.

\subsection{The second law holds by construction, and the guarantee has a price}

An unconstrained model gives no structural reason for the sign of a passive energy exchange to be correct. Our solution is architectural: the port in Eq.~\eqref{eq:form} can only remove the quadratic diagnostic energy of its channel (Prop.~\ref{prop:discrete2law}), so within that scope non-injective dissipation is a property the model cannot violate rather than one it must learn. The remaining questions are whether it recovers the amount of dissipation quantitatively and at what representational cost, which we test on the cavity with a learned rate, a driven condition, a fluctuating bath, an engineered violator, and a sign-free port that exposes what the constraint excludes.

\begin{figure}[h]
  \centering
  \includegraphics[width=\linewidth]{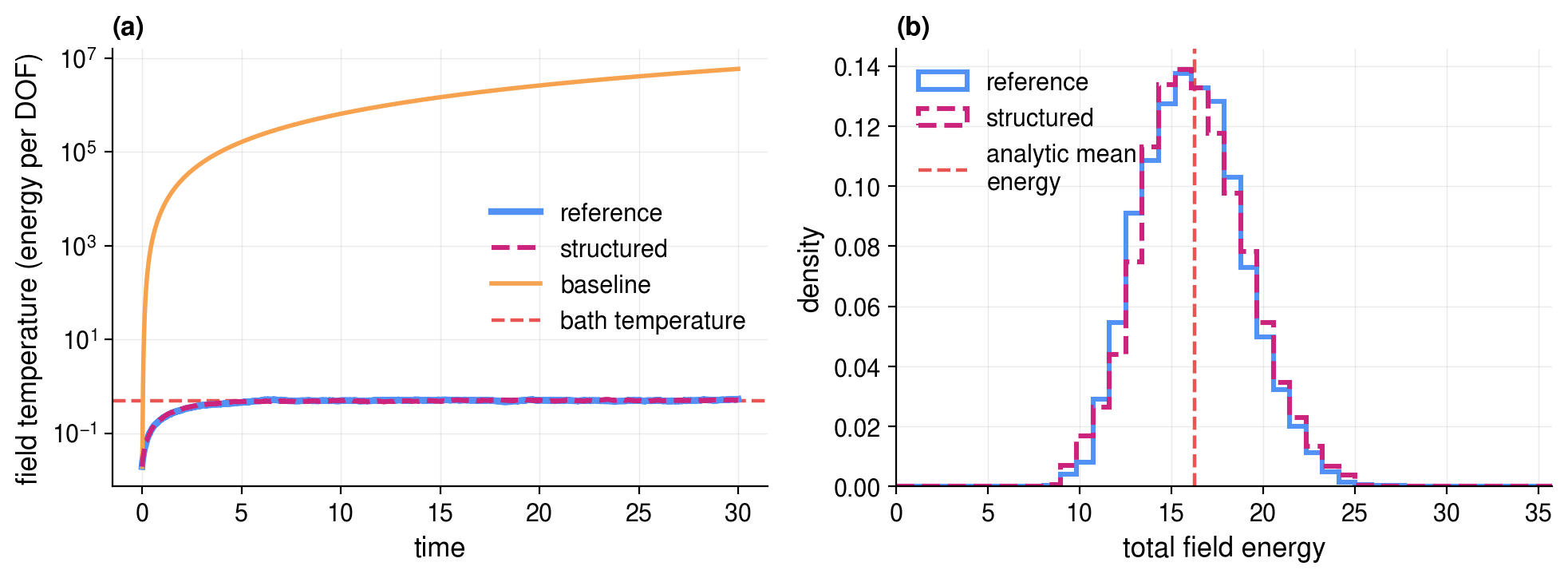}
  \caption{The stochastic second law under a fluctuating bath: field temperature over time (a) and the stationary distribution of the field energy (b). Reference in blue, structured model in dashed magenta, equal-capacity baseline in orange, and the bath temperature and analytic mean as dashed red lines; the baseline is off-scale in (b).}
  \label{fig:secondlaw}
\end{figure}

The deterministic decay rate is recovered to within $0.4\%$, and the driven steady state is reproduced correctly. Figure~\ref{fig:secondlaw} shows the fluctuating-bath experiment. The left panel reports field temperature over time, while the right panel shows the stationary field-energy distribution. From a cold start, the structured model relaxes with the reference to the bath temperature, with both field sectors agreeing, whereas the equal-capacity baseline reaches a temperature roughly six orders of magnitude too high. A noise-doubled violator is flagged reliably while the structured model never triggers the diagnostic. The stronger evidence comes from the magnetic sector, which is neither directly damped nor driven. It can thermalise only through the learned reversible coupling, so its convergence to the same temperature shows that the learned conservative core transfers energy with the required balance, something a dissipation-side fit alone cannot produce. The representational price is equally explicit. When asked to run a negative dissipation rate, the constrained port clamps at zero while the sign-free port tracks the divergent world. A model family that guarantees the sign of dissipation therefore cannot represent the opposite sign through that channel, and genuine energy injection must instead enter through the explicit drive (App.~\ref{sec:secondlaw}). This constraint matters most near $\sigma=0$, where the data provide the weakest sign information and the sign-free control in the previous section invents a small energy source.

\subsection{Interventions on physical parameters transfer where the map is supplied and are refused where it is not}
\label{iclr:cf}

\begin{table}[h]
  \centering
  \caption{A charge-magnitude counterfactual on the cavity. Rollout error (nMSE at $1000$
    steps, mean $\pm$ s.d.\ over repeated runs) at three unseen magnitudes, for a model
    trained at $|q|=1$ and a control trained on the wider range $|q|\in\{1,1.5\}$; both
    share the supplied charge map. The two rows coincide across all test values, showing
    that wider parameter coverage provides no measurable benefit. Residual error is set by
    the state region visited by the rollout rather than by distance from the training
    value.}
  \label{tab:counterfactual}
  \begin{tabular}{lccc}
    \toprule
    training range & $|q|=1.5$ & $|q|=2.0$ & $|q|=2.5$ \\
    \midrule
    $|q|=1$            & $0.46\pm0.12$ & $0.45\pm0.14$ & $0.46\pm0.12$ \\
    $|q|\in\{1,1.5\}$  & $0.45\pm0.09$ & $0.45\pm0.11$ & $0.44\pm0.09$ \\
    \bottomrule
  \end{tabular}
\end{table}

The second behaviour asks whether a model can answer a ``what if'' about a physical parameter at a value absent from training, which learned simulators often answer with the training law. The criterion in \S\ref{iclr:criterion} says when a response is available by construction: the parameter must enter through a supplied map, and the query must leave its training support. We test both sides: transfer should hold for a flipped charge sign and for unseen magnitudes, rates, and gravity when their dependence is factored, and should be refused when no map is supplied, with an oracle given the missing map showing the refused physics is learnable.

The factored model follows the flipped-sign world on both electromagnetic regimes ($p=0.001$), and the result replicates on fresh data. The baseline and non-factored control instead perform at or \emph{below} chance. Every training trajectory rotates in one direction, so the black box learns this rotation as a property of the dynamics and retains it after the sign flip. Magnitude, rate, and gravity also transfer when the generator is linear in the intervened parameter. A doubled charge matches an oracle trained directly at that value, while the non-factored control reverses its rotation; a dissipation rate trained at $0.5$ transfers to $0.25$; the fluid viscosity transfers with comparable accuracy; and changing gravity from $1.0$ to $1.5$ yields error $6.8\times10^{-11}$, indistinguishable from the in-distribution value, while the non-separable control and baseline fail by three to twelve orders of magnitude. Increasing the training coverage of the parameter provides no additional benefit (Table~\ref{tab:counterfactual}). The residual error is determined by the region of state space visited by the new trajectory rather than by the numerical distance of the parameter from training. Two observations sharpen this distinction. First, composition and extrapolation separate cleanly: both field-mediated models compose an unseen mixed-charge configuration, whereas only the factored model extrapolates the governing law. Locality therefore supports composition, while the supplied parameter map supports the support-moving intervention. Second, the remaining error is limited by state-space coverage rather than parameter-space distance: the flipped grid trajectory eventually leaves its trained state envelope, while the cavity remains accurate for a doubled charge through $1000$ steps.

The refusals are equally informative. In each regime, an intervention whose parameter map was not supplied causes the factored model to fail while an oracle given the missing map succeeds. A particle-shape width inside a Gaussian produces error $111\times$ the in-distribution value, a pressure exponent inside $h^\gamma$ produces approximately $10^{9}\times$ the in-distribution error, and the grid width produces a $7.0\times$ increase, below the pre-specified $10\times$ threshold. The corresponding oracle succeeds in every case (App.~\ref{sec:falsifiers}). Misspecification of the broader function class can also be detected empirically: a confining-quadratic energy fitted to a quartic world raises its fit floor by four orders of magnitude and, once the quartic contribution becomes sufficiently strong, fits worse than a capacity-matched black box even without access to the true model class. These are implementation-level failures rather than an impossibility theorem for nonlinear maps. Their significance is that the criterion predicts where the tested structure should stop transferring, turning the successful interventions from isolated examples into a falsifiable rule.

\subsection{Long-horizon stability needs a confining energy and an energy-consistent coupling}
\label{iclr:behaviours}

The third behaviour is remaining physically bounded far beyond the training horizon. Unconstrained simulators fail because small one-step errors accumulate without a structural mechanism that bounds the state. A conserved energy alone might appear sufficient, but the ablations show that two ingredients are required. With an unconstrained learned field energy, the cavity model fits the training dynamics and passes the counterfactual tests, yet its $50{,}000$-step rollout diverges, reaching deviations of $2.9\times10^{7}$. Restricting the learned energy to the confining class changes the result from uniformly diverging to uniformly bounded. Figure~\ref{fig:behaviours} shows the full rollout for three models. The structured model remains near $10^{-2}$ relative true-energy deviation, the equal-capacity baseline exceeds $10^{8}$ early in the rollout, and the non-factored control reaches $10^{4}$ despite sharing the same confining energy.

\begin{figure}[h]
  \centering
  \includegraphics[width=0.8\linewidth]{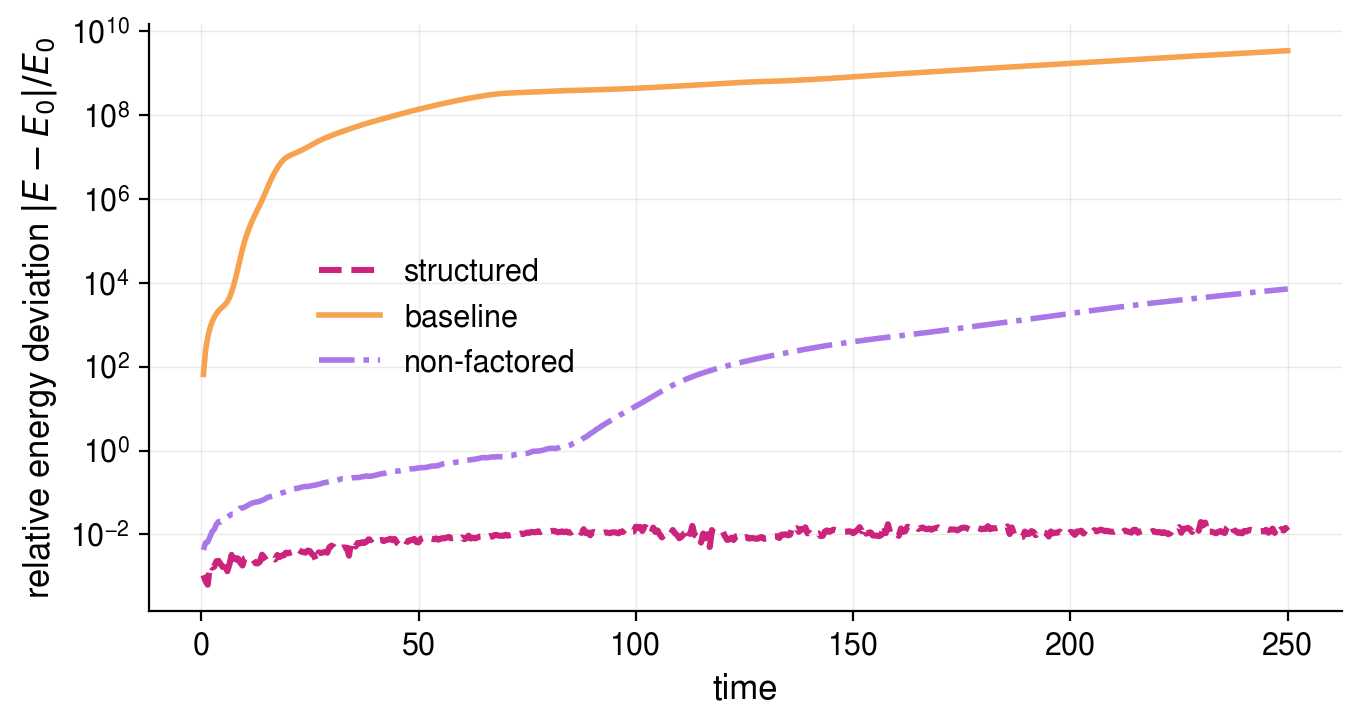}
  \caption{Long-horizon energy drift on the conservative cavity over $50{,}000$ steps
    ($100\times$ the training horizon, log scale): structured model in dashed magenta, equal-capacity baseline in orange, and the non-factored control in dash-dotted purple.}
  \label{fig:behaviours}
\end{figure}

Confinement alone is therefore insufficient. The non-factored control still diverges because its learned coupling injects net energy at the particle-field interface. In the true dynamics, the particle force and its induced field current are derivatives of the same interaction energy, so their energy exchange cancels. A black-box coupling learns the two sides independently, allowing a small mismatch to perform net work each cycle, which accumulates over a long rollout. The fluid provides the complementary control. Its non-separable variant black-boxes only the gravity dependence while preserving the energy-consistent flux coupling. It therefore remains bounded but fails the gravity counterfactual. Stability thus follows coupling consistency, whereas counterfactual transfer follows the supplied parameter map. Boundedness should not be confused with accuracy: rare cavity episodes exceed the $5\%$ worst-case energy criterion even though typical trajectories remain well within it, while the grid closes this gap. We therefore report boundedness, energy fidelity, and rate accuracy as distinct outcomes.

\subsection{Disturbances, scale, noise, and everyday physics}
\label{iclr:disturbance}
\label{iclr:scaling}

The fourth behaviour is absorbing a disturbance, and three supporting studies then ask whether the preceding effects could instead be explained by scale, absence of noise, or the idealised systems. A mid-rollout velocity kick is absorbed because it moves the state along the reachable manifold, whereas a coherent field perturbation is off-manifold and remains the sensitive direction. Conserved invariants double as validity meters, transporting an injected violation without amplification (App.~\ref{app:perception}). Scale does not recover the missing response: across four black-box families from $10^{4}$ to $10^{8}$ parameters, no competent regime shows a reliable sign-flip counterfactual (App.~\ref{app:scaling-probe}). The advantage also survives $5\%$ training noise, and where a clean, correctly specified dictionary exists symbolic regression is the more accurate tool, so the motivation for learned structure is high dimensionality, partial observation, unknown dictionaries, and long-horizon stability rather than noise alone. Finally, an eleven-system suite of everyday physics and the four scope exceptions are evaluated in App.~\ref{sec:zoo}, including the one pre-specified selection gate the full matrix misses and why (App.~\ref{sec:supporting}).

\section{Conclusion and Limitations}
\label{iclr:conclusion}
\label{iclr:limitations}

This paper asks how little structural inductive bias a learned world model needs in order to recover identifiable physics and behave reliably outside its training distribution. Fully unconstrained models cannot infer true dynamics from trajectories alone: they often diverge over long rollouts and collapse under interventions on physical parameters. Adding a small amount of geometric and thermodynamic scaffolding changes this picture entirely. A Poisson-style bracket, a confining energy class, a port-Hamiltonian dissipation channel, and factored parameter maps together make the remaining constitutive relations identifiable and accurate. The resulting model behaves like a calibrated instrument rather than a black box: it cleanly separates conservative and dissipative regimes, learns nonlinear constitutive laws, transfers interventions where equally expressive baselines fail, and stays stable over long rollouts. The general geometric structure helps rather than hurts, because it provides the coordinate frame within which data can pin down the specific physics.

However several limitations remain. The dissipation guarantee applies only to quadratic port-channel energy, and boundedness is proven only for conservative dynamics (else supported empirically). Potential force is identifiable only up to an unobservable Casimir constant. Our work is not unsupervised law discovery - brackets, channels, and kernels are user-supplied, which also blocks unspecified parameter interventions. The results in this work rely on synthetic classical simulators rather than experiments, with edge-case failures (e.g., cavity energy tails), generator shifts, and inconclusive scaling results defining necessary operating boundaries.

\section*{AI use statement}

A large language model assisted the authors in polishing the language of the manuscript. The authors verified that all claims, proofs, mathematical formulations, and reported values are valid, checked them against the implementation and results, and take full responsibility for the manuscript.

\section*{Ethics statement}
This work uses only synthetic physics simulations; no human subjects, personal
data, or scraped datasets are involved. We do not foresee direct harmful
applications; the contribution is a measurement discipline and structural
analysis for physical world models. Integrity-relevant practices
(pre-registration, recorded verdicts, disclosure of overturned verdicts and
process defects) are documented in Appendices~\ref{app:integrity}
and~\ref{app:record}.

\section*{Reproducibility statement}
The reference simulators, calibration gates, model families, and training
protocols are specified in Appendices~\ref{sec:program}--\ref{sec:zoo}; the theory
and proofs in Appendices~\ref{sec:coverage} and~\ref{app:theory}; the
per-experiment record with per-seed results in Appendix~\ref{app:record}; and the
verdict-integrity protocol in Appendix~\ref{app:integrity}, under which the
constants that determine every reported verdict are fixed by a pre-registration
committed before the affected runs and the result tables are generated directly from
the stored measurements rather than transcribed by hand. All data are regenerated
deterministically from seeds. Our implementation, the pre-registration documents with
their dated amendment history, and the result artifacts are provided as supplementary
material.

\bibliography{refs}
\bibliographystyle{iclr2027_conference}

\appendix
\section*{Overview of the appendices}
These appendices are organised by purpose so that each can be read on its own. The first
group describes \emph{how} the work is measured and \emph{on what}:
Appendix~\ref{sec:program} states the measurement methodology and the discipline used to
turn observations into claims, and Appendix~\ref{sec:systems} specifies the three
physical systems, their calibration, and the model and control families with their
training details. The second group presents the \emph{experimental evidence}, one
appendix per scientific question: the criterion for counterfactual transfer and its
refusal tests (Appendix~\ref{sec:criterion}); the two ingredients of long-horizon
stability (Appendix~\ref{sec:stability}); dissipation, driving, and stochastic
consistency (Appendix~\ref{sec:secondlaw}); discovery of the physics and selection of
its form (Appendix~\ref{sec:discovery}); and generality across a different physics and a
symbolic-regression competitor (Appendix~\ref{sec:generality}). The third group steps
back: Appendix~\ref{sec:synthesis} synthesises which structural element buys which
behaviour, the stated scope exceptions and an eleven-system suite are tested in
Appendix~\ref{sec:zoo}, and Appendix~\ref{sec:discussion} discusses what the evidence
does and does not support. The final group is \emph{reference material}: the
verdict-integrity protocol (Appendix~\ref{app:integrity}), the formal propositions and
their proofs (Appendix~\ref{app:theory}), the complete per-experiment record
(Appendix~\ref{app:record}), and an extended related-work discussion
(Appendix~\ref{app:related}).

\section{The measurement program: behaviours, controls, and claim discipline}
\label{sec:program}

\subsection{Why a measurement program at all}
The failure mode of the field is not a shortage of architectures but a shortage of
\emph{discriminating tests}. A model can score well on trajectory error while
violating energy conservation catastrophically at longer horizons. It can ``pass''
a counterfactual test that any interpolating model would pass. And its energy drift
can look flat because it predicts trivial motion. Each behaviour below therefore
comes with three elements: (i) a confound-resistant metric, (ii) a \emph{positive
control} (an engineered case the diagnostic must flag, so a silent diagnostic cannot
masquerade as a passing model), and (iii) a control model that isolates which
structural element the behaviour credits.

\subsection{The four behaviours}
\label{sec:behaviours}
\begin{table}[h]
  \centering
  \setlength{\tabcolsep}{4pt}
  \caption{The behaviours, structured model against controls. Each row's control
    lacks exactly one structural element. Seeds are $5$ unless stated; flips are
    seed-level exact binomials over $10$ seeds.}
  \label{tab:behaviours}
  \begin{tabular}{@{}p{0.29\linewidth}p{0.19\linewidth}p{0.18\linewidth}p{0.27\linewidth}@{}}
    \toprule
    Behaviour (regime) & Structured & Unstructured baseline & Ablated control \\
    \midrule
    Bounded seeds, $100\times$ horizon (cavity) & $5/5$; worst case $5$--$12\%$ & $0/5$ & non-factored coupling: $1/5$ \\
    Decay-rate error; stationary $T$ error (cavity) & $0.4\%$; $0.3\%$ & injects power; hot ${\sim}10^{6}$ & sign-free port: negative rate on $3/5$ seeds \\
    Charge-sign flip, seeds (cavity; grid) & $10/10$; $10/10$ & $2/10$; $4/10$ & non-factored: $1/10$; $5/10$ \\
    Gravity $1.0\to1.5$, nMSE (fluid) & $6.8\times10^{-11}$ & $2.5\times10^{3}$ & non-separable: $1.06$ \\
    Refusal of an unsupplied map (cavity; fluid; grid) & $111\times$; $10^{9}\times$; $7\times$ & --- & oracle with the map: $1.4\times$; $10^{-6}$; $0.34\times$ \\
    Fitted rate, conserving / dissipating (cavity, $10$ seeds) & $1.2\times10^{-5}$ / $0.481$ & informed: rollout ${\sim}10^{6}$ & no port: rate error $0.99$ \\
    Velocity kick, post/pre error (cavity; grid) & $0.85$; $1.14$ & divergent & non-factored passes insertion (locality) \\
    \bottomrule
  \end{tabular}
\end{table}

\subsection{Disturbances, scale, and noise: the numbers behind the main-text summary}
\label{sec:supporting}
A mid-rollout velocity kick is absorbed (post/pre error ratio $0.85$ on the cavity,
$1.14$ on the grid) because it moves the state along the dynamics' reachable
manifold. A coherent field bump, by contrast, manufactures an off-manifold
configuration and remains the sensitive direction. Inserting a third particle
mid-rollout continues on the correct law, and because the non-factored control also
passes this test, the credit goes to locality rather than to the factored map.

Re-identifying the state from a single rendered frame misses velocity (continuation
error $22.8$), whereas three strided frames recover it by differencing ($4.5$;
oracle floor $0.013$). The boundary is therefore momentum, not perception in
general (App.~\ref{app:perception}).

Scaling four black-box families from $10^{4}$ to $10^{8}$ parameters under a
competence-gated protocol produced no rung with a reliable sign flip. The result is
inconclusive on its own terms, since only one rung was competent. A
width-controlled study with the representation fixed finds no gain from width
(one-sign error $8.16$--$11.83$, Holm $p=0.105$; App.~\ref{app:scaling-probe}).

Under $5\%$ training noise the flip stays intact and the stability envelope loosens
gracefully ($0.045\to0.285$ at $10$k steps). Symbolic regression with the same
dictionary is most accurate on the noise-free cavity (worst coefficient error $0.07\%$). However,
its central-difference form collapses under noise ($0.02\to0.23\to0.94$ at $0.5$,
$2$, $5\%$), while the weak form stays competitive
($0.007\to0.028\to0.10$). The suite comprises eleven systems with twelve pass cards,
eight clean and four with one identified limit (App.~\ref{sec:zoo}).
\textbf{The second law.} The model must dissipate at the true rate to the true
terminal state, track genuine external injection without manufacturing it, and
(the statistical half) relax to the correct stationary distribution and temperature
under a fluctuating bath. Metrics are rates and distributions, never trajectory
error, because in chaotic or stochastic systems trajectory error is meaningless (a
lesson enforced throughout). Required controls
include a genuinely driven system the model must \emph{not} force to zero, and an
engineered second-law violator the production diagnostic must catch.

\textbf{Counterfactuals.} Trained at one parameter setting, tested at a never-seen setting
(sign flip, unseen magnitude, unseen rate, unseen coupling constant). Metrics: a
regime/orbit-sense match \emph{plus} error against the true counterfactual trajectory at a stated
horizon. A test only counts if the metric has \emph{power} on the model under test (the model
must get the \emph{trained} setting right on held-out episodes). Task validity is defined
by reference accuracy, intervention magnitude, and measurement sensitivity, not by the
comparator's outcome: an equally-structured non-separable control is always run and its
result reported, and a task on which both models succeed is recorded as ``not
discriminating between these methods'', an informative outcome that locates where the
structural constraint is unnecessary, rather than an exclusion rule.

\textbf{Extreme horizon.} Rollouts to $\ge 100\times$ the training horizon. Energy-drift
slope and boundedness are co-reported with rollout error, because a low energy drift can be achieved by trivial motion and
boundedness alone hides excursions; worst-case per-seed deviations are reported, not only means.
Control: the equal-capacity baseline is always run at the same horizon and reported; a horizon
on which it also stays bounded is recorded as not discriminating between the methods, not
discarded, since finite trajectories are not necessarily accurate ones and both fidelity and
boundedness are reported for every model.

\textbf{Disturbance robustness.} Graded sub-tests, scored separately because they credit
different structure: (a) reset to unseen-but-valid states, swept small to large; (b) structural
intervention \emph{given} the new state (add or teleport an object mid-rollout); and (c)
re-identification of the change from lossy \emph{observation}, carried through this paper as a
boundary to be mapped, not a capability to be claimed (\S\ref{app:perception}).

\subsection{Claim discipline}
\label{sec:discipline}
Every quantitative claim carries one of three evidence labels, used consistently
throughout the appendices. \decisive{} marks a large effect with non-overlapping
confidence intervals that replicates across independent training seeds.
\directional{} marks a scoped finding whose intervals are wide, or in which one
pre-registered arm falls below its threshold. \boundarylbl{} marks a mapped limit of
the method, reported as a result of equal standing rather than hidden. For outcomes that
are essentially pass-or-fail, we test at the level of whole seeds rather than individual
episodes, since episodes within one seed are correlated (a broken control once scored a
perfect 32/32 episode block, which retired per-episode averaging permanently;
Appendix~\ref{app:integrity}).

Credit assignment is fixed \emph{in advance}: tests passed by any field-mediated
model (composition, intervention-given-state) credit \emph{locality}, and only
support-moving counterfactuals and energy-consistent-coupling results credit
\emph{factoring}. The verdict text is generated directly from the pre-registered constants rather than transcribed by hand.
Headline verdicts require an escalation (ten seeds and fresh data draws) before any
guard moves. Two preliminary scaling verdicts failed their escalation and are
reported as such, with the full ledger of metric corrections in
Appendix~\ref{app:integrity}.

\section{Three regimes, calibrated before any training}
\label{sec:systems}

No model in this paper is trained before its reference simulator passes conservation gates over
the \emph{full evaluation horizon}. The reason is arithmetic: our stability claims run to
$100\times$ the training horizon, so reference drift at that horizon is a direct floor on every
number we report. The calibration-first requirement also caught real defects before they could
contaminate results, and we report them because they are part of why the numbers can be trusted.

\subsection{Electromagnetic ring cavity (modal)}
Charged particles on a 1D periodic ring coupled to $K=16$ field modes through a fixed smoothing
kernel, with a palindromic (Strang) symplectic integrator. \textbf{Calibration}: over
50{,}000 steps, worst relative energy drift $1.38\times10^{-5}$, which is $7.25\times$ under the
$10^{-4}$ gate on the binding (mixed-charge) configuration, and the Gauss-law residual
$\le 2.0\times10^{-14}$, which is exactly conserved. Two defects were caught by the
pre-registered test battery before any training. The first was a first-order operator
splitting (energy error $O(\Delta t)$ and secular, instantly diagnosed by the
halve-$\Delta t$ test that distinguishes first- from second-order), fixed by the
palindromic composition. The second was a tensor-view aliasing bug introduced by the
fix itself (a symplectic pair update through in-place slices mutated its own inputs),
fixed by snapshotting.
\emph{Why this matters beyond hygiene:} both defects would have shown up as ``model instability''
at long horizon, exactly the phenomenon under study, had they survived.

\subsection{2D electromagnetic grid (particle-in-cell)}
A $64\times64$ periodic Yee-FDTD grid, Boris particle pusher, charge-conserving (Esirkepov)
current deposition, staggered-consistent force gather, binomial smoothing applied identically to
sources and gathered fields. \textbf{Calibration}: energy drift
$\le4.6\times10^{-5}$ (gate $10^{-2}$) and discrete Gauss residual $3.1\times10^{-13}$ (gate
$10^{-10}$) at 50{,}000 steps, so the particle-in-cell numerical-heating risk did not materialize
at this configuration. The battery caught a static self-force (a stationary charge felt
$0.08$--$0.2$ of its own field; a moving charge lost kinetic energy to spurious drag),
root-caused in two layers: the force gather must drop one B-spline order per staggered axis to
pair with Esirkepov currents, and the remainder is the lattice self-energy gradient, suppressed
$40\times$ by the smoothing. \emph{Why it matters:} an uncaught self-force is a phantom
dissipation channel that would have corrupted every second-law measurement on this system.

\subsection{1D shallow-water fluid (non-electromagnetic)}
Water height $h(x)$ and momentum on a periodic line: conservative core (mass, energy
$=$ kinetic $+\ \tfrac12 g h^2$), viscous drag (the physical analog of the dissipative port),
and wind forcing (the drive), advanced by a spectral, symmetric implicit-midpoint integrator.
\textbf{Calibration}: over 50{,}000 steps ($250\times$ the 200-step training
horizon), mass is exact ($3\times10^{-15}$), inviscid energy drift is $1.2\times10^{-7}$ (linear) and
$7.5\times10^{-5}$ at the nonlinear operating amplitude ($h$ varying $\pm10\%$), gravity-wave
speed is within $0.6\%$ (worst mode, gate $1\%$), and viscous decay matches the analytic $2\nu$ law to
$10^{-10}$. The system passes 18 unit tests. Scope: the smooth (pre-shock) regime, where randomly seeded multi-mode
initial data disperses and recurs rather than steepening into a coherent shock at our amplitude.
This is what makes a $250\times$ clean horizon available at a genuinely nonlinear operating
point. \emph{Why a fluid:} it changes everything incidental, namely the conservation structure
(mass instead of charge), the field content (scalar height instead of vector fields), and the
physical domain, while keeping the imposed form, so any behaviour that reproduces here is a
property of the form, not of electromagnetism.

\subsection{Structured models and the control family}
\label{sec:models}
All structured models share one imposed form. The state evolves by a fixed energy-conserving
flow generated by a \emph{learned} energy functional: per-mode confining quadratics with learned
bounded corrections on the cavity; a learned constitutive energy density entering fixed
curl-form (Yee) stencils on the grid; and learned mass-metric $M(h)$ and potential-derivative
$V'(h)$ functions in the fluid, each parameterized as a linear term plus a bounded
nonlinear correction,
\begin{equation}
  \begin{aligned}
    M(h)&=\mathrm{softplus}(a)\,h+\alpha_M\tanh\!\big(\mathrm{MLP}_M(h)\big),\\
    V'(h)&=\mathrm{softplus}(c)\,h+\alpha_V\tanh\!\big(\mathrm{MLP}_V(h)\big),
  \end{aligned}
  \label{eq:sw-constitutive}
\end{equation}
together with a viscosity port. Irreversibility exists only as a \emph{learned non-negative}
port (rate $\sigma\ge0$ multiplying the exact dissipative structure); energy injection exists
only as \emph{known explicit} drives. Coupling constants that the criterion requires to be
counterfactual-ready (charge $q$, gravity $g$) enter multiplicatively. Because gravity $g$
is supplied, a recovered fluid ``gravity'' of $1.000$ is the supplied value times the
learned slope of $V'$, an effective coefficient rather than a discovered constant.

The control family is the paper's inference engine:

\textbf{Unstructured baseline} (DeepSet on the cavity, CNN on grid and fluid), capacity
matched to a disclosed few percent in these comparisons. Its reported failures
show that parameter count alone does not ensure success under this training protocol;
they do not exclude alternative unstructured architectures or optimization methods.

\textbf{Non-factored control}: identical locality and field mediation, coupling through a
black box, with \emph{slightly more} parameters. The comparison tests the effect
of supplying the parameter map within the evaluated implementation; individual
successes or failures do not uniquely attribute a mechanism.

\textbf{Oracle variants}: the counterfactual physics baked in at matched training budget,
showing that the regime can be learned with that additional information.
This is an implementation-level control, not a universal transfer impossibility.

\textbf{Sign-free port}: dissipation rate free to go negative, dissociating the
\emph{guaranteeing} of the second law from the \emph{representing} of its violation.

\textbf{Informed baseline} (for the discovery claims): the unstructured model \emph{plus}
every input the structured model gets (condition index, drive and field features, the auxiliary
loss). This controls the listed inputs, although architecture, optimization,
and the use of those inputs can still affect the comparison.

Models are $8.6$--$11.5$k parameters throughout the behavioural program (grid:
$8{,}711$ structured, $8{,}591$ baseline, $11{,}306$ non-factored); the scaling study
(\S\ref{app:scaling-probe}) is where capacity varies. The unstructured baseline is
capacity-matched to a disclosed few percent and certified by a three-rate
learning-rate sweep, rollout-stabilizing noise, and the best competent configuration
reported. The informed baseline used for the discovery claims is the unstructured
model plus every input the structured model receives (condition index, drive and field
features, and the auxiliary energy loss).

\paragraph{Training.} All models minimise one-step prediction error on per-dimension
standardized states with Adam for $120$--$400$ epochs by regime. Ports use a
$10\times$ learning rate and, in the shared-content protocol, an $L_1$ penalty of
$10^{-5}$ on the rates and a $50$-step episode-energy match read out with the reference
energy functional; that last term is supplied supervision, which the label-free
variant removes entirely.

\paragraph{Metrics.} With $\hat x_{0:T}$ a model rollout, $x_{0:T}$ the reference from
the same initial state, $E$ the true energy of a state, and $\hat\sigma$ a learned
rate,
\begin{equation}
  \mathrm{nMSE}=\frac{\lVert\hat x_{0:T}-x_{0:T}\rVert^2}{\mathrm{Var}(x_{0:T})},\qquad
  D(t)=\frac{\lvert E(\hat x_t)-E(\hat x_0)\rvert}{\lvert E(\hat x_0)\rvert},\qquad
  \epsilon_\sigma=\frac{\lvert\hat\sigma-\sigma\rvert}{\sigma},
  \label{eq:metrics}
\end{equation}
where the energy drift $D(t)$ is reported as the per-seed median and the worst
episode. \emph{Regime match} is the fraction of held-out episodes whose orbit sense
agrees with the true counterfactual, and the stochastic second law is scored by the
stationary temperature and an energy-histogram Wasserstein distance. Near-binary
outcomes use seed-level exact binomial tests, since episodes within a seed are
correlated. Headline numbers use five seeds, the sign-flip and unified-model results
ten, and the rate-transfer and anti-dissipation arms three; $\pm$ denotes a $95\%$
$t$-interval over seeds. Every finding in the appendices carries an evidence label,
replicated, directional, or boundary, whose criteria are given in
App.~\ref{sec:program}, \ref{app:record}, and \ref{app:integrity}.

\section{A structural criterion for counterfactuals, with its refusal tests}
\label{sec:criterion}

\subsection{A sufficient construction and an unrestricted-class limit}
Suppose a true generator and a learned generator share the same supplied
parameter dependence. Accurate learned content then transfers through that
dependence at new parameter values, on the state region where the content
remains accurate. Multiplicative coupling is the construction tested here.
Known nonlinear maps can also transfer, so separability is not necessary.

A different statement concerns unrestricted functions of state and parameter.
Observations restricted to the training support cannot distinguish functions
that agree there and differ off support. This is not a converse to the
sufficient construction. Leaving marginal parameter support defines the
extrapolation tests here; in-support interventions and novel combinations
remain meaningful tests of interpolation and composition.
These test categories are not identified with a causal positivity condition.
Proposition~\ref{prop:B} (App.~\ref{app:theory}) states the two mathematical claims
separately. The refusal experiments measure failures of the tested
implementations, not an impossibility for every nonlinear parameter map.

\subsection{Coverage and separation for the four-way selector}
\label{sec:coverage}
The selector that turns the criterion into a decision rests on a finite-sample
guarantee rather than on an informal equivalence. Two statements make this
precise. The first bounds the chance that the true in-family candidate is
retained; the second states when empirical risk separates the candidates and when
two candidates are indistinguishable in principle.

The first is the coverage statement.
\begin{proposition}[Marginal true-candidate retention]
  \label{prop:fourway}
  For candidate $k$ let $s_{k,1},\ldots,s_{k,n_k}$ be calibration scores,
  $m_k=\lceil(n_k+1)(1-\alpha)\rceil$, and $q_k$ the $m_k$th smallest calibration
  score, with $q_k=+\infty$ if $m_k>n_k$. Define the retained set
  $\mathcal C(D)=\{k:s_k(D)\le q_k\}$. If the new score of the true candidate $k^*$
  is exchangeable with its calibration scores, then
  $\Pr\{k^*\in\mathcal C(D)\}\ \ge\ 1-\alpha$.
\end{proposition}
\begin{proof}[Proof of Proposition~\ref{prop:fourway}]
  With continuous scores, exchangeability makes the augmented rank uniform.
  The selected order statistic excludes the new score with probability at
  most $\alpha$. With ties, accepting scores at or below the threshold
  gives the same conservative lower bound; randomized tie-breaking is
  an alternative construction \citep{lei2018conformal}.
\end{proof}
Set cardinality defines a decision rule, not a truth certificate.
A singleton can be wrong, a nonempty set can accept an out-of-family world,
and an empty set can exclude an in-family truth. The excitation-and-conditioning
gate detects its specified numerical/data failures, but it does not certify global
optimization. Coverage of the raw candidate set does not imply selective
coverage after conditioning on that gate or on singleton selection.

\paragraph{Reported calibration protocol.}
Calibration uses $n_k=100$ worlds per candidate and $\alpha=0.10$.
The finite-sample construction above therefore uses the 91st order statistic.
The supplied score description combines the held-out residual of a
three-coefficient linear acceleration fit with $0.01$ times squared standardized
distances from candidate coefficient regions. The gate requires condition
number $\le10^6$, at least eight held-out samples, and charge/velocity feature
standard deviations $\ge0.02$. Confirmation uses 50 worlds per label per family.
The implementation realises this specification exactly: the empirical quantile is the
$\lceil(n_k+1)(1-\alpha)\rceil$-th order statistic (the $91$st of $100$), ties at the
threshold are retained by the $\le$ rule, calibration and confirmation worlds use
disjoint seeds with thresholds frozen before any confirmation world is drawn, the score
standardises each coefficient distance by that world's own fitted standard error, and
the gate is evaluated before the risk set is read. Each of these properties is verified
independently, including the order-statistic identity, empirical marginal coverage of
$0.90$ on exchangeable heavy-tailed scores, tie retention, and monotonicity in $\alpha$.
The family evaluations include generator changes. Exchangeability with
calibration is not established by procedural freezing or a new seed offset.
The observed family coverages are therefore empirical shift assessments;
they are not asserted to inherit the nominal $90\%$ guarantee.
The raw family-level confusion matrices and calibration provenance are
needed to evaluate that assumption and quantify each decision's uncertainty.

\begin{proposition}[Separation and ambiguity]
  \label{prop:separation}
  Suppose a finite candidate family has population risks $R_k$. If one candidate
  $k^*$ has margin $R_k-R_{k^*}>2\epsilon$ for every $k\ne k^*$, and every
  empirical risk is within $\epsilon$ of its population value, then empirical
  risk separates $k^*$. If two candidates have identical observable
  distributions on the available support, no selector using only those
  observations can distinguish them uniformly.
\end{proposition}
\begin{proof}
  For every $k\ne k^*$,
  \begin{equation*}
    \hat R_k\ \ge\ R_k-\epsilon\ >\ R_{k^*}+\epsilon\ \ge\ \hat R_{k^*},
  \end{equation*}
  so $k^*$ is the unique empirical minimiser. For the second statement, a selector is
  a measurable function of the observations; if two candidates induce the same
  distribution on the observations, the selector's output has the same
  distribution under both, so no output rule can have error below $1/2$ on both
  candidates simultaneously.
\end{proof}

The second is the second-law guarantee for the dissipative port at the level of the
implemented discrete update.
\begin{proposition}[Quadratic contraction and the isolated OU update]
  \label{prop:discrete2law}
  Let $Q(q)=\tfrac12q^\top Kq$ with fixed $K\succ0$, fixed scalar $\sigma\ge0$, step
  $\Delta t>0$, and $a=e^{-\sigma\Delta t}$. The deterministic update $q^+=aq$
  satisfies $Q(q^+)-Q(q)=-(1-a^2)\,Q(q)\le0$, with equality only if $\sigma=0$ or
  $q=0$. If instead $q^+=aq+\xi$ with independent
  $\xi\sim\mathcal N\big(0,(1-a^2)k_BT\,K^{-1}\big)$, then $\mathcal N(0,k_BT\,K^{-1})$
  is invariant for this isolated update.
\end{proposition}
\begin{proof}[Proof of Proposition~\ref{prop:discrete2law}]
  Homogeneity gives $Q(aq)=a^2Q(q)$, so $Q(q^+)-Q(q)=-(1-a^2)Q(q)$, which is
  non-positive and vanishes only if $a=1$ ($\sigma=0$) or $Q(q)=0$ ($q=0$).
  For Gaussian input with covariance $k_BT\,K^{-1}$, the independent-noise
  update has zero mean and covariance
  \begin{equation*}
    a^{2}k_BT\,K^{-1}+(1-a^{2})k_BT\,K^{-1}=k_BT\,K^{-1},
  \end{equation*}
  so the complete Gaussian law is preserved.
\end{proof}
The covariance reduces to $(1-a^2)k_BT I$ in unit-stiffness coordinates.
The stochastic invariant-law statement requires a fixed rate; a
state-dependent frozen-rate update does not generally inherit it.
No property of a reversible splitting step is needed for the isolated
contraction proof. Composition with other updates requires separate analysis;
numerically, the implemented quadratic-class port never increases the channel
energy at any tested $(\sigma,\Delta t)$, and the fluctuation-dissipation-matched
stochastic step has an analytic stationary-temperature error of $1.7\times10^{-3}$
at $\Delta t=0.1$, falling to $1.7\times10^{-5}$ at $\Delta t=0.01$.

Coordinate contraction need not dissipate an arbitrary learned energy.
For example, the confining energy
\begin{equation}
  E(x)=x^{2}+2\tanh x+2
\end{equation}
increases at $x=-0.5$ under the contraction $\dot x=-x$, with
$\dot E\approx0.286$. This discrepancy persists in the continuous
generator and cannot be removed by refining the time step.
The reported full-model temperature and distribution checks are empirical;
they do not establish a general entropy law, detailed balance, or the
invariant distribution of the complete coupled nonlinear sampler.

\subsection{Composition and extrapolation dissociate three levels of structure}
\label{sec:dissociation}
\textbf{Result (cavity, grid).} Three models, graded structure, two
tests. An unseen \emph{mixed-charge configuration of seen pieces} (composition): the unstructured
baseline is at chance (grid regime match $0.469\pm0.16$); the non-factored control passes
($0.896\pm0.045$, just clearing its $0.90$ gate, disclosed); and the factored model passes
($0.979\pm0.045$). A \emph{support-moving} test (the charge-sign flip): only the factored model
passes, and the non-factored control fails despite its composition success. \decisive{} for the
three-way pattern, on both electromagnetic systems.

\why Composition succeeds without factoring because a field-mediated, shared-per-particle law is
an interpolator on the state manifold: when every marginal (each particle's charge, each field
amplitude) stays in-support, a novel joint configuration is reached by the same local functions
evaluated at seen argument values, so locality converts a combinatorially new scene into
pointwise-familiar inputs. Extrapolation fails without factoring because the flipped charge is a
\emph{point off the marginal support}: the black-box coupling has literally no training signal
distinguishing the true odd-in-$q$ continuation from any other, and gradient-trained
interpolators resolve that freedom by continuing the function they fit, which is what
\S\ref{sec:cf-sign} measures as rotation lock-in. The dissociation is therefore not a mere tally
but evidence for the tested design: supplied locality supports composition,
and the supplied multiplicative channel carries the intervened value in these
experiments. Other correctly specified parameter maps can also transfer.

\subsection{Support-moving transfers: sign, magnitude, rate, gravity}
\label{sec:cf-sign}
\textbf{Charge sign (cavity, grid).} Trained at $q=+1$ in an external
field, tested at $q=-1$; scored at the seed level with exact binomial statistics. Cavity: the
factored model reverses circulation on \textbf{10/10} seeds (pooled 318/320 episodes,
$p=0.001$); the controls are not at chance but \emph{below} it (2/10 and 1/10 seeds). Grid:
10/10 (pooled 315/320), controls at or below chance. The cavity result survives fresh independent
data draws 15/15 (pooled 479/480). \decisive{}.

\why The factored model reverses by construction: $q$ multiplies the coupling, the force is odd
in $q$, so the flip is a symmetry of its learned action rather than a datum to fit. The controls'
\emph{below}-chance behaviour is the mechanism made visible: all $q=+1$ training orbits rotate
one way in the external field, so the black box learns ``rotate that way'' as a property of the
\emph{dynamics}, not of the charge input; on flipped charge it keeps rotating the trained way
and is therefore systematically wrong, not randomly wrong. (This lock-in is
  configuration-specific, since it requires a single global rotation sense to exist, and we scope it
as such.) The dissociation between at-chance and below-chance matters later: in the scaling
study the same lock-in signature appears at every capacity (\S\ref{app:scaling-probe}).

\label{sec:cf-fidelity}
\textbf{What transfers at what fidelity (grid, cavity).} The flip
shows a reported $0.8\%$ energy error across the tested horizons. Energy is even in $q$,
so the factored energy functional is flip-invariant by construction. Orbit \emph{kinematics}
(speed, kinetic energy, gyration radius) transfer in-envelope ($4.6\%$ speed error at 25 steps)
and degrade out-of-band ($36\%$ at 200): the flipped trajectory eventually leaves the band of
states the learned content saw. This is the same envelope limit that magnitude extrapolation
meets (below), \emph{not} phase decorrelation, and it is system-dependent: the cavity holds a
never-seen $|q|=2$ to horizon 1000 (speed error $0.027\%\to1.9\%$; energy error $0.12\%$; the
trajectory phase decorrelates while every phase-insensitive observable transfers), and widening
the training range changes nothing ($1.00\times$), so no benefit from widening is detected in this tested range. The
grid's $|q|=2$ divergence is therefore grid-specific (its plausible mechanism, coupling to
grid-scale modes at higher particle speed, is disclosed as an untested hypothesis).
\decisive{} for the transfers; the envelope is a \boundarylbl{}.

\why Proposition~\ref{prop:B} separates a correctly supplied parameter map from content
accuracy on the visited states. Only when both are correct, with the stated
regularity assumptions, does its exact-generator conclusion apply.
Factoring alone does not imply accurate off-support content, conserved-energy
error, or trajectory fidelity. The different empirical fidelity patterns are
consistent with this distinction but are not predicted quantitatively by the proposition.

\textbf{Charge magnitude (cavity).} At never-seen $|q|=2$ the factored model matches
an oracle trained on $|q|=2$ physics to a ratio $\approx1.09$ (nMSE arm), while the non-factored
control does not merely degrade: it \emph{reverses its rotation} (gyration-frequency ratio
$-0.46$ against the correct $+1$), the unidentified-off-support signature again. A
pre-registered gyration-frequency pass criterion was \emph{retired} when the oracle itself
failed it on 4/5 seeds (chaotic phase makes the frequency estimator ill-posed), a failed
pre-registration of a \emph{metric}, reported as such, with the verdict resting on the surviving
arm. \decisive{} (nMSE arm).

\textbf{Dissipation rate (cavity; fluid).} A rate is linear at the
generator level ($\dot E=-\sigma E$), hence factorable even though ``friction'' feels
non-invertible. Trained at $\sigma=0.5$ the port transfers to unseen $\sigma=0.25$ with rate
error $0.0017$, and the fluid's viscosity transfers $\nu=0.5\to0.25$ with rate error
$2.9\times10^{-8}$. The learned port scale is $1.0045\pm0.0007$, a $+0.45\%$ systematic with
the truth outside the interval, disclosed rather than rounded away. \decisive{}.
\why The common intuition that dissipation cannot be extrapolated conflates the \emph{rate} (a
linear generator coefficient, factorable) with the \emph{sign} (bounded by the second-law
guarantee, \S\ref{sec:secondlaw}). The experiment separates them cleanly.

\textbf{Gravity (fluid).} Trained at $g=1.0$, tested at never-seen $g=1.5$:
held-out error $6.8\times10^{-11}$, indistinguishable from in-distribution
($7.7\times10^{-11}$), showing agreement to the reported numerical precision; the non-separable control fits
in-distribution ($3.2\times10^{-5}$) and fails extrapolation at $1.06$, $\sim10^{10}\times$
worse. \decisive{}.
\why $g$ multiplies the potential term $\tfrac12 h^2$ in the learned energy,
providing the required parameter map in this configuration. The very small
reported errors do not establish exact transfer for general fluid flows or
identify smoothness or absence of chaos as their cause.

\subsection{Implementation-level counterfactual stress tests}
\label{sec:falsifiers}
We also evaluate interventions whose correct parameter dependence is not
supplied to the tested model. The reported pre-registered predictions are that
the model's counterfactual error increases and an oracle control succeeds.
The oracle uses the same class and matched budget but includes or trains on
the intervened physics. Success of this control shows that the configuration
can be learned with the additional information; it does not prove that the
intervention is unidentifiable for every model or isolate a unique failure cause.

\begin{center}
  \resizebox{\linewidth}{!}{%
  \begin{tabular}{lllll}
    \toprule
    regime & unsupplied parameter map & refusal & oracle control & label \\
    \midrule
    cavity & particle-shape width (inside a Gaussian) & $111\times$ & $1.4\times$ & \decisive{} \\
    fluid & pressure exponent $\gamma$ (inside $h^\gamma$) & $\sim10^{9}\times$ & $10^{-6}$-level fit & \decisive{} \\
    grid & particle-shape width & $7.0\times$ (bar: $10\times$) & $0.34\times$ & \directional{} \\
    \bottomrule
  \end{tabular}}
\end{center}

Cavity: the width enters every field-gather and current-deposit inside a Gaussian; the
$\sigma_s\!=\!0.10$-trained model misses $\sigma_s\!=\!0.16$ physics at $111\times$ its
in-distribution error while the oracle-kernel control reaches $1.4\times$ (a first oracle run
  at half the training budget failed its bar; the confound was found and the control re-run
size-matched, on record). Fluid: trained at $\gamma=2$, refused at $\gamma=1.7/2.3$ with errors
$0.75/0.27$ ($\sim\!10^{9}\times$ in-distribution) while the $\gamma$-oracle fits each refused
law at $1.3\times10^{-6}/3.7\times10^{-7}$. Grid: refused at $7.0\times$, real but below the
pre-registered $10\times$ bar, reported \directional{}, with the milder effect plausibly
reflecting the gentler generator perturbation (a smoothing-order step rather than a kernel-width
shift).

\why These interventions change parameter dependence that the tested model
does not correctly supply. The resulting failures delimit this implementation.
They do not establish that nonlinear parameter dependence is intrinsically
unidentifiable: known or restricted nonlinear maps can support transfer.

\begin{figure}[t]
  \centering
  \includegraphics[width=\linewidth]{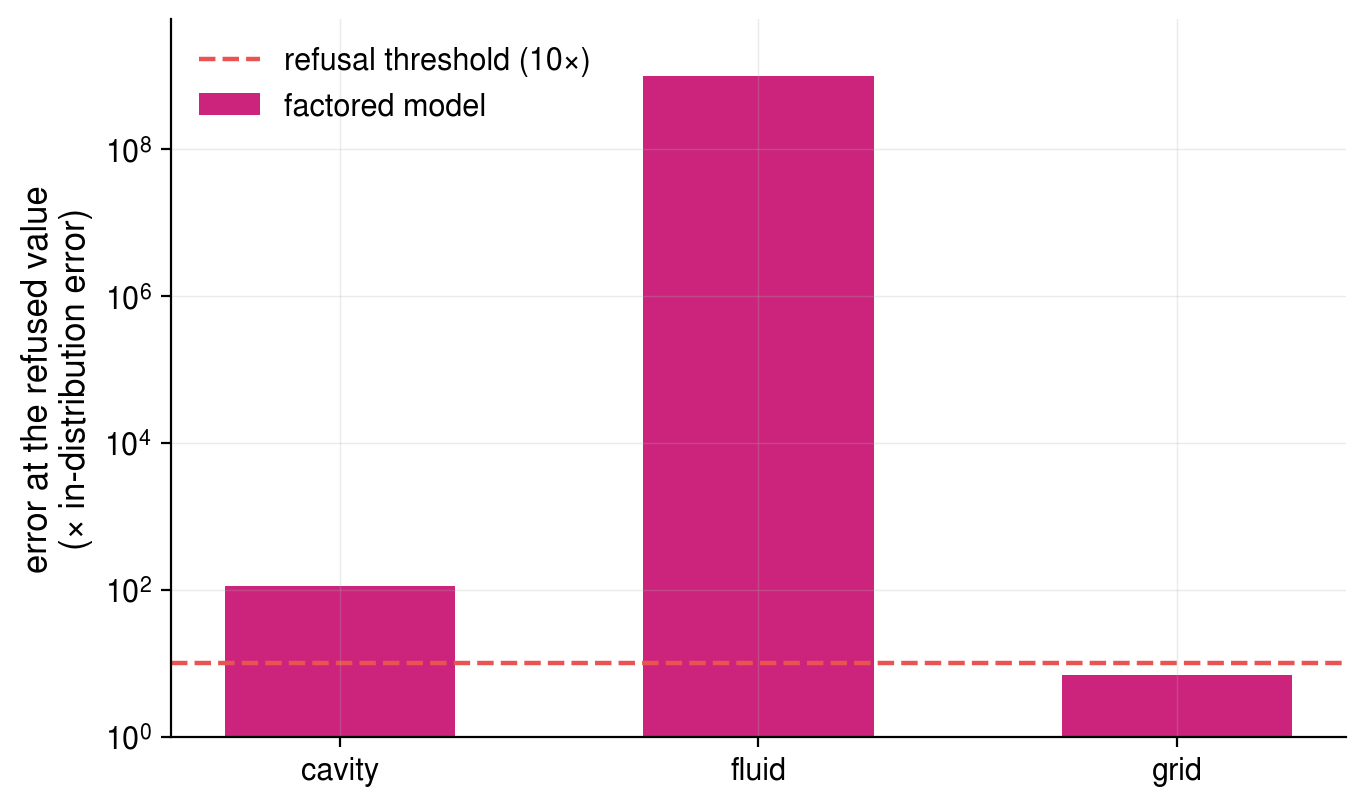}
  \caption{Refusal tests with oracle controls. Bars: the factored model's error at the
    refused intervention, as a multiple of its in-distribution error (log scale), in the
    three regimes; the dashed red line is the $10\times$ refusal threshold fixed before the
    run. The numbers and the matched oracle controls are discussed in
    \S\ref{sec:falsifiers}.}
  \label{fig:falsifier}
\end{figure}

\textbf{An empirical fit-gap alarm.}
A one-degree-of-freedom oscillator test compares a trained quadratic-plus-bounded
energy parameterization on a quartic truth with a quadratic control.
The one-step fit error is $2.7\times10^{-3}$ versus $9.0\times10^{-8}$,
and the fitted curvature at $x=2$ is $-4.5$ rather than $13.0$.
A bounded correction can reproduce quartic behaviour on compact support.
Consequently, these results establish fitting difficulty for the tested
implementation, not global representational impossibility or an effect
caused by coercivity.

Across the reported quartic-strength sweep, the fit gap grows, and the
structured model changes from $42\times$ better than an
equal-capacity generic model at $\beta=0$ to $213\times$ worse at $\beta=4$.
This is an empirical alarm. Optimization error and differences between
nonnested function classes can also produce such a gap.
The single quadratic control does not establish a general false-positive rate.

A complementary ambiguity test compares $G$ with $|G|$ on positive
training support. These maps are observationally equivalent there and
can differ under a negative intervention. The reported procedure returns
ambiguity in 40 of 40 such development worlds. This outcome does not
establish universal refusal: the frozen confirmation selector also
accepts some out-of-family worlds.

\begin{figure}[t]
  \centering
  \includegraphics[width=\linewidth]{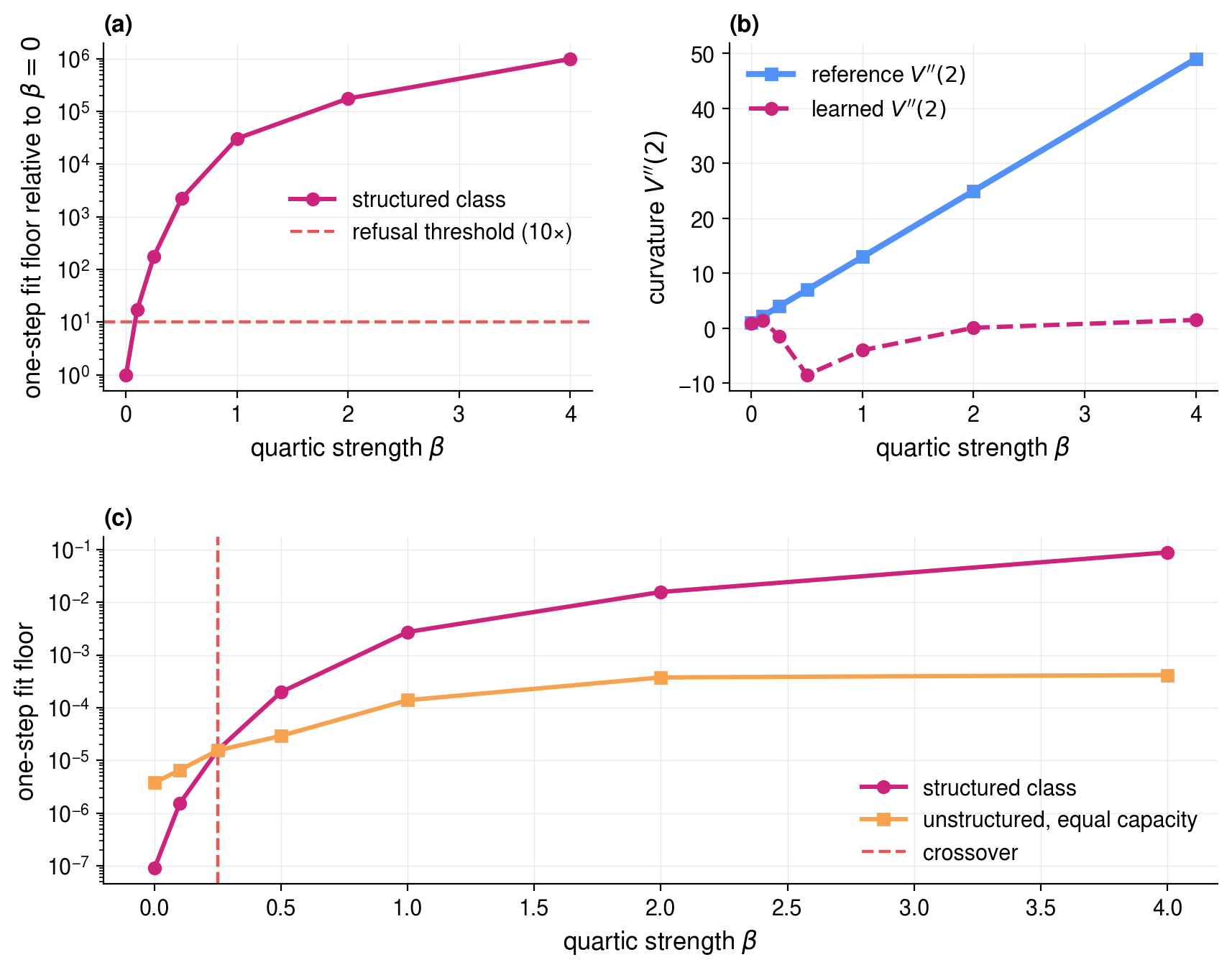}
  \caption{Fit-gap diagnostics on the quartic test, for a confining-quadratic energy class
    trained on a quartic world of strength $\beta$. (a) One-step fit floor relative to
    $\beta=0$ with the $10\times$ threshold (dashed). (b) Curvature $V''(2)$ of the
    reference (blue) and of the learned energy (dashed magenta). (c) One-step fit floor of
    the structured class (magenta) against a capacity-matched unstructured network
    (orange), with the crossover marked (dashed vertical line). The three panels are read
    in the text below.}
  \label{fig:g36}
\end{figure}

\section{What long-horizon stability structurally requires}
\label{sec:stability}

Long-horizon stability decomposed, experimentally, into two separately necessary conditions,
one about the learned energy's \emph{shape} and one about the coupling's \emph{bookkeeping}, plus
one residual boundary we characterized to its correlate.

\subsection{Condition 1: a confining learned energy}
\textbf{Result (cavity).} With an unconstrained learned field
energy the model passes fit and counterfactuals but fails the stability test: every seed's 50{,}000-step rollout
diverges (deviations to $2.9\times10^{7}$). Imposing a confining form (per-mode quadratics
with learned positive stiffness plus bounded learned corrections, on per-mode-scaled inputs)
takes the same suite from $0/5$ to $5/5$ seeds bounded with a flat drift slope, an eight-order
improvement, with no regression in fit or counterfactuals. \decisive{}.

\why A symplectic integrator conserves the \emph{learned} energy to backward-error order, and
that constrains the state only if the energy's level sets are compact (and the total energy, including the
coupling, is coercive; Prop.~\ref{prop:A}). A tanh-MLP energy flattens far from data: out there, level
sets open up, and the flow, faithfully conserving a flat function, can drift to infinity at
constant learned ``energy.'' Probing the failed models confirmed exactly this geometry (the
learned potential saturates where the divergences escape). Confinement restores compact level sets,
so excursions must return. The learned correction only reshapes the landscape \emph{inside} the
confined region, where data constrain it. The scope is stated in Proposition~\ref{prop:A}: this is
\emph{vacuum-field} stability, since energies unbounded below (gravitational or bare-Coulomb
attraction) are the recorded exception, inherited from physics itself rather than from learning.
One parameterization lesson travelled program-wide: learned physical coefficients must live in
units where their target is $O(1)$ (the raw-input stiffness spanned $0.5$--$128$ and could not be
traversed by training; per-mode scaling fixed it), the same lesson later reappearing in the
dissipation port's $\sigma=\mathrm{raw}^2$ form (\S\ref{sec:discovery}).

\subsection{Condition 2: energy-consistent coupling, the mechanism, measured}
\textbf{Result (cavity, grid, fluid).}
Confinement alone is \emph{not} sufficient: the non-factored control, given the identical
confining energy, still diverges where its \emph{coupling} is black-boxed (4/5 seeds unbounded
on the cavity; diverges on the grid). The fluid is the informative exception that makes the
rule precise: its non-separable control black-boxes only the gravity \emph{entry} while the flux-form
coupling remains energy-consistent, and it stays bounded ($5/5$ seeds, max
$|\Delta E|/E_0=9.2\times10^{-5}$) yet fails extrapolation at $10^{10}\times$, exactly the
dissociation Propositions~\ref{prop:A} and~\ref{prop:B} jointly predict: stability tracks coupling consistency,
counterfactuals track separable entry (the $0/5$-bounded fluid model is the unstructured
baseline, whose coupling \emph{is} black-boxed). The mechanism was measured directly: instrumenting the coupling's two ends (the force on particles and the current into
fields) shows the non-factored model injects $\sim9\times$ the net energy per step of the
factored model under identical conditions, a phantom power source at the particle--field
interface. \decisive{}, and the measurement corrected our own earlier account (the original
failure had \emph{two} components, not one; on record).

\why In the true physics, the force on the particle and the particle's back-reaction current are
two derivatives of \emph{one} interaction energy; their energy exchange cancels identically.
Factoring enforces this by making both ends share one coefficient, so the model \emph{cannot}
represent a force without its energetically matched back-reaction. A black-box coupling learns
the two ends as independent functions; any mismatch between them is a non-conservative loop that
does net work every cycle, and a long rollout is an integrator of exactly such work. This is why
the failure grows with horizon while fit metrics stay excellent, and why no amount of
confinement (which bounds the \emph{learned} energy, not the injected physical energy) can
rescue it. For the fluid the same role is played by the flux form: mass conservation to
$4.5\times10^{-16}$ is structural, and functions as the validity meter that charge conservation
provides in the electromagnetic systems.

\subsection{The residual worst-case energy tail}
\textbf{Result (cavity).} The structured configuration reported throughout is
bounded and trend-free at $100\times$ horizon, but rare episodes spike to $5$--$12\%$ energy
deviation (typical episodes hold $1$--$3\%$), missing a pre-registered $5\%$ worst-case bound by
$\sim2\times$. We report this as a \boundarylbl{} and characterized it. The long-carried
``close-encounter'' attribution was \emph{refuted} by direct test: across 400 episodes the
correlation of closest approach with worst-case deviation is $r=+0.13$ (the wrong sign), and
on attracting pairs that genuinely reach zero separation there is still no correlation
($-0.06$); the strongest actual correlate is \emph{initial kinetic energy} ($r=+0.24$). Both
pre-registered levers fail: multi-step (pushforward) training widens seed variance and costs
fit; $67\times$ close-encounter oversampling changes nothing ($0.124$ vs $0.091$
natural). The grid closes the same bound outright ($0.4\%$ worst case), so the tail is cavity-specific content strain, not a defect of the form.

\why What the refutation teaches is as useful as a fix: the tail is not a rare-\emph{event}
problem (which coverage would cure) but a high-\emph{energy} problem, since fast episodes traverse
the learned content's least-constrained region at velocity, so small content errors integrate
into transient energy excursions that then relax (the excursions return to the $1$--$3\%$ band;
the flow remains confined). Multi-step training fails because its window ($\ll$ the excursion
timescale) never sees the error it is meant to shrink; oversampling fails because the strained
region is reached by \emph{energy}, not by proximity. The mechanism statement is: the
boundary is set by where data constrain the learned content, and its currency is kinetic energy.

\subsection{Metric discipline for stability claims}
Energy-drift alone can be satisfied trivially (a model predicting trivial motion has perfect drift), so every
stability claim co-reports rollout error; boundedness alone hides excursions, so worst-case
per-seed deviations are reported alongside means; and reported means never silently drop
diverged seeds (survivor-conditioning is disclosed wherever aggregation could hide a divergence).
Intervals retain their originally reported construction; they are not
reinterpreted as min--max ranges. Small-seed summaries are descriptive and are
not used as calibrated evidence without a justified uncertainty analysis.
The counterfactual sign-flip tests report their exact binomial tests separately;
non-overlap of error bars is not itself a significance test.
These rules exist because each was, at some point in this program, the difference between a true
and a flattering number (Appendix~\ref{app:integrity}).

\section{Dissipation, driving, and stochastic consistency}
\label{sec:secondlaw}

\subsection{Nonnegative damping and its same-channel representational limit}
\textbf{Result (cavity).} Two ports, identical everywhere except the sign
constraint. The \emph{guarantee} port (rate $\ge0$ by construction) transfers a never-seen
dissipation magnitude with the reported rate error (\S\ref{sec:cf-sign}) and \emph{refuses}
anti-dissipation: asked to run $\sigma=-0.5$, it clamps at zero (rate error $1.03$), failing
\emph{by construction}, as intended. The \emph{sign-free} port represents the
anti-dissipative dynamics faithfully: it tracks the genuinely divergent trajectory (true energy
growth $\sim22{,}000\times$ over 2{,}000 steps) to $\sim1\%$ energy error across 3 seeds. This
is representation fidelity of a divergent process, not a stability claim, and not a
second-law violation by the model of a stable world.

\why The sign constraint restricts the same scalar damping channel:
$\sigma\ge0$ contracts its specified positive quadratic diagnostic, whereas
$\sigma<0$ expands it away from zero. This is not an entropy-production
theorem for everything the model can express. The sign-free control
measures the representational consequence of relaxing that constraint.
External work and thermal noise remain separate sources of subsystem energy.

\subsection{Decay, terminal state, and explicit driving}
\label{sec:decay}
\textbf{Result (cavity).} Trained on dissipative trajectories, the port recovers the
medium's resistivity to $0.4\%$ ($\sigma=0.502$, truth $0.5$), matches the field-energy decay
rate to $0.0038$ relative error and the driven steady state (tail ratio $0.998$; injection
  tracked at $6.04$ against a true $6.25$, a discrepancy of about $3.4\%$); the entropy-production diagnostic flags an engineered violator $5/5$ and the
structured model $0/5$, and itself flags the unstructured baseline. \decisive{}. Metric scope,
disclosed: the decay-rate is computed on \emph{field} energy, since kinetic energy is undamped in
this medium and would swamp a total-energy terminal ratio.

\why Identification is easy \emph{because} the structure is right: the port is a single scalar
multiplying the exact resistive structure, so fitting $\sigma$ is a one-parameter regression on
an exponential, and the architecture converts a behaviour (``obey the second law at the right
rate'') into an estimation problem the data can actually pin. The driven control matters because
a model could ``pass'' dissipation tests simply by damping everything; tracking a known
antenna's injection at $6.04/6.25$ while never injecting through the port shows injection and
dissipation live in the right components (drive vs port), which is the design rule
(represent the known external input separately from passive damping) tested here.

\subsection{The stochastic second law: recovering the correct temperature}
\textbf{Result (cavity).} Under a fluctuating thermal bath
(fluctuation--dissipation-matched noise on a charge-0 linear reservoir), the trained model
relaxes from a cold start to the correct stationary temperature: $T=0.502\pm0.007$ against a
truth of $0.5$ ($0.3\%$), \emph{both} field sectors within $0.4\%$, energy-histogram Wasserstein
distance $0.016$ (gate $0.10$), with independently learned $\sigma=0.497$. Controls: the same
model in a noise-doubled (FDT-mismatched) bath equilibrates near $2T$ and is flagged $5/5$; the
correct model is flagged $0/5$; the unstructured baseline reaches a temperature about $10^{6}\times$ too high
(finite, a gross temperature failure, not a numerical divergence). The reference Langevin system
was calibrated first ($6/6$ tests; Gibbs temperature to $0.6\%$). \decisive{}. Scope: one
$(T,\sigma)$ pair, a charge-0 linear bath; the reservoir's $T$ and its FDT relation are supplied
as known physics, and the claim is that the \emph{learned} $\sigma$ and reversible core are
quantitatively consistent with them.

\why The magnetic sector is undamped and undriven in this test, so its
equilibration probes energy exchange through the learned coupling.
Agreement of sector temperatures and the reported energy histogram is
consistent with the supplied bath in this charge-zero linear setting.
It does not uniquely identify the coupling, prove detailed balance,
or establish the invariant law of the full nonlinear simulator.
The noise-doubled control tests sensitivity to a specified FDT mismatch.

\subsection{The fluid's second law: same structure, different physics}
\textbf{Result (fluid).} Viscous decay rate matched to $1.8\times10^{-8}$; the
no-port control cannot dissipate at all (rate error $1.0$, so the port is necessary content, not
decoration); wind injection tracked exactly (factor $17.4$, equal to ground truth); the violator
flagged $5/5$; anti-friction refused by the guarantee (rate clamps to zero exactly).
\decisive{}. A measurement lesson that generalises: the diagnostic must read the port's
\emph{fluctuation} energy, not total energy, since the constant $\tfrac12 g h_0^2 L$ background would
otherwise swamp the signal, the fluid analog of the kinetic-vs-field scoping in
\S\ref{sec:decay}.

\section{One model that estimates the physics and selects its form}
\label{sec:discovery}

\begin{figure}[t]
  \centering
  \includegraphics[width=\linewidth]{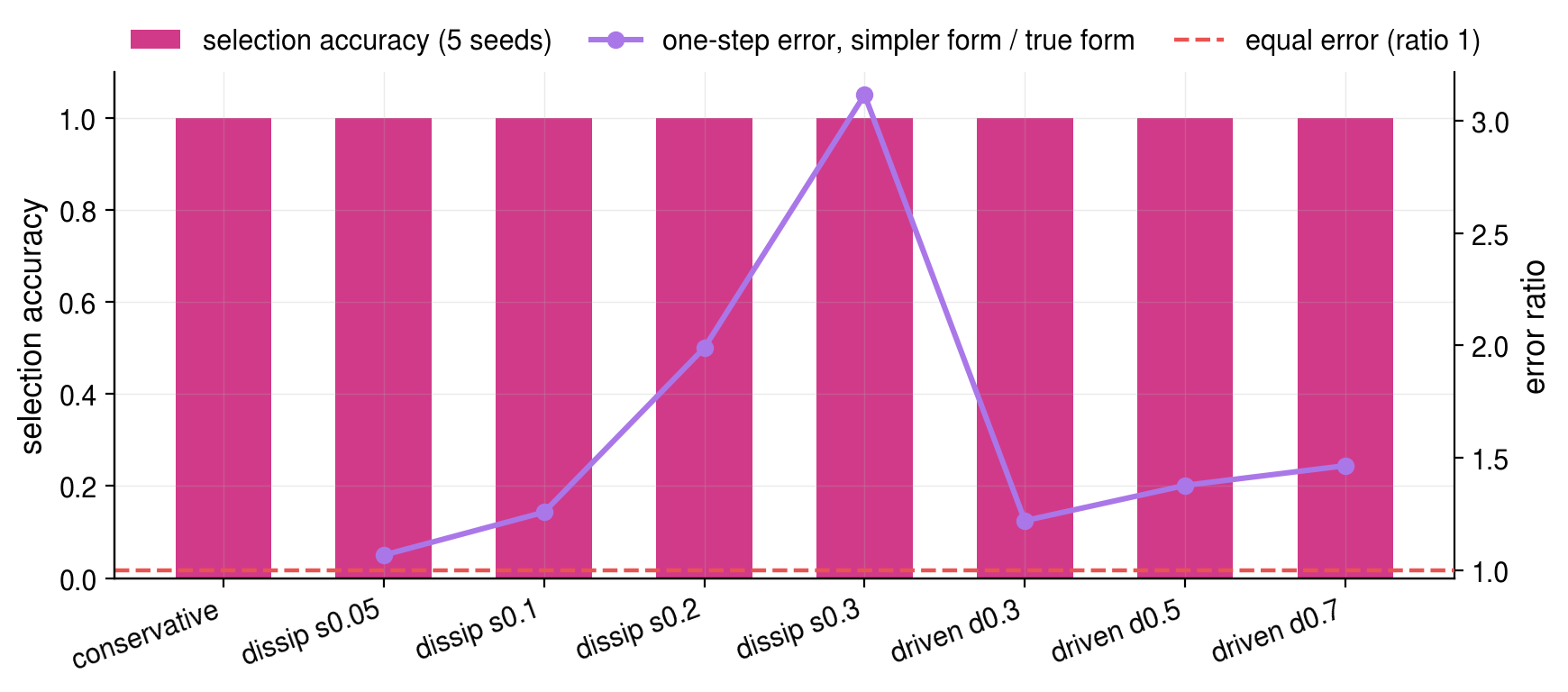}
  \caption{Automated structure identification on the one-degree-of-freedom stress ladder
    (development instrument). Bars: fraction of runs on which the selection loop picks the
    true general form, across a conservative world, dissipative worlds of increasing rate,
    and driven worlds of increasing amplitude. Line (right axis): one-step error of the
    next-simpler form relative to the true form, with equal error as the dashed line. The
    ladder result and the flagship-cavity margins are given in the text.}
  \label{iclr:fig-loop}
\end{figure}

The reversible/irreversible split has a consequence that turns architecture into epistemology:
``conservative'' is not a different model class but the $\sigma=0$ member of the dissipative
family. A single model can therefore \emph{decide from data} whether each world it watches
dissipates, the smallest possible instance of discovering \emph{which laws apply} (one bit per
system: does the second law bind? plus one number: how fast). This section earns that claim in
three escalating forms (supervised by a condition index, tested by rate-clamping controls, and
label-free) and then extends discovery from constants to \emph{functions}.

Section~\ref{sec:form} of the main text states the division of labour component
by component: what the architecture supplies as general form, what the data must
estimate as specific content, and how each estimated quantity is tested. The value
of the approach rests on that division, since a wrong fit of the estimated content
gives wrong numbers under the same imposed form.

\subsection{Supervised discovery: one checkpoint, four worlds}
\label{sec:unified}
\textbf{Result (cavity; 10 seeds; one checkpoint per seed, scored as-is).} One
model (shared learned energy, coupling, and scales; per-condition rates
  $\sigma_c=\mathrm{raw}_c^2$; a 10$\times$ port learning rate; and an \emph{energy-trajectory
  auxiliary loss}, using 50-step windows matching the model's episode energy series to the data's;
  its energy \emph{readout} uses the reference functional, disclosed known-physics supervision,
  given to the informed baseline too, and removed entirely in the label-free variant of
\S\ref{sec:discovery}.4, which is the answer to this soft spot) trained jointly on
conservative, dissipative ($\sigma=0.5$), driven, and external-field episodes, told only
\emph{which condition} each episode belongs to, never what any $\sigma$ is.
It drives the rates apart by over four orders of magnitude, per system, correctly:
$\sigma_{\rm cons}=1.2\times10^{-5}$ (implied spurious energy loss over the full 50{,}000-step
horizon: $0.6\%$); $\sigma_{\rm diss}=0.481$; the external-field condition reaches \emph{exact}
zero. On that single set of weights, the four behaviours hold: composition match $1.0$; the
charge-sign flip $10/10$ seeds ($p=0.001$); stability bounded $10/10$ with flat slope (the $5\%$
  worst-case bound fails on the same characterized tail as the specialist,
\S\ref{sec:stability}; disclosed, not averaged away); second-law decay error $3.6\%$ (median), driven injection
$6.33$ vs true $6.25$, violator flagged $10/10$, model clean. Controls: without the port the
same model cannot fit the dissipative world (decay error $0.99$); an \emph{informed}
equal-capacity baseline (given the condition index, drive and field features, \emph{and} the
auxiliary loss) still fails everything (rollout $\sim10^{6}$, second-law rate error $\sim7$); and the
sign-free port learns a small \emph{negative} conservative rate on 3/5 seeds
($-1.6$ to $-3.2\times10^{-4}$), positing a spurious energy source on ambiguous data, exactly
the failure the one-way port forbids by construction. \decisive{} at the behaviour level. The
fluid reproduces the entire result: $\sigma_{\rm inviscid}=1.4\times10^{-4}$ vs
$\sigma_{\rm viscous}=0.500$, with the sign-mode control showing no spurious
anti-dissipation.

\why Two design decisions carry the identification, and both are general lessons rather than one-off tricks.
\emph{Signal:} a single step constrains $\sigma$ only through the factor
\begin{equation}
  e^{-\sigma\,\Delta t}\approx1-\sigma\,\Delta t,
\end{equation}
a small sensitivity for a short step, while the
\emph{episode-level energy series} carries $\sigma$ with signal proportional to horizon; and
unlike full-state multi-step training (which fails here, \S\ref{sec:stability}), the energy is
smooth and non-chaotic, so rollout-training the \emph{invariant} extracts the long-horizon
signal without the long-horizon chaos. \emph{Parameterization:} $\sigma=\mathrm{raw}^2$ can
reach zero \emph{exactly} (at $\mathrm{raw}=0$), where a softplus approaches it only
asymptotically; the gradient $d\sigma/d\mathrm{raw}=2\,\mathrm{raw}$ itself vanishes at zero,
so the parameterization buys exact representability of $\sigma=0$, not healthy gradients
there, and the training signal for small $\sigma$ comes from the episode-level energy
series (whose sensitivity $\partial e^{-\sigma t}/\partial\sigma=-t\,e^{-\sigma t}$ grows with
horizon), not from a single step, where $\partial e^{-\sigma\Delta t}/\partial\sigma=-\Delta
t$ is small because the step is short. The sign-free control then explains \emph{why the guarantee matters
precisely here}: on nearly-conservative data the likelihood surface around $\sigma=0$ is almost
flat, so an unconstrained port wanders to small negative rates, harmless at short horizon but
compounding at $10^5$ steps. The one-way port converts an ill-conditioned sign inference the
data cannot settle into a constraint physics already settles.

\subsection{Identifiability: the energy cannot absorb the dissipation}
\textbf{Result (cavity).} Clamping $\sigma_{\rm diss}$ at wrong values and letting
every shared component retrain to compensate leaves a clean fit-loss minimum at the truth
($0.25\!\to\!0.012$, $0.4\!\to\!0.004$, $\mathbf{0.5\!\to\!0.0004}$, $0.6\!\to\!0.008$,
$0.75\!\to\!0.015$), with the conservative-condition energy uncorrupted. A $\lambda$-sweep
exonerates the shrinkage penalty for the unified run's residual $\sim4\%$ low bias in
$\sigma_{\rm diss}$ (L1 contributes $-0.08$ percentage points; at $\lambda=0$,
$\sigma=0.50005$): the bias is \emph{protocol-specific} (the four-condition/auxiliary-loss
interaction; mechanism recorded as open), and a prior attribution in our own record was
corrected accordingly. Fresh-data replication corroborates: the cheap two-condition
protocol reproduces $\sigma_{\rm diss}=0.4988$ ($0.2\%$).

\why The sharp referee question (``couldn't the flexible energy quietly absorb the
dissipation, making the learned $\sigma$ decorative?'') has a structural answer made
empirical: a Hamiltonian flow conserves its own generating energy, but can
change a different reference diagnostic. The clamp experiment shows that
the tested shared components did not compensate for the specified wrong
rates under this protocol. It does not establish global uniqueness of
the conservative/dissipative decomposition.

\subsection{Fresh-data robustness and the precision boundary}
\textbf{Result (cavity).} Regenerating all training data with three fresh
\label{sec:robustness}
generation seeds: the sign-flip dissociation survives outright ($15/15$ seeds, pooled
$479/480$). The $\sigma$-identification survives \emph{at the order-of-magnitude level}:
$\sigma_{\rm cons}$ and $\sigma_{\rm diss}$ stay $\sim3.7$ orders apart on every draw, but the
strict pre-registered per-value gate ($\sigma_{\rm cons}<10^{-4}$ on every fit)
marginally fails: $7/9$ values below, maximum $1.07\times10^{-4}$; fresh-draw inference is
$\sim7\times$ looser than the shared-data run, and the worst seed's implied 50{,}000-step
spurious loss is $5.2\%$, marginally over the $5\%$ bound. We keep the gate failure as a
\boundarylbl{} (an early draft that relabeled it was itself corrected; the recorded verdict was
right, Appendix~\ref{app:integrity}).

\why Near-zero rates show greater relative estimation sensitivity in these
runs. Boundary constraints, parameterization, finite-window signal, and
optimization can all contribute, and the experiment does not isolate them.
Absolute errors and their long-horizon consequences are therefore reported
alongside relative rate accuracy.

\subsection{Label-free inference of the medium from observation}
\textbf{Result (cavity).} An encoder reads the first 100 steps and estimates
$\hat\sigma$ using a frozen core's short-rollout prediction loss, without
rate labels. In the reported five-seed experiment, every held-out
conservative estimate is below every dissipative estimate. A sliding window
tracks a medium switch from $0$ to $0.5$ with estimates $0.078$ and $0.457$.
Intermediate-rate correlation is $0.90$ against the declared $0.95$ gate,
with slope $0.775$. The conservative estimate is approximately $0.07$,
and the corresponding rollout error is $0.60$, versus $0.005$ in the
reference-rate comparison. Short and shuffled windows weaken the separation,
although a permutation-invariant cue remains in one seed.

These results demonstrate regime discrimination with residual estimator bias,
not an information-theoretic limit. Direct fitting on the same windows
achieves smaller errors, as reported below. Freezing the core separates
context estimation from content adaptation. Accuracy at two training rates
does not establish accuracy on every intermediate trajectory or state region.

\subsection{Discovering functions, not just constants}
\textbf{Result (grid; fluid).} On the grid, within the declared confining functional parameterization and supplied operators, the learned constitutive derivatives come out \emph{linear}
($r^2=1.0000$) with slopes $f_E'=1.043$, $f_B'=1.001$: the model discovers that vacuum
electromagnetism's energy is quadratic, and reports their slope ratio,
$r_{\rm slope}=a_B/a_E=0.959\pm0.017$ (reference ratio $1$).
Its interpretation as squared wave speed depends on the field-variable and
bracket normalization, which is not derived in this report.
We therefore retain it as a constitutive-slope diagnostic rather than an
independent wave-speed measurement. The approximately $4\%$ E-side discrepancy
is reported with its direction. Smoothing and discretization are plausible
contributors, but their causal contribution is not established here.
The slope-regression estimator was
pre-registered (a pointwise-ratio probe bug, which diverges near zero field, is on record with
its correction). In the fluid, the learned mass-metric and potential-derivative are linear with
slopes $1.000$ ($R^2=1.000$); since gravity is supplied, this yields an effective
coefficient $1.000\pm5\times10^{-6}$ per unit supplied gravity, not an independent
gravity measurement. \decisive{}, with the
scope stated where it belongs: the reported estimates agree with these tested truths within the supplied class.
Class membership alone does not ensure identifiability or successful optimization.

\why Why can an energy \emph{functional} be recovered from one-step data at all, and modulo
what? With the update rule pinned to the curl-form stencils, observed field increments constrain
the operator image $L\nabla E$, not $\nabla E$ itself: any admissible $C$ with $L\nabla C=0$
(Casimirs of the fixed bracket, additive constants, and any direction the curl operator
annihilates) leaves the dynamics unchanged, so $E$ is identifiable only up to that equivalence
class, and only on the states the data visit. We distinguish operator-null directions from observable quantities:
a product diagnostic removes one scaling ambiguity, and spatial derivatives
remove additive force constants. These operations do not establish a unique
underlying gauge. The gauge test below exhibits distinct admissible functions
with identical trajectories, which is why the nonlinear-recovery errors are
reported gauge-aligned (mean difference removed) alongside the raw values. The
amplitude sweep broadens the tested operating support, but does not
establish full rank of the observation-sensitivity operator. Recovery concerns
the operator-visible content on the sampled states. A high $R^2$ against
known functions is empirical agreement, not an identifiability proof.

\textbf{Gauge test (distinct admissible functions, identical trajectories).} We checked the
equivalence class directly on the oracle-content fluid model. Shifting the potential force by
a constant, $V'(h)\to V'(h)+c$ (equivalently $E\to E+c\cdot\mathrm{mass}$, a Casimir direction
of the fixed bracket), leaves every trajectory unchanged to machine precision: over $16$
episodes and $200$ steps the rollout nMSE between the shifted and unshifted models is
$10^{-27}$ for $c\in\{0.5,-0.3,2.0\}$ at both the standard ($p=2$) and polytropic ($p=3$)
truths. Shifting the mass density instead, $M(h)\to M(h)+c$, changes the flux and the
trajectories (nMSE $0.39$ to $5.8$ over the same shifts). The $V'$ intercept is therefore a
genuine non-identifiable direction, which is why only the $h$-dependence (slope and shape) of
$V'$ is the intended recovery target. The tested constant shifts of $M$
are distinguishable; this does not establish uniqueness over all admissible
functions or exclude compensating changes in other learned components.

\textbf{Nonlinear recovery, downstream prediction, and the linear comparator.}
The fluid experiment uses a polytropic truth $p=3$, with $V'(h)=h^2$ and $M(h)=h$, at
excitation amplitudes $a\in\{0.03,0.05,0.10\}$, step sizes $0.01$ and $0.005$, and three
seeds per cell, giving $18$ cells in all. A learned-content
arm is compared with an oracle-content arm (same class and integrator, exact reference
functions) and, at $\Delta t=0.01$ and every amplitude, with an explicitly linear
constitutive comparator
\begin{equation}
  M(h)=\mathrm{softplus}(\alpha)\,h,\qquad V'(h)=\mathrm{softplus}(c)\,h,
\end{equation}
trained on identical data (the tanh corrections zeroed and frozen; its two parameters need an Adam step of
$10^{-2}$ to converge within the same $30$ epochs, after which its loss is flat from epoch
$20$). Function-level comparisons use the uniform grid over $[1-3a,\,1+3a]$. Because the
additive constant of $V'$ is a Casimir direction ($V'+c$ leaves every trajectory unchanged
to $10^{-27}$), the gauge-aligned error removes the
mean difference before comparing, and the operator image is the exact derivative $dV'/dh$
that the momentum equation sees.
\begin{table}[h]
 \centering
 \caption{Nonlinear fluid truth $V'=h^2$ at $\Delta t=0.01$, means over three seeds (seed
 ranges within $\pm0.01$ for function errors and $\pm15\%$ for prediction errors).
 Gauge-aligned and operator-image errors are relative RMSE on the uniform operating range;
 ``best linear'' is the best linear approximant of the truth on that range. Prediction
 errors are nMSE on held-out episodes: the rollout at the training gravity and the $g=1.5$
 intervention.}
 \label{tab:nonlinear-recovery}
 \small
 \setlength{\tabcolsep}{4pt}
 \begin{tabular}{@{}llccccc@{}}\toprule
 $a$ & class & gauge-aligned $V'$ & best linear & $dV'/dh$ & rollout & $g{=}1.5$\\\midrule
 $0.03$ & nonlinear (learned) & $0.0036$ & $0.0024$ & $0.13$ & $1.4\times10^{-4}$ & $1.1\times10^{-4}$\\
 $0.03$ & linear comparator & $0.0024$ & $0.0024$ & $0.052$ & $2.0\times10^{-3}$ & $1.9\times10^{-3}$\\
 $0.05$ & nonlinear (learned) & $0.008$ & $0.0066$ & $0.17$ & $4.6\times10^{-4}$ & $4.5\times10^{-4}$\\
 $0.05$ & linear comparator & $0.0066$ & $0.0066$ & $0.087$ & $5.9\times10^{-3}$ & $5.7\times10^{-3}$\\
 $0.10$ & nonlinear (learned) & $0.025$ & $0.025$ & $0.24$ & $1.4\times10^{-2}$ & $1.8\times10^{-2}$\\
 $0.10$ & linear comparator & $0.025$ & $0.025$ & $0.17$ & $3.6\times10^{-2}$ & $4.2\times10^{-2}$\\
 \bottomrule\end{tabular}
\end{table}

Across the $18$ cells the raw relative RMSE of the learned $V'$ against the truth is
$0.041$/$0.053$/$0.064$ at $\Delta t=0.01$ and $0.044$/$0.058$/$0.074$ at $\Delta t=0.005$
for amplitudes $0.03$/$0.05$/$0.10$ (seed ranges within $\pm0.005$); most of it is the
unidentifiable constant (gauge-aligned $0.004$/$0.008$/$0.025$; the linear comparator
learns $V'=2.00h$ with zero intercept, so its raw error of $0.9$--$1.0$ is entirely that
constant). The fitted quadratic coefficient is $0.78$--$1.08$ (truth $1$);
learned $M$ is within $0.016$--$0.042$, whereas the linear comparator, whose class contains
$M=h$, recovers it to $10^{-4}$; the oracle-content arm reproduces the reference to
$10^{-15}$. Every cell is bounded, and the learned arm's long-horizon energy deviation
equals the oracle's and the linear comparator's to two digits ($0.0008$/$0.0021$/$0.0080$),
so stability is not where constitutive content shows. It shows in prediction: the
converged linear comparator's one-step loss is $6\times$/$8\times$/$5\times$ the
nonlinear model's at $a=0.03$/$0.05$/$0.10$ ($8.1\times10^{-4}$ against
$1.0\times10^{-4}$ at $a=0.05$), and its rollout and intervention errors are
$14$--$17\times$ larger at $a=0.03$, $13\times$ at $a=0.05$, and $2.4$--$2.6\times$ at
$a=0.10$. Function-level error on the uniform range does not separate the classes: a
linear law approximates $h^2$ to within $0.24\%$/$0.7\%$/$2.5\%$ there, the nonlinear
model's gauge-aligned error is no lower than that floor, and its derivative error is
larger ($0.13$--$0.24$ against $0.05$--$0.17$; the linear slope estimator likewise returns
$R^2\ge0.993$ on the quadratic truth, so slope agreement alone is not a nonlinearity
test). The nonlinear content that the data constrain is therefore established at the
level of curvature and of prediction on the sampled distribution, not at the level of
uniform-range function error. The step-size check separates numerics from content: at
$\Delta t=0.005$ the intervention error is $3.6\times10^{-5}$/$1.9\times10^{-4}$/$3.3\times10^{-3}$
against $1.1\times10^{-4}$/$4.5\times10^{-4}$/$1.8\times10^{-2}$ at $\Delta t=0.01$, three-
to five-fold lower at essentially unchanged constitutive error, a numerical contribution
that must be held fixed in any constitutive-error-to-prediction comparison.

\textbf{Observation-window inference: encoder versus a direct per-window fit.} On the
identical test windows and the identical frozen core, a direct per-window fit of
$\sigma$ (grid then local refinement over the window prediction error) reaches a
conservative-rate floor of $0.009$ (window $K=25$), $0.011$ ($K=50$), and $0.007$
($K=100$) on clean windows, rising to $0.029$--$0.039$ at $2\%$ observation noise, with
rate correlation $0.90$--$0.98$ across the trained grid; the amortized encoder on the same
inputs floors at $0.051$/$0.058$/$0.070$ with correlation $0.54$/$0.76$/$0.90$, and is
brittle off its training window length (absolute error at $\sigma=0.5$ of $0.40$ at
$K=25$ against $0.02$ for the direct fit). The $0.07$ floor is therefore predominantly
amortized-encoder error in this comparison. Direct fitting achieves errors
near $0.01$ clean and $0.03$--$0.04$ at $2\%$ noise; these are estimator outcomes,
not information-theoretic lower bounds. The grid's approximately $4\%$ slope-ratio discrepancy is a separate
systematic recovery error. Smoothing and discretization are plausible
contributors; the reported agreement does not uniquely establish their
causal contribution.

\subsection{Discovery rung 2: selecting the general form}
The discoveries above identify content \emph{within} a fixed general form; selecting \emph{among}
forms begins to answer who designs the structure. In the lumped one-degree-of-freedom harness
(a Hamiltonian core, a non-negative dissipation port, and an optional explicit drive; learned
inertia, stiffness, and $\sigma$) we score the candidates
$\{H,\ H{+}\text{port},\ H{+}\text{port}{+}\text{drive}\}$ by a pre-specified selection statistic, the
Bayesian information criterion on the one-step fit, $\mathrm{BIC}=N\ln\mathrm{MSE}+p\ln N$. Under a
hardened protocol (eight worlds spanning a $\sigma\!\to\!0$ dissipation ladder and a
$\text{drive}\!\to\!0$ ladder, five seeds each, with $1\%$ observation noise) the minimum-BIC
candidate is the true form on all $40$ runs (conservative$\,\to H$; dissipative$\,\to H{+}$port;
driven$\,\to H{+}$port${+}$drive), a \emph{non-nested} wrong form (dissipation on the wrong variable)
is never selected, and the correct form's one-step error separates from the next-simpler candidate
by $1.1$ to $3.1\times$, the separation shrinking toward the conservative boundary, where at
$\sigma\!\lesssim\!0.05$ selection leans on the complexity penalty rather than the fit, a
characterized near-degenerate boundary. The general model's learned diagnostics agree
($\hat\sigma=0.305$ against a true $0.30$; the drive term switches on only when driven). Scored this
way, the port-necessity and injection ablations the program already runs \emph{are} a model-selection
statistic. The same harness recovers the lumped constants to $1$--$3\%$ and performs label-free
dissipation discovery ($\hat\sigma=0.002$ conservative versus $0.252$ dissipative), the capability of
this section carried into the one-DoF form. Scope, in the same spirit as the discovery results
above: the worlds are synthetic with the truth inside the candidate set, so the claim is that the
selection statistic recovers the correct form \emph{when it is present}, the selection analog of
the ``selects the correct member'' claim of \S\ref{sec:unified}; $\sigma$ is weakly identified under
a strong drive and near the conservative boundary, disclosed (App.~\ref{app:record}). See
also the lift to the flagship cavity below (Fig.~\ref{fig:g37cav}).

\textbf{Lifted to the flagship.} The same pre-specified statistic, applied to the paper's central
system, the $16$-mode, two-particle EM cavity of \S\ref{sec:systems}, with the
$8.9$k-parameter energy network \emph{shared} across all three candidate forms so only the
structural hypothesis (does it dissipate? is it driven?) is penalized, and, per this program's
standing rule, escalated in seeds before any headline: at three seeds it selected the true form in
all three worlds; \emph{the five-seed escalation softened that} (Fig.~\ref{fig:g37cav}). Wherever
there is dissipative or driven structure to find, selection is decisive and seed-stable: the
dissipative world picks $H{+}$port by $\Delta$BIC $2257$ and the driven world
$H{+}$port${+}$drive by $1054$, with the learned resistivity recovering physics it was not told
($\hat\sigma=0.502$ dissipative, $0.502$ driven; true $0.5$). On the \emph{conservative} world the
five-seed verdict flips to $H{+}$port by $29$ BIC at $\hat\sigma=0.0126$, a fit difference of
$1\%$ of wMSE, swung by a single seed. Stating the degeneracy's scope precisely matters, because
it is a \emph{fit-horizon} artifact, not an equivalence of physics: over the $500$-step
evaluation window $\hat\sigma=0.0126$ sheds only $6\%$ of the field energy (the two candidates
are indistinguishable \emph{there}), but the cavity port is palindromic (field energy decays as
$e^{-2\hat\sigma t}$), so over the $50$k-step stability horizon this paper's stability claims live at,
that same $\hat\sigma$ dissipates $99.8\%$ of it. A rolled-out $H{+}$port selection would be
drastically wrong physics at long horizon. The five-seed verdict is $2/3$, superseding the
three-seed $3/3$ (Appendix~\ref{app:integrity}); the surviving claim is that the flagship
selector finds \emph{dissipation and drive} decisively and recovers their rates, while on
the conservative world the residual learned rate $\hat\sigma_{\rm cons}$ is an artefact of
the one-step objective whose long-horizon consequences make removing it necessary rather
than cosmetic. Adding the energy-trajectory auxiliary loss alone did not remove it
($\hat\sigma_{\rm cons}$ moved only $1.2\times$ against a $10\times$ bar, and the incorrect
selection widened), so the residual is an identifiability boundary of the objective. An
$L_1$ parsimony prior in the unified model restores the correct selection ($3/3$, with the
conservative form $H$ winning that world by $280$ BIC and an independent
profile-likelihood tie-break agreeing) and leaves the identifiable rate unbiased
($\hat\sigma_{\rm diss}=0.502$), but the residual conservative rate is unchanged at the
protocol training budget ($\hat\sigma_{\rm cons}=0.0126$), still failing the stricter
eval-horizon gate. Longer training then removes it: the residual is a finite-training
transient that shrinks by about $11.4\times$ per decade of epochs, and a pre-registered
long-training run crossed the $10^{-4}$ eval-horizon bound at $8.0$--$8.5$k epochs and
reached $\hat\sigma_{\rm cons}=5\times10^{-6}$ by $27{,}500$ epochs, below even the unified
model's $1.27\times10^{-5}$. Plain training at sufficient budget therefore closes the
question with no special mechanism, and the decay accelerates ($\approx t^{-2.5}$) while a
static-regression estimate of any true residual stays at zero throughout, so the
transient falls to meet a floor that never rises. At the selection protocol's budget the
correct selection is carried by the $L_1$ prior and the tie-break; nothing about the
conservative world remains open (App.~\ref{app:record}).

\subsection{The selection loop and its fit-gap alarm}
The selector compares a finite menu of supplied forms. The fit-gap alarm
(\S\ref{sec:falsifiers}) compares the best structured fit against a
capacity-matched unstructured reference under the stated fitting protocol.
In the reported perturbation sweep the ordering reverses at
$\beta\ge0.25$, while smaller perturbations lie in the alarm's observed
blind region. This is evidence for an empirical screening signal, not a
test that certifies whether the true dynamics belong to the menu.

The proposed workflow is to fit the candidates and reference, inspect the
fit gap, and then select or investigate further. An alarm can reflect
optimization failure, measurement error, an inadequate reference comparison,
or candidate misspecification. Its absence does not establish class
correctness. Candidate selection can therefore be wrong, including accepting
an out-of-family world. Expanding or repairing the menu is a subsequent
scientific decision, not an automatic consequence of a correct diagnosis.

The reported ingredients include $40/40$ selections on the $1$-DoF stress
ladder and strong BIC separation for dissipative and driven cavity cases
($\Delta$BIC $\ge1054$, five seeds). They also include a conservative
near-degeneracy: the five-seed $H$-versus-$H{+}$port tie-break changes even
though the small port affects evaluation-horizon behaviour. These experiments
use synthetic worlds and small menus. Automatic proposal of new forms is
not implemented; neither BIC ranking nor the separate marginal conformal
retention guarantee supplies a universal correct-form-or-refusal outcome.

\begin{figure}[t]
  \centering
  \includegraphics[width=\linewidth]{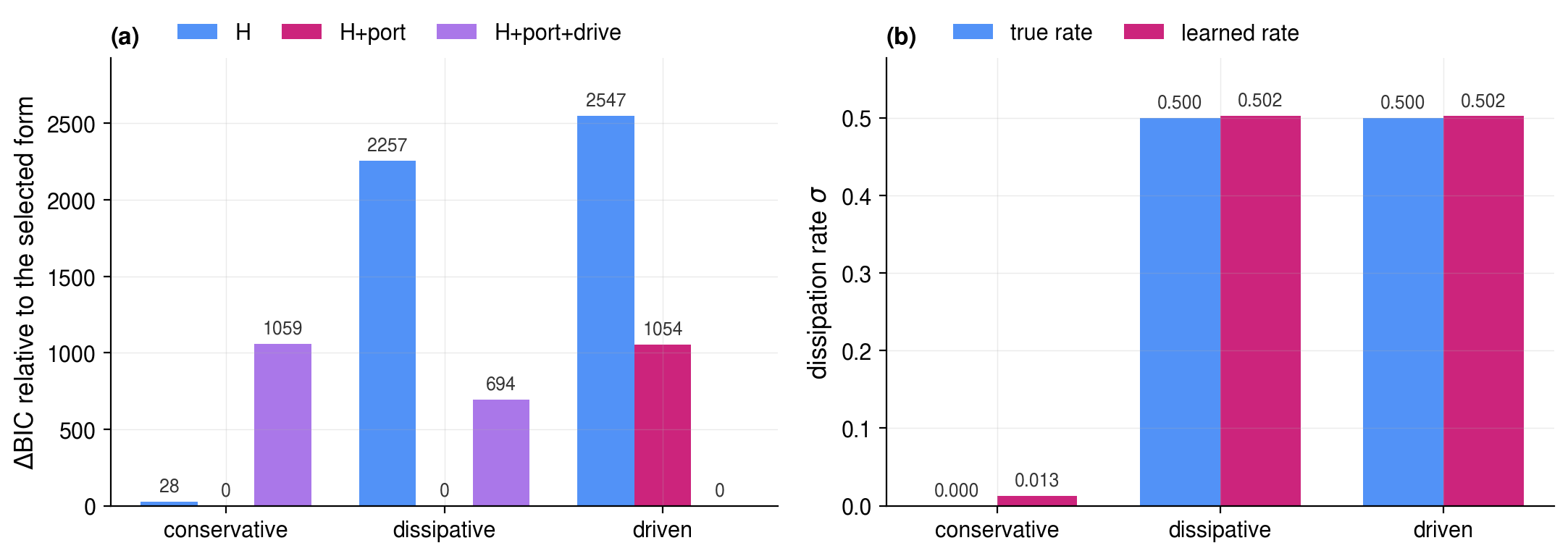}
  \caption{Discovery on the flagship electromagnetic cavity: the same selection statistic
    lifted from the one-degree-of-freedom surrogate to the $16$-mode two-particle cavity.
    (a) $\Delta$BIC of each candidate form ($H$, $H{+}$port, $H{+}$port${+}$drive)
    relative to the selected form, which has $\Delta$BIC $=0$ and so no bar. (b) True
    dissipation rate (blue) against the rate learned by the selected form (magenta). The
    margins, and the near-degenerate conservative tie-break, are discussed in the text.}
  \label{fig:g37cav}
\end{figure}

\section{Generality, noise, and the strongest competitor}
\label{sec:generality}
\label{sec:scaling}

\subsection{The same program in a different physics}
The fluid block (\S\ref{sec:systems}) is not an extra benchmark but the
generality experiment: mass replaces charge as the structural invariant, a scalar height field
replaces vector fields, and gravity and viscosity replace coupling and resistivity, while every
pillar reproduces at the same $\sim$8.7k scale: fit $7.8\times10^{-11}$ with the baseline
failing in-distribution at long horizon ($0/5$ bounded at 10k steps, median-vs-mean disclosure
on record); gravity counterfactual transfer (at the supplied value) with the non-separable control failing at $10^{10}\times$;
machine-precision mass conservation; the complete second-law suite; and the certified refusal test.
\decisive{}. \why The transfer of the \emph{whole program}, including its refusal test, is what
licenses the word ``general'' in our title: the behaviours attach to the imposed form
(conservative flow + one-way port + known drive + factored constants), not to electromagnetism,
and the criterion's predictions (which parameters transfer, which are refused) were made in advance
from the fluid's generator alone.

\section{Synthesis: which structure buys which behaviour, and why}
\label{sec:synthesis}

The findings admit a compact mechanistic summary because the program was built to
dissociate: each behaviour was tested against controls differing in exactly one structural
element. Table~\ref{tab:map} states the resulting map; the paragraphs below give the
\emph{why} in five sentences each, with the evidence rows that pin it.

\begin{table}[t]
  \centering
  \begin{tabular}{p{0.3\linewidth}p{0.35\linewidth}p{0.38\linewidth}}
    \toprule
    \textbf{Structural element} & \textbf{Buys (evidence)} & \textbf{Mechanism in one line} \\
    \midrule
    Locality + field mediation (shared per-object functions) &
    composition; object add/remove; in-distribution competence at $10^4$ params
    where $10^8$ global params fail &
    novel scenes decompose into pointwise-familiar inputs; interaction structure matches the
    generator's \\
    \addlinespace
    Factored (multiplicative) constants &
    support-moving counterfactuals: sign, magnitude, rate, gravity;
    energy-consistent coupling $\Rightarrow$ stability &
    an architectural channel carries values data cannot; one shared coefficient forces force and
    back-reaction to book energy consistently \\
    \addlinespace
    Confining learned energy &
    boundedness at $100\times$ horizon &
    conservation constrains the state only if its level sets are compact; coercive sublevels restrict the state when an energy-sublevel bound holds \\
    \addlinespace
    One-way (non-negative) port &
    quadratic-channel contraction; empirical decay and temperature checks &
    the fixed sign constrains the designated damping channel, not total entropy or arbitrary learned energy \\
    \addlinespace
    Known explicit drives &
    injection tracked, never manufactured &
    gain and loss live in separate, separately-audited components \\
    \addlinespace
    Fixed conservative flow + learned functionals &
    constitutive recovery on tested support; grid $r_{\rm slope}$; fitted rates &
    the operator exposes only its image of the learned derivatives; gauges and excitation limit recovery \\
    \bottomrule
  \end{tabular}
  \caption{The structure-to-behaviour map, with the dissociating evidence. Entries summarize the evaluated controls and their scope; they do not
  establish universal necessity of each architectural component.}
  \label{tab:map}
\end{table}

\subsection{Interpretation and limits of the map}
Locality and shared functions support composition in the evaluated models.
Supplied parameter maps prescribe responses outside the parameter training
support, provided the learned content remains accurate on the visited states.
Neither observation establishes that a generic learned model can never
acquire a useful extrapolation prior.

Architectural constraints restrict representability, but their guarantees
must be stated for the quantities they actually control. The scalar port
dissipates its quadratic diagnostic; the full learned energy and stochastic
stationary behaviour are separate empirical outcomes.
The limits also have different causes. Gauge ambiguity is a property of
the observation operator, while an encoder's residual rate error may reflect
estimation or optimization. A finite-horizon energy tail is not explained
solely by a correlation with initial energy.
The resulting practical lesson is to evaluate the supplied representation,
recovered quantities, numerical update, and downstream task separately.

\section{The stated exceptions, tested, and a suite of everyday physics}
\label{sec:zoo}

\begin{figure}[t]
  \centering
  \includegraphics[width=\linewidth]{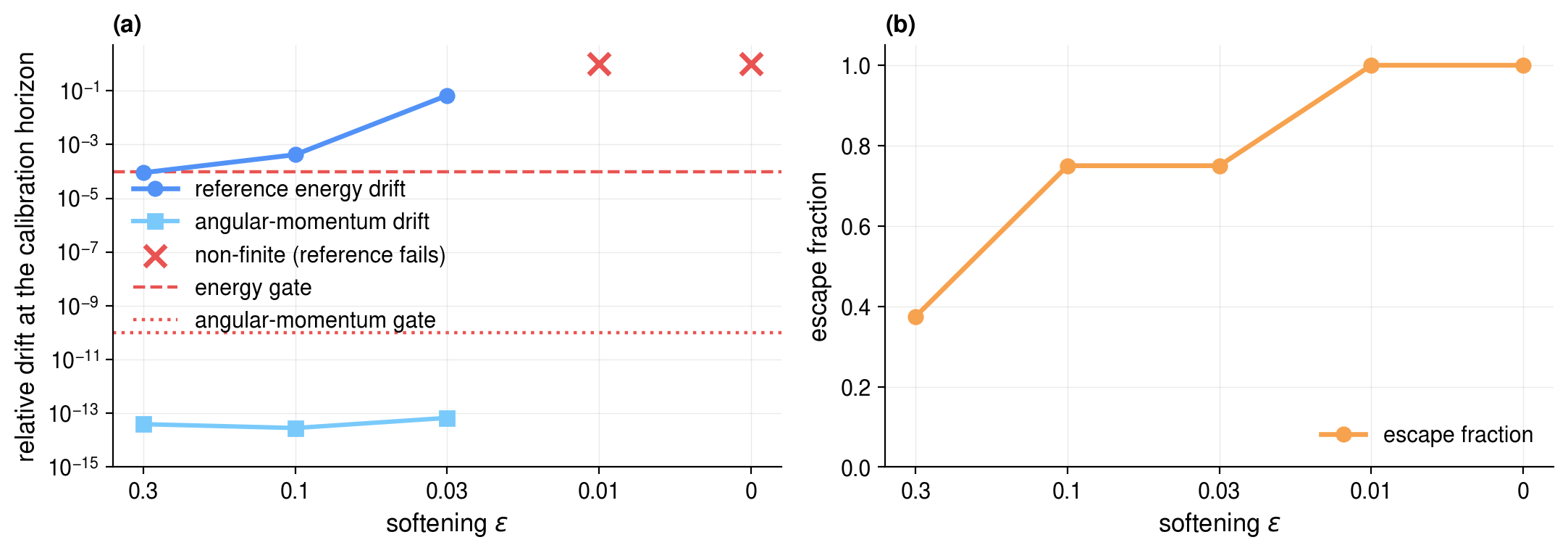}
  \caption{The bare-gravity softening ladder. (a) Relative energy drift (blue circles) and
    angular-momentum drift (light-blue squares) of the reference integrator against the
    softening $\varepsilon$, with the energy and angular-momentum gates (dashed and dotted)
    and red crosses where the reference is non-finite. (b) Fraction of episodes in which a
    body escapes, at each rung. The reading, that the ladder measures the reference's own
    numerical reliability, is in the text.}
  \label{iclr:fig-x1}
\end{figure}

\begin{figure}[t]
  \centering
  \includegraphics[width=\linewidth]{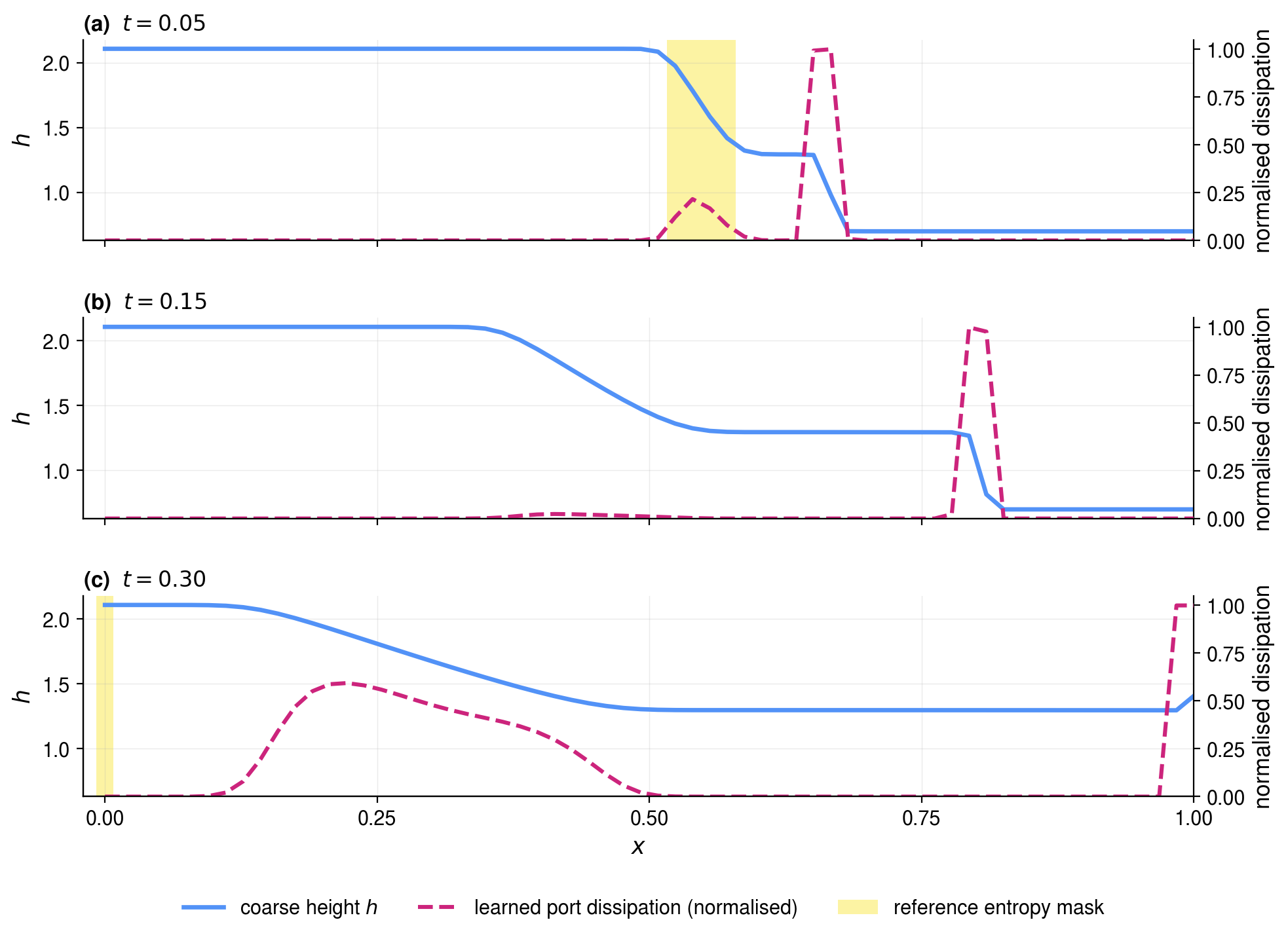}
  \caption{Shock dissipation as a learned, state-dependent, localised port on the
    dam-break fluid at three times (a)--(c). Blue (left axis): the coarse height profile
    $h$. Dashed magenta (right axis): the learned port dissipation, normalised per panel.
    Yellow: the reference entropy mask, which the model never sees. The text describes how
    the learned dissipation tracks the moving shock front.}
  \label{iclr:fig-x2}
\end{figure}

Every scope limit stated earlier in this paper is tested here: bare gravity was
\emph{excluded} by the confinement hypothesis of Proposition~\ref{prop:A}, contact by smoothness, shocks by the
smooth-regime fluid scope, and interacting baths by the earlier linear-bath disclaimer. This section pays
those notes. We test all four carried exceptions plus an eleven-system suite of everyday
classical physics, chosen to \emph{tile} a declared taxonomy (bracket $\times$ port
$\times$ smoothness $\times$ state $\times$ drive), not sampled, so the aggregate claim is
coverage of the tiling, and the taxonomy was committed before any system was built. Each system gets the full pass card of its type (conservative, dissipative,
or chaotic, the three pre-registered variants with the pre-registered constants), a certified baseline (LR
sweep, rollout-stabilizing noise), a discovery gate, and where declared, a refusal-test world.
Table~\ref{tab:zoo-cards} is the ledger; verdicts are printed by the harness, and the
pre-registration (with its numbered amendments, all committed \emph{before} the affected runs)
is in the record appendix.

\subsection{The pass-card ledger: eleven systems, twelve cards, three card types}
\textbf{Result.} All eleven common-life systems earn their pass card
(twelve cards with the pendulum's double-pendulum arm), eight clean and four with exactly one carried
boundary (the pre-registered pass rule), every boundary with a measured mechanism. Denominators used
below: the eleven suite systems plus the four exceptions are the fifteen registry worlds of the
pre-registered $\ge14/15$ selection gate; the selection matrix has sixteen system rows because the
double-pendulum card is scored as its own row, plus two out-of-menu controls. The
conservative cards hold their energy drift below the 5\% worst-case gate at $100\times$ horizons
($2.16\%$ pendulum; $1.4\times10^{-7}$ spinning top, whose Casimir is machine-exact
\emph{by construction}); the dissipative cards recover their rates and long-run states (RLC
  port rate $0.2999$ vs $0.300$; friction $\hat\mu N$ to $0.012\%$; motor power balance
$3.1\times10^{-17}$, an architectural identity, not a fit);
the chaotic cards are graded distributionally and pass climate 3/3 and counterfactual-climate
10/10 seeds while carrying the same boundary twice (the double pendulum, Lorenz-63): the structured model outperforms
the certified baseline at the sub-Lyapunov horizon by only $13\times$ (the double pendulum) or loses to it
($0.38\times$, Lorenz-63) while reproducing the climate correctly, so the split buys long-horizon
accuracy, not short-horizon accuracy.

\why The suite exists to answer whether the program's approach holds up on ordinary
physics, not only on the cavity it was developed on. The answer has a clear shape: structure helps precisely on
the axes the theory says it must (invariants, long-run states, rates, transfer) and is
weaker where it promises nothing (short-horizon trajectory error on chaotic systems).

\subsection{Refusal tests at scale: the refusal quota, and what a refusal instrument must be}
\textbf{Result.} Eleven refusal/guard entries pass across six
non-separability types (the harness tally in Table~\ref{tab:zoo-cards}; the pre-registered
quota was six across four): geometry (a pivot-geometry world hidden in the \emph{kinetic} term refuses the
fixed-kinetic class at $1.5\times10^{6}\times$ its own base floor), exponent (drag law,
$1.1\times10^{4}\times$), state-dependence (nonlinear resistance; conductivity
$\kappa(u)$ at $1.6\times10^{7}\times$; diffusivity $D(u)$), material law (static-vs-kinetic
friction), separability itself (an FPU quartic trained \emph{across} two stiffness settings,
absorbable at any single setting, refused by the factorization), and locality (a fourth-order
stiffness operator vs a nearest-neighbor class). The instrument was refined through
pre-registered amendments: refusal is \emph{within-class} (the class's floor on the refusal test
world over its floor on its own base world), and the learnability certificate is either a
learned right-structured oracle separating by $\ge10\times$ on the same world, or a hand-set
truth member of the oracle class at reference-competitive floor, because cross-class
common-base ratios conflate integrator truncation with class membership (measured: a learned
  Tao oracle read $2\times10^{5}\times$ against a common base while its hand-set truth floor sat
at $0.0031\times$). \textbf{Disclosure:} the refusal instrument was redesigned mid-block
for this measured reason; under the redesign \emph{two verdicts flipped} FAIL$\to$PASS (the pendulum's
geometry and the spring lattice's separability refusal tests, both via truth anchors), and the same redesign
initially \emph{failed} the heat rod's refusal test, forcing the bond-form oracle fix, so the redesign
moved verdicts in both directions. Both pre- and post-redesign results are in the record; the flips changed the
quota tally, not any pass-card count.

\why Three pitfalls we encountered and worked around, kept in the record because any
replication will meet them: (i) \emph{oracle classes must provably contain their refusal test
world} (a pointwise mobility $M(u)\nabla E$ does not contain a flux-form
  $\partial_x(\kappa(u)\partial_x u)$ world, since they differ at $O(\partial\kappa\,\partial u)$, and
  a two-cell pair stencil does not contain the biharmonic's boundary rows, while curvature bonds
contain it exactly); (ii) \emph{refusal-test worlds can be degenerate} (our first
  static-friction world had $\mu_s$ exactly at the drive amplitude, so every trajectory
eventually stuck forever and class and oracle floors collapsed together); and (iii) \emph{optimizer
adequacy is part of a floor instrument} (refusal test floors compare classes, so they run under
  a certified fitting procedure, $2\times$ epochs, cosine anneal, motivated by a measured three-decade
gap between a hand-set truth floor and a fixed-LR fit).

\subsection{The four exceptions}
\textbf{Contact (closed).} The guard-first protocol is the point: the smooth class must
\emph{detectably refuse} the bouncing world against its own soft-wall base before the augmented
class is credited, and it does, after an instructive false start (free fall is a
  non-stationary base world that saturates any bounded-input network; the amendment is in the
record). The complementarity class then recovers contact to machine precision:
$\hat g = 1-3\times10^{-14}$, $\hat e = 0.8-3\times10^{-12}$, per-bounce energy ratio
$\hat e^{2}$ exact, settling time exact, and the augmented class's own restitution is clamped
$\hat e<1$ \emph{architecturally}, the $\sigma\ge0$ lesson at contact, so the
energy-gaining violator lives in the world and is flagged by the meter, not representable by
the model.

\textbf{The interacting thermal bath.} The earlier linear-bath scope lifts: with soft-Coulomb
interactions \emph{on}, the drift is fitted and the noise is \emph{measured} (the measure-the-noise discipline), giving $\hat\gamma$, $\hat T=\hat s^{2}/2\hat\gamma$, and a two-sector
equipartition meter that stays quiet on the FDT bath and flags a sector-asymmetric violator at
$66\%$ mismatch. The magnetic-sector mechanism check survives interactions: a workless force
must be implemented worklessly (a discrete force-kick $B(p\times\hat z)$ heats at
$(hB)^{2}|p|^{2}$ per kick, so reference and model both carry Boris rotations), and
$\hat B$ is recovered with equipartition preserved. (Two thermostat-grade integrator upgrades,
  BAOAB and Boris, were forced by the system's \emph{own} calibration gates, which is what
calibration-first is for.)

\textbf{Bare gravity: a numerical-reliability ladder.}
The reference is evaluated at softening values
$\varepsilon\in\{0.3,0.1,0.03,0.01,0\}$ with an adaptive reference integrator.
Only $\varepsilon=0.3$ meets the stated energy gate, with drift
$9.0\times10^{-5}$ against $10^{-4}$. At $0.1$ the drift is
$4.2\times10^{-4}$, at $0.03$ it is $6.7\%$, and at $0.01$
the numerical integration becomes non-finite. The original escape diagnostic
reports $100\%$ there; this cannot be interpreted as verified physical escape.
Reported angular-momentum residuals near $10^{-14}$ do not restore energy
accuracy or validate a non-finite trajectory.

At the accepted reference setting, the fitted pair function has
$r^2=0.9999999$, the reported baseline contrast is $1.3\times10^6$,
and the $G\to2G$ error is $0.007$ of the reported oracle error.
The worst-case energy deviation is nevertheless $813\%$.
These summaries describe a limited numerical test, not a transition in
physical confinement. Softening removes the collision singularity but
does not make attractive gravity globally coercive.
Symplecticity of an adaptive update requires a specified construction;
it is not established by the integrator label in this record.

\textbf{Shocks and the localized port.} The hardest case is shock dissipation, which
is not a constant rate but a \emph{state-dependent, spatially localized} one. Against a
certified Rusanov reference (mass $4\times10^{-16}$; fine-grid energy strictly decreasing;
Stoker front to $0.8\%$), coarse-grained dam breaks are fit by one flux family whose three
variants differ only in the port: none, constant, or $\lambda_\theta(\text{local features})
\ge0$. The pre-registered localization score (dissipation mass inside the true entropy mask
over the mask's width fraction) measures whether the learned port \emph{finds the front};
the pre-registration commits to either outcome, score $\ge5$ or the representational limit
of coarse local flux stencils on shocks, characterized. Table~\ref{tab:zoo-cards} carries the
verdict.

\subsection{One menu, fifteen worlds, two controls}
\textbf{Result.} The selection matrix (Table~\ref{tab:zoo-matrix}) runs the
\emph{same} menu ($\{H,\ H{+}\mathrm{port},\ H{+}\mathrm{port}{+}\mathrm{drive},\
\mathrm{port\ only},\ (\pm\mathrm{drive})\}$ plus each system's declared bracket variant)
across every world with the frozen selector (min-BIC over \emph{structural} channels, L1 on
  port/drive amplitudes, the clamp-CI tie-break, and the budget rule that refuses to
trust a conservative call whose port sibling has not drained). Both out-of-menu controls fire:
the anti-dissipative world (architecturally outside every member, since $\sigma=\mathrm{raw}^2$
cannot anti-damp) and the delay oscillator (non-Markovian) both return \emph{none-of-these}
via the reference-free alarm. The pre-registered $\ge14/15$ gate is \textbf{missed}: 13/16
rows are correct as reported (12/16 under the strict ``accepted only if'' reading of the
budget rule, which drops bare gravity's self-flagged row, both numbers printed by the harness), and
the miss is recorded in the integrity ledger.

\why The matrix's errors are as informative as its hits, and all three are structural
failures of the \emph{alarm's} scope, not the selector's: on the contact world the smooth
member outright \emph{outperforms} the flexible 4-step-history reference ($0.030$ vs $0.404$,
misfit concentrating in rare event steps, where flexibility cannot collect); on friction the
reference is better but only by $2.2\times$, far under the $10\times$ inversion the alarm needs;
and on the coarse shock world the reference is itself incompetent at this budget. A
floor-ratio alarm requires dense misfit \emph{and} a competent reference; impulsive and
shocked worlds offer neither. Rare-event misspecification belongs to the guard-first protocol,
which fired for all three systems at card level ($3.5\times10^{4}\times$ on friction).
Writing that division of labor down (alarms for dense misfit, matched-structure guards for
events, budget flags for fake-conservative calls) is the section's methodological
yield.

The report-only zoo matrix is superseded as a headline by one canonical selector
evaluation, which the sealed development-then-confirmation protocol
(\S\ref{iclr:program}) freezes before any confirmation world is drawn. The frozen
selector is then scored across four separately implemented generator families,
with the per-family four-way accuracy ($0.84$--$0.89$), true-candidate coverage
($0.727$--$0.900$), selective error among one-candidate calls ($6.8\%$--$15.6\%$),
and out-of-family non-rejection ($2\%$--$24\%$) reported in
Table~\ref{tab:selector-full}. This is cross-implementation empirical validation
with visible selective risk. Table~\ref{tab:selector-full} reports the per-family measures.
Only the true-candidate coverage is the quantity Proposition~\ref{prop:fourway} bounds
(marginally, at $1-\alpha$); the error among one-candidate decisions and out-of-family
non-rejection are measured endpoints that the proposition does not predict, and an empty
prediction set is not a certificate that the truth lies outside the menu, since a true
candidate is excluded with probability up to $\alpha$.

\begin{table}[h]
  \centering
  \caption{Frozen selector confirmation by generator family.}
  \label{tab:selector-full}
  \begin{tabular}{@{}lcccc@{}}
    \toprule
    Measure & Continuous & Discrete & MuJoCo & Adversarial \\
    \midrule
    Four-way accuracy & 0.860 & 0.890 & 0.870 & 0.837 \\
    One-candidate rate & 0.533 & 0.480 & 0.487 & 0.390 \\
    Error among one-candidate decisions & 0.156 & 0.090 & 0.116 & 0.068 \\
    True-candidate coverage & 0.900 & 0.873 & 0.860 & 0.727 \\
    Ambiguity recall & 0.700 & 0.880 & 0.760 & 0.860 \\
    Fitting-failure recall & 1.000 & 1.000 & 1.000 & 1.000 \\
    Out-of-family non-rejection & 0.240 & 0.160 & 0.120 & 0.020 \\
    \bottomrule
  \end{tabular}
\end{table}

\begin{table}[t]
  \centering
  \footnotesize
  \caption{The pass-card ledger. \textsuperscript{b} marks the single boundary the
    pre-registered pass rule allows a system to carry, each with a measured mechanism
    (Appendix~\ref{app:record}). Competence is the rollout error relative to the certified
    history-based baseline; the refusal column is the within-class refusal ratio.}
  \label{tab:zoo-cards}
  \setlength{\tabcolsep}{5pt}
  \renewcommand{\arraystretch}{1.25}
  \begin{tabular}{@{}p{0.3\linewidth}p{0.38\linewidth}p{0.12\linewidth}p{0.215\linewidth}@{}}
      \toprule
      system & verdict & competence ratio & refusal/guard \\
      \midrule
      Pendulum & \textbf{PASS} & 4.23e+04$\times$ & refuse 1527337.4x \checkmark \\
      Double pendulum & \textbf{PASS}\,\textsuperscript{b} (competence vs subLyapunov) & 13.4$\times$ &  \\
      Spring lattice & \textbf{PASS} & 1.78e+04$\times$ & refuse 244.0x \checkmark \\
      RLC circuit & \textbf{PASS}\,\textsuperscript{b} (counterfactual-vs-oracle) & 9.53e+04$\times$ & class 27317.2x \checkmark \\
      Heat rod & \textbf{PASS} & 1.9e+07$\times$ & refuse 16246015.1x \checkmark \\
      Advection--diffusion & \textbf{PASS} & 7.06e+07$\times$ & refuse 2096997.6x \checkmark \\
      Plucked string & \textbf{PASS} & 2.79e+04$\times$ & refuse 3333.8x \checkmark \\
      Drag projectile & \textbf{PASS} & 2.23e+03$\times$ & class 10862.8x \checkmark \\
      Spinning top & \textbf{PASS} & 5.04e+11$\times$ &  \\
      Lorenz-63 & \textbf{PASS}\,\textsuperscript{b} (competence vs subLyapunov) & 0.376$\times$ &  \\
      Dry friction & \textbf{PASS}\,\textsuperscript{b} (counterfactual vs oracle) & 1.47e+06$\times$ & 34683.2x \checkmark \\
      DC motor & \textbf{PASS} & 4.17e+05$\times$ & class 56576.9x \checkmark \\
      Interacting thermal bath & BOUNDARY (stationary-stability, equipartition-meter) &  &  \\
      Bare gravity (ladder) & \textbf{PASS}\,\textsuperscript{b} (worst-case-stability) & 1.31e+06$\times$ &  \\
      Bouncing contact & \textbf{PASS} & 7.01e+14$\times$ & refuse 3512.4x \checkmark \\
      Shallow-water dam break & BOUNDARY (constant-port-guard, no-port-guard, dissipation-rate, competence-ratio, discrete-gravity, counterfactual-vs-oracle) & 31$\times$ & refuse 50.2x $\to$ carried \\
      \bottomrule
  \end{tabular}
\end{table}

\begin{table}[t]
  \centering
  \caption{The selection matrix: one menu, fifteen worlds, two controls
    (the frozen selector = min-BIC over structural channels + L1 + clamp-CI tie-break;
    alarm = reference-free misspecification inversion vs a 4-step-history MLP). Rows print the
  selector's output; errors are kept and mechanised, never relabeled.}
  \label{tab:zoo-matrix}
  \footnotesize
  \setlength{\tabcolsep}{4pt}
  \begin{tabular}{lllll}
    \toprule
    world & truth row & selected & alarm & correct \\
    \midrule
    Pendulum & H & H & -- & \checkmark \\
    Double pendulum & H/H(nonsep) & H(nonsep) & -- & \checkmark \\
    Spring lattice & H & H & -- & \checkmark \\
    RLC circuit & H+port+drive & H+port+drive & -- & \checkmark \\
    Heat rod & port-only & port-only & -- & \checkmark \\
    Advection--diffusion & bracket(skew+port+drive) & bracket(skew+port+drive) & -- & \checkmark \\
    Plucked string & H+port & H+port & -- & \checkmark \\
    Drag projectile & H+port & H+port & -- & \checkmark \\
    Spinning top & bracket(LP) & bracket(LP) & -- & \checkmark \\
    Lorenz-63 & bracket(metriplectic) & bracket(metriplectic) & -- & \checkmark \\
    Dry friction & none-of-these & H+port & -- & $\times$ \\
    DC motor & bracket(portH) & bracket(portH) & -- & \checkmark \\
    Bare gravity & H & H & -- & \checkmark \\
    Shallow-water dam break & none-of-these & port-only & -- & $\times$ \\
    Bouncing contact & none-of-these & H & -- & $\times$ \\
    Interacting thermal bath & H+port & H+port & -- & \checkmark \\
    Out-of-menu control A & none-of-these & none-of-these & fired & \checkmark \\
    Out-of-menu control B & none-of-these & none-of-these & fired & \checkmark \\
    \midrule
    \multicolumn{5}{l}{systems: 13/16 correct; controls: 2/2 fired} \\
    \bottomrule
  \end{tabular}
\end{table}

\section{Discussion: what structure buys, what it costs, where it ends}
\label{sec:discussion}

\subsection{What the experiments support}
Within the supplied representations, the models recover useful constitutive
content and fitted rates on the tested support. Quadratic-channel
contraction is a mathematical property; complete-model temperature,
energy fidelity, and counterfactual prediction remain empirical tests.
The fluid coefficient is reported per unit supplied gravity, not as
independently discovered gravity.

\subsection{What remains conditional}
A wrong prior can fit observed data and fail after intervention.
An empirical fit-gap alarm can also fail or react to optimization error;
the confirmation selector has measured false acceptance.
The nonnegative port excludes the opposite sign of that channel,
but does not exclude driven or stochastic energy increases.
The achieved precision of rate estimators is protocol-dependent, not
an information lower bound. Functional recovery is limited by operator
nullspaces, excitation, and state support. These restrictions motivate
the supplied-versus-learned accounting rather than universal discovery claims.

\subsection{The boundary map}
Stated once, together, as results: (1) the cavity's Stability, worst-case tail, bounded, trend-free,
characterized to its correlate (initial kinetic energy), both candidate levers refuted, $5\%$
bound carried; the grid closes it. (2) State-support envelopes for extrapolation fidelity:
grid magnitude/flip rollouts degrade out-of-band (grid-specific; the cavity holds $|q|=2$ to
horizon 1000 and added coverage changes nothing). (3) Perception: single-frame momentum
invisibility, located, with the minimal multi-frame fix quantified; from-pixels binding and
closed-loop pixel rollouts open. (4) Stochastic thermodynamics beyond one $(T,\sigma)$ and a
charge-0 linear bath: open. (5) The universal scaling question: our scaling-claim evidence is an
existence result over four families, five rungs, and a certified protocol; whether some
training procedure manufactures stable black-box competence at scale, and whether its implicit
bias then locks in or flips, is open, pre-registered, and twice answered
``inconclusive'' by our own escalation rules. (6) Regimes: 2D/shock fluids, contact, and
singular attraction are future physics for this program.

\subsection{Threats to validity}
We audit ourselves the way we expect referees to. \emph{Shared-data seeds}: all headline claims
were replicated on fresh independent data draws; remaining intervals measure optimization
variance and say so. \emph{Metric power}: every counterfactual verdict is conditioned on
trained-setting reliability (the competence-first lesson, applied as a gate). \emph{Survivor bias in
aggregates}: diverged seeds are counted, not dropped, and worst cases are reported beside
means. \emph{Undertrained giants}: the scaling study's competence rungs carry learning-rate
sweeps, plateau evidence with published curves, unlimited-data and longer-horizon retries, and
disclosed compute caps, and where competence still failed we labeled the verdict inconclusive
rather than victorious. \emph{Experimenter degrees of freedom}: thresholds pre-registered,
verdict text generated directly from the pre-registered constants, and the full ledger of corrections (including
  two overturned verdicts, one retired metric, one superseded statistic, and one refuted
mechanism attribution) published in Appendix~\ref{app:integrity}.

\subsection{Outlook}
Three rungs order the future work. \emph{Up the discovery ladder}: from one learned bit (does
this world dissipate?) toward selecting which symmetries and which bracket structure apply, since
the fluid result suggests the program's form survives regime changes, so the selection problem
is the bottleneck. \emph{Down the perception stack}: the momentum boundary is located here and its
minimal fix quantified, since a single frame is velocity-blind and a short window recovers
momentum by differencing (\S\ref{app:perception}); grounding physical state from raw pixels,
including joint learned detection and binding, remains the open direction. \emph{Across the
scaling question}: the experiment our own verdicts demand, a black-box family engineered for
\emph{stable} competence at multiple scales, then the same flip, would either finally earn the
universal leg or discover a generic bias that substitutes for one clause of the criterion;
both outcomes are progress, and the measurement program published here is built to decide
between them.

\subsection{Three distinct boundaries}
First, in-support observations do not identify off-support truth over an
unrestricted function class. Supplied or learned restrictions can prescribe
a continuation; the tested scaling results do not rule out all such learning.
Second, failure to ground the tested frozen-token pipelines is an empirical
limitation, not an impossibility for token-based representations.
Third, a static renderer depending only on position maps different
velocities at that position to the same image. It cannot identify arbitrary
velocities without additional information. Statistical image predictors may
still exploit restricted support or correlations, while multiple frames
supply motion information directly. These are different mathematical and
empirical limitations and should not be combined into three universal
impossibility claims.

\section{Verdict discipline and the integrity ledger}
\label{app:integrity}

\subsection{The rules}
(1) Metrics and thresholds are pre-registered before results exist; printed verdict text is
generated directly from the pre-registered constants rather than transcribed by hand, so every ``PASS'' in a results table is the
output of a function whose thresholds are the registered ones. (2) Every diagnostic ships with a
positive control that must trip; a diagnostic that has never flagged anything is treated as
untested, not as good news. (3) Near-binary outcomes use seed-level exact binomial statistics;
episodes within a seed are correlated and are never pooled as independent. (4) Counterfactual
verdicts require metric power: the model under test must get the \emph{trained} setting right on
held-out data, else its counterfactual score is uninformative and excluded. (5) Headline
verdicts require an escalation (ten seeds, fresh data draws) before any thesis-level guard
moves. (6) A FAIL is reported as a FAIL; a refuted attribution is corrected at its source with a
dated marker, and every claim that a correction was applied program-wide is accompanied by the search that verifies it.

\subsection{The ledger}
Chronological; each entry names what changed, why, and where the full record lives.

\textbf{Two integrator defects (cavity build).} First-order operator splitting (secular
$O(\Delta t)$ energy error) and, introduced by its fix, a tensor-view aliasing bug ($30\times$
worse). Both caught by pre-registered conservation tests before any training; fixed
(palindromic composition; snapshotting). Lesson exported: the halve-$\Delta t$ scaling test
distinguishes integrator orders instantly.

\textbf{A self-force in the PIC build (grid).} A stationary charge felt its own field;
root-caused in two layers (B-spline order pairing with charge-conserving deposition; lattice
self-energy gradient, smoothed $40\times$). Caught before training; without this, every
second-law number on the grid would have carried a phantom dissipation channel.

\textbf{A statistic superseded (grid flip).} Per-episode circulation means with $t$
intervals were retired when a broken control scored a perfect 32/32 episode block, because
episodes within a seed are one correlated draw. All headline flips were re-scored at the seed level with
exact binomials and got \emph{stronger} (cavity 10/10 with controls below chance).

\textbf{A metric retired by its own control (magnitude counterfactual).} The
pre-registered gyration-frequency criterion failed on the oracle itself (4/5 seeds; chaotic
phase makes the estimator ill-posed). Retired as a failed pre-registration \emph{of a metric};
the verdict rests on the surviving nMSE arm.

\textbf{An estimator bug corrected (functional discovery).} A pointwise relative-error
probe diverges near zero field and mis-scored the discovery as a failure; the pre-registered
slope-regression estimator is the instrument of record ($r_{\rm slope}=0.959\pm0.017$, the $\pm$
a $95\%$ Student-$t$ interval over three grid seed estimates $0.951$, $0.963$, $0.964$;
a computed interval, not the observed min--max range). Lesson: measure slopes with
regressions, never pointwise ratios near zeros.

\textbf{A control upgraded (unified model).} The mixed-data baseline was re-run
\emph{informed} (given the condition index, drive/field features, and the auxiliary loss)
so that its failure isolates structure rather than information asymmetry. It still fails.

\textbf{An attribution refuted and propagated (stability tail).} The ``close-encounter''
mechanism carried since the early cavity runs was disproved by direct correlation test (wrong
sign; the correlate is initial kinetic energy) and the coverage lever refuted ($67\times$
oversampling, no effect). Corrected at every occurrence with dated markers; the grep is in the
record.

\textbf{A gate failure kept against temptation (fresh-data $\sigma$).} The strict
per-value $\sigma_{\rm cons}<10^{-4}$ gate fails on 2/9 fresh fits (max
$1.07\times10^{-4}$). An early draft relabeled the recorded ``NOT confirmed'' outcome as a bug; the
correction was reversed (the boolean was right, the prose was wrong) and the runner's prose
is now generated from the pre-registered constant. The order-of-magnitude identification stands; the
strict gate is carried as a boundary.

\textbf{Scaling verdict \#1 overturned (FNO ladder).} ``Extrapolation at chance across
$100\times$ scale'' was retracted by our own review: no rung was competent, the fit-improves
prediction had failed under the fixed-epoch protocol, and a diverged rollout's flip is a coin
toss. Relabeled \directional{}; superseded by the competence-first design.

\textbf{Scaling verdict \#2 overturned (competence-first scaling study).} A 3-seed EARNED
(two competent rungs below chance) did not survive the pre-registered 10-seed fresh-data
escalation: one rung's trained-sign reliability proved seed-fragile (mean $0.867<0.90$) and
dropped out; the recorded verdict is \textsc{inconclusive-single-anchor}. The thesis guard stayed
up; the main text's scaling claim (\S\ref{iclr:scaling}) is the one the surviving evidence
supports. Every stale table was marked superseded; the pipeline that produced each number is
labeled.

\textbf{Flagship selection verdict softened by its own escalation (flagship cavity,
$3/3\to2/3$).} The three-seed flagship lift read DISCOVERED ($3/3$, every runner-up $\ge278$ BIC),
its smallest margin being exactly the conservative cell whose near-degeneracy we had already
characterized. The five-seed repeat flipped that cell to $H{+}$port by $29$ BIC with
$\hat\sigma=0.0126$ (a $1\%$ wMSE difference swung by one seed); the
recorded verdict is $2/3$ (BOUNDARY). The surviving claim is stated in \S\ref{sec:discovery}:
dissipative/driven selection is decisive and seed-stable ($\Delta$BIC $\ge1054$) and the rates are
recovered; the conservative tie-break is fit-degenerate but \emph{not} horizon-degenerate, since
$\hat\sigma=0.0126$ sheds $6\%$ of the field energy over the $500$-step window but $99.8\%$
over the $50$k-step stability horizon ($E\propto e^{-2\hat\sigma t}$, palindromic port), an
arithmetic that corrected our first ``same physics'' framing; the corrected
reading (\S\ref{sec:discovery}) makes removing this residual necessary. Every stale $3/3$/$\ge278$ site was corrected at source, including the figure, whose
annotations are now generated from the verdict JSON.

\textbf{The response-level $\omega$ statistic is a POST HOC promotion, disclosed.} The
odd/even decomposition was pre-registered as a \emph{mechanism probe} to explain the
high-flip anchor seeds, not as a scaling-claim verdict statistic. Its outcome, a clean, seed-robust
separation between the structured model and every trained black box, led us to promote $\omega$
to the mechanism-level half of the scaling-claim evidence. We flag the promotion as post hoc: the
\emph{claim} it supports (black boxes do not acquire the charge-flip symmetry) is exactly the
pre-registered scaling hypothesis, and $\omega$ is an additional measurement of it, not a moved
goalpost; the accompanying random-init control (also post hoc) is what fixes its interpretation
(0.5 is the parity floor, not ``half the symmetry'').

\subsection{Why we publish this}
The ledger is calibration. Each entry is an instance where the
program's rules produced a less flattering number than the first analysis offered, which is
precisely the property a reader needs a measurement program to have before trusting the
numbers that survived. 

\section{Propositions}
\label{app:theory}

The two propositions below carry the boundedness and transfer claims; their scope
clauses are stated inline, and the port and selector propositions are in the main
text with proofs in App.~\ref{sec:coverage}.

\subsection{Energy-consistent coupling and boundedness}
\begin{proposition}[Energy-consistent coupling and boundedness]
\label{prop:A}
\emph{Setting.} The imposed flow class: state $=$ particles $\oplus$ field; dynamics generated
by
\begin{equation}
  H = T(p) + U_\phi(\text{field}) + q\,C_\psi(\text{particle},\text{field})
\end{equation}
through a fixed
symplectic (or symplectic-plus-port) integrator, with $U_\phi$ confining (bounded below,
radially unbounded on the field sector) and the coupling \emph{factored}: one shared learned
coefficient serves both the force on particles and the current into the field.

\emph{Claim.} (i) The exact \emph{conservative} flow (no port, no drive) conserves $H$. If in
addition the total $H$ is coercive, its level sets are compact and rollout energy is bounded for
all time. Coercivity of $H$ is an assumption here, not a consequence of $U_\phi$ confining and
$T$ bounded below: with an unbounded indefinite coupling those do not suffice, since
\begin{equation*}
  T=p^{2},\qquad U=f^{2},\qquad C=-2pf\qquad\Longrightarrow\qquad H=(p-f)^{2},
\end{equation*}
whose sublevel sets are unbounded along $p=f$. We therefore require a
coercivity-preserving bound on the interaction $q\,C_\psi$, or coercivity of $H$
directly. For exact flows, local well-posedness and
regularity on the relevant compact sublevels are also required.
With a drive or a port not aligned to $H$, finite-horizon state bounds
require control of all positive energy contributions, including port work, and the discrete symplectic integrator bounds oscillation of a
shadow energy only under the usual step-size, regularity, and domain conditions (constants in the
formal version). (ii) If the coupling is \emph{not} factored (force
and back-reaction learned independently) there exist arbitrarily small function-approximation
mismatches whose closed-loop work is nonzero per cycle, that is, a phantom power source; no
confinement of $U_\phi$ bounds the resulting physical energy injection.
\end{proposition}
\begin{proof}[Proof sketch]
(i) is conservation of $H$ by the exact flow together with compactness of the
sublevel sets of a coercive continuous function. (ii) Let the force be
$F=-\partial_x(qC)+\delta_1$ and the current $J=\partial_\phi(qC)+\delta_2$ with
independent learned errors $\delta_1,\delta_2$; the power exchanged at the interface
is $\dot x\cdot F+\dot\phi\cdot J$, whose non-conservative part
$\dot x\cdot\delta_1+\dot\phi\cdot\delta_2$ has nonzero circulation along a
periodic orbit whenever $\delta_1$ and $\delta_2$ are not derivatives of one
function, which holds for generic independent approximations of arbitrarily small
norm; the integrated work then grows linearly in the number of cycles regardless of
the confinement of $U_\phi$.
\end{proof}

\emph{Measured instantiations.} (i): the confinement dissociation (cavity and fluid).
(ii): the $9\times$ coupling power imbalance measured at the interface and the
non-factored control's divergence under identical confinement on cavity and grid. The
fluid is the joint instantiation of both propositions rather than a third divergence: its non-separable
control keeps the energy-consistent flux coupling (so Proposition~\ref{prop:A} predicts boundedness;
measured: $5/5$, $9.2\times10^{-5}$) while black-boxing only the $g$-entry (the reported implementation fails the counterfactual test by $10^{10}\times$).

\emph{Scope.} Excludes energies unbounded below (gravitational/bare-Coulomb attraction),
physics' own exception, inherited; vacuum-field stability is the claim.

\subsection{Supplied-map transfer and unrestricted extrapolation}
\begin{proposition}[Supplied-map transfer and unrestricted extrapolation]
\label{prop:B}
Let the true and learned generators be
\begin{equation}
  G_*(x;\theta)=\theta\,g_*(x)+h_*(x),\qquad
  G_\vartheta(x;\theta)=\theta\,g_\vartheta(x)+h_\vartheta(x).
\end{equation}
If both learned components equal their true counterparts on $\mathcal R$,
then $G_\vartheta(x;\theta')=G_*(x;\theta')$ there for every admissible
$\theta'$. This follows by substitution.
With the same initial condition and uniqueness of the differential equations,
the continuous trajectories agree while they remain in $\mathcal R$.
Numerical predictions retain their separate integration error.
Separately, an unrestricted class of functions of $(x,\theta)$ contains
functions agreeing on the observed support and differing outside it, so
in-support data alone do not determine off-support dynamics over that class.
\end{proposition}

This is a sufficient supplied-map construction. More generally, a known
nonlinear map can transport accurate content without multiplicative separability.
The nonidentifiability statement is not a converse
to the sufficient construction, and does not exclude transfer under
additional known or identifiable restrictions.
Parameter-support shift defines our extrapolation tests; it is not a
necessary condition for counterfactual answerability or a causal positivity
assumption.

\subsection{Restriction of the passive channel}
For the same isolated scalar coordinate-damping channel and positive
quadratic diagnostic, imposing $\sigma\ge0$ excludes the anti-damped
channel $\sigma<0$. Representing that sign reversal requires relaxing
the constraint or changing the model. This does not forbid external
work, thermal energy injection, or other driven dynamics.

\section{The complete per-experiment record}
\label{app:record}

This appendix reports one block per finding, in program order, and each block follows
the same structure: the setup, the gates fixed in advance, the result with its per-seed
statistics, the controls, and the verdict class with any boundary the finding carries.
Negative results, refutations, and retired metrics are reported on equal footing with
successes. The per-seed measurements, model checkpoints, and run logs behind these
numbers are provided as supplementary material; the values here are final and should not
be re-rounded.

\subsection*{Cavity reference calibration}
Ring cavity, palindromic symplectic integrator; 16 episodes/config, 50{,}000 steps. Gates:
relative energy drift $<10^{-4}$; Gauss residual $<10^{-10}$ and constant. Result: drift
$9.3\times10^{-6}$ (same-sign) / $1.38\times10^{-5}$ (mixed; binding case, $7.25\times$ margin);
Gauss $\le2.0\times10^{-14}$, exactly conserved. Build defects caught pre-training: first-order
splitting (secular; fixed palindromic), tensor-view aliasing (fixed by snapshotting).
\decisive{} gate; recalibration required on any configuration change.

\subsection*{Cavity behaviour suite: the first pass and its informative failure}
4 models $\times$ 5 seeds. one-step competence: structured recovers content to $0.2\%$. Counterfactual, charge-sign flip:
structured 0.95 vs non-factored 0.33, baseline 0.24 (circulation match). reset-kick ratio
$0.995\pm0.14$; the intervention test $\sim1$. \textbf{Stability FAIL}: all learned models exceed the 5\%
energy gate, since unconstrained learned field energy flattens away from the training data (probed directly: tanh-MLP
saturation where divergences escape), deviations to $2.9\times10^{7}$. Pre-registered overall
verdict FAIL, reported as such; mechanism identified in the confinement result below. Also seeded the counterfactual
informative-fail that became the criterion.

\subsection*{Composition and extrapolation dissociate}
Two counterfactual types on identical models. Unseen mixed-charge configuration (marginals
in-support): passed by both field-mediated models, failed by the baseline. Charge-sign flip
(support-moving): passed only by the factored model. Assigns composition to locality and
extrapolation to factoring; formalized as the structural criterion (\S\ref{sec:criterion}).
\decisive{} (upgraded statistics in the seed-level flip result).

\subsection*{Confinement restores stability; the coupling condition appears}
Confining per-mode energy (quadratic, learned stiffness on per-mode-scaled inputs $+$ bounded
correction). Gates: 4/5 pass (one-step competence; counterfactual primary+composition; reset-kick@1.0 $0.85\pm0.05$; the intervention test
$0.88\pm0.43$). Stability transformed: $0/5\to5/5$ bounded, slope $1.1\times10^{-7}$, but worst-case
episodes $5$--$12\%$ vs the 5\% bound, carried (characterized in the stability-tail result). Same-magnitude
field-bump channel misses the $\le2$ gate ($2.72\pm0.95$), disclosed. \textbf{Non-factored
control NOT rescued} by confinement (4/5 unbounded): stability needs the coupling condition too.
Parameterization lesson: learned coefficients in units where targets are $O(1)$.

\subsection*{The coupling power imbalance, measured}
Instrumented force/current interface under identical confining energy. Non-factored model
injects $\sim9\times$ the net energy per step of the factored model. Also self-corrected the first-pass
account (the original failure had two components: in-band pumping \emph{and} the flat tail).
\decisive{}; the pool-mean is a mechanism localizer, not a rollout number.

\subsection*{Magnitude counterfactual (cavity)}
Train $q=+1$, test $|q|=2$; 5 seeds. nMSE arm: structured $0.096\pm0.035$ vs oracle
$0.088\pm0.036$ (ratio $\approx1.09$); baseline $7.7\times10^{3}$; non-factored $4.2$ with
gyration-frequency ratio $-0.46$ (rotation reversed). Pre-registered frequency-within-5\%
criterion RETIRED: the oracle itself fails it 4/5 seeds (chaotic phase; ill-posed estimator), a
failed pre-registration of a metric. \decisive{} on the surviving arm.

\subsection*{Rate counterfactuals; guarantee vs representation}
Ports trained at $\sigma=0.5$ (label given), tested at $\sigma=0.25$ and $-0.5$; 3 seeds.
Learned scale $1.0045\pm0.0007$ ($+0.45\%$ systematic, truth outside interval). Rate transfer:
error $0.0017\pm0.0034$. Anti-dissipation: guarantee port clamps (rate error $1.03$, a
by-construction refusal); sign port tracks the genuinely divergent truth (energy growth
$\sim22{,}000\times$/2{,}000 steps) to $\sim1\%$; amplification ratio $2.15\pm0.83$ (the
pre-registered $>10\times$ instability prediction NOT confirmed, reported). No stability claim
for anti-dissipation; representation fidelity only. \decisive{}.

\subsection*{Multi-step training does not close the stability tail}
Pushforward training (K-step windows): widens seed variance (worst cases to $0.28$--$0.32$),
mildly costs fit; the structured model remains the reported configuration. Mechanism: training window $\ll$
excursion timescale. Negative, kept.

\subsection*{Deterministic second law PASS}
Dissipative + driven cavity; 5 seeds; port model. $\sigma=0.502\pm0.0003$ (truth 0.5);
field-energy decay-rate error $0.0038\pm0.0021$ (gate $\le0.10$); terminal ratio $0.995$;
driven tail ratio $0.998\pm0.003$; injection $6.04$ vs true $6.25$. Production diagnostic:
engineered violator flagged 5/5; structured 0/5; baseline itself flagged. Metric scope: rates on
field energy (kinetic undamped). Identification note: port under-fit at shared learning rate;
$10\times$ LR group fixed it ($41\%\to2.7\%$ override diagnostic). \decisive{}; deterministic
scope (stochastic half below).

\subsection*{Symbolic regression positioning}
SINDy (STLSQ), generous disclosed dictionary. Clean data: most accurate as positioned (worst-family
coefficient error $0.07\%$, zero false positives). Observation noise: collapses between $0.5\%$
and $2\%$ (particle-force family first; derivative amplification). Dictionary/dimensionality
boundary stated for the grid. A weak-form (integral, test-function) variant with the identical library is far
more robust (worst-family error $0.007/0.028/0.102$ at $0.5/2/5\%$ noise against $0.02/0.23/0.94$ for the
central-difference form). Both sides are reported.

\subsection*{Structural advantages survive observation noise}
Noise sweep over $0.5/2/5\%$ (training-time, inputs and targets). At $5\%$: sign-flip match $1.0$;
one-step fit $\sim$unchanged; stability envelope loosens gracefully ($0.045\to0.285$ at 10k steps);
baseline at chance throughout. Scope: eval-time estimation is the perception test. \decisive{}.

\subsection*{The criterion's refusal test (cavity)}
Non-separable parameter: particle-shape width $\sigma_s$ ($0.10\to0.16$), inside a Gaussian in every
gather/deposit. Pre-registered: (a) structured fails $\ge10\times$; (b) oracle-kernel control
succeeds $\le3\times$. Result: structured $4.78\pm0.02$ = $111\times$ in-distribution
(CONFIRMED); size-matched oracle $0.0607\pm0.013$ = $1.4\times$ (CONFIRMED; a first half-budget
control missed its bar, confound found, re-run size-matched, on record). \decisive{}.

\subsection*{Grid reference calibration}
Yee/Boris/Esirkepov PIC, $64^2$, staggered-consistent gather, binomial smoothing $\times3$.
Gates at 50{,}000 steps: drift $\le4.6\times10^{-5}$ (gate $10^{-2}$); Gauss residual
$3.1\times10^{-13}$ (gate $10^{-10}$). Numerical heating did not materialize. Build defects
caught pre-training: plane-wave \emph{test} direction (simulator right); static self-force
(two-layer root cause; smoothing $40\times$). \decisive{} gate; per-configuration scope.

\subsection*{The criterion reproduces on the grid}
3 seeds plus the external-field arm; equal capacity (8{,}711/8{,}591/11{,}306). One-step competence: structured one-step
$1.6\times10^{-9}$, rollout nMSE $0.044$; baseline broken (LR-sweep certified: best
$3.3\times10^{-4}$, still $\sim2\times10^{5}\times$ worse). Composition: $0.979/0.896/0.469$
(non-factored barely meets the 0.90 gate, QUALIFIED). Stability: structured 3/3 bounded, max
$0.0039$, flat; both controls diverge. reset-kick $1.14$; the intervention test $1.03$. $|q|=2$: out-of-band
divergence beyond $\sim T{=}75$ (boundary; resolved grid-specific below). Flip: superseded
statistics below. \decisive{} pattern.

\subsection*{Functional discovery: vacuum Maxwell}
Learned constitutive derivatives, slope-regressed over $\pm1.5$ field-std (pre-registered
estimator; the pointwise probe bug is on record): $f'_E$ slope $1.043$, $f'_B$ $1.001$, both
$r^2=1.0000$; $r_{\rm slope}=0.959\pm0.017$ (truth 1), a reported $\sim4\%$ E-side slope discrepancy
(discretization), disclosed with direction. In-band scope; linear (vacuum) constitutive law.
\decisive{}.

\subsection*{Grid sign-flip transfer at the seed level}
10 seeds, exact binomial, majority-per-seed. Structured 10/10 (pooled 315/320, $p=0.001$);
baseline 4/10; non-factored 5/10. Transfer is orbit-sense at trajectory level (nMSE $\sim3$ on
non-divergent seeds), resolved further below. \decisive{} (regime).

\subsection*{One model, four worlds (ten seeds)}
Unified cavity model: per-condition $\sigma_c=\mathrm{raw}_c^2$, shared content; L1 $10^{-5}$;
$10\times$ port LR; energy-trajectory auxiliary loss (50-step windows; targets from data;
readout uses the reference energy functional, disclosed known-physics supervision).
$\sigma_{\rm cons}=1.2\times10^{-5}$ (range $3\times10^{-6}$--$2.5\times10^{-5}$; implied 50k
loss $0.6\%$); $\sigma_{\rm diss}=0.481$ ($\sim4\%$ low, protocol-specific per the identifiability result);
$\sigma_{B_0}=$ exact 0. One checkpoint per seed, scored as-is: composition $1.0$; flip 10/10
($p=0.001$); stability bounded 10/10, slope $3.6\times10^{-7}$, but the pre-registered $5\%$ energy-deviation threshold is not met on 6 of 10 seeds, so boundedness and fidelity are reported separately (per-seed max
  $[0.333,0.663,0.052,0.036,0.031,0.044,0.107,0.286,0.057,0.037]$; worst $5.5\times$ the
specialist's, same characterized tail); second-law rate error $3.7\%$, injection $6.32/6.25$,
violator 10/10, clean. Controls: no-port fails dissipation (0.99); informed baseline (condition
$+$ drive $+$ $B_0$ features $+$ aux loss) still fails (rollout $\sim10^{6}$; second-law error
$\sim7$); sign-free port drifts negative on 3/5 seeds ($-1.6$ to $-3.2\times10^{-4}$).
\decisive{} at behaviour level; bound carried.

\subsection*{Cavity sign-flip transfer at the seed level, with controls below chance}
10 seeds. Structured 10/10 (pooled 318/320, $p=0.001$); baseline 2/10, non-factored 1/10,
\emph{below} chance: trained-rotation lock-in (configuration-scoped). $|q|=2$ arm: the
circulation metric is non-discriminating for sign-preserving moves (controls pass it 7--9/10
without the intended mechanism), scope correction applied program-wide; magnitude claims rest on the fidelity arms. \decisive{}.

\subsection*{What the grid flip transfers, horizon-resolved}
Phase-insensitive observables on existing checkpoints (10 seeds, medians): total energy $0.8\%$
at all horizons; speed error $4.6\%@25\to12\%@50\to33\%@100\to36\%@200$; controls fail at every
horizon ($8\times$--$67\times$). Mechanism: out-of-band drift (the $|q|$ envelope), not phase
decorrelation. Upgrades ``orbit-sense only'' to ``energy exactly; kinematics in-envelope.''

\subsection*{The stochastic second law in full}
Thermal cavity (charge-0 linear bath, FDT-matched noise; $T=0.5$, $\sigma=0.5$); 5 seeds;
temperature from the last 3{,}000 of 6{,}000 steps, cold start. Reference calibration first:
6/6, Gibbs $T$ to $0.6\%$. Model: $T=0.502\pm0.007$ ($0.3\%$; gate $10\%$); sectors
$0.502/0.502$ (magnetic sector thermalizes only via the learned reversible coupling);
Wasserstein $0.016\pm0.014$ (gate $0.10$); learned $\sigma=0.497$. Controls: FDT violator
(noise variance doubled, same model) settles $\approx2T$, flagged 5/5; clean 0/5; baseline runs
hot $\sim10^{6}\times$ (finite). Convergence note: an early $40$-epoch run showed a spurious
$28\%$ magnetic over-heat that was pure under-training and disappeared by $120$ epochs;
we record it so that a training-horizon artefact is not mistaken for physics. \decisive{}; one $(T,\sigma)$, known-reservoir scope.

\subsection*{Fluid reference calibration}
1D shallow-water, spectral implicit-midpoint. 50{,}000 steps ($250\times$ training horizon):
mass $3\times10^{-15}$; inviscid drift $1.2\times10^{-7}$ (linear) / $7.5\times10^{-5}$
(amplitude 0.1); wave speed $\le0.6\%$ worst mode (gate 1\%); viscous decay matches $2\nu$ to
$10^{-10}$; 18 tests. Smooth-regime scope.

\subsection*{Four behaviours in the fluid + the gravity counterfactual}
5 seeds, $\sim$8.7k params. One-step competence: structured $7.8\times10^{-11}$; baseline in-distribution
median $0.28$ (per-seed $[0.28,428,0.17,6.1,0.22]$, mean dominated by one seed; disclosed)
and 0/5 bounded at 10k steps ($h<0$, $|u|\to2400$). counterfactual gravity: train $g=1.0$, test $1.5$:
$6.8\times10^{-11}$ vs in-dist $7.7\times10^{-11}$ (exact); non-separable control
$3.2\times10^{-5}$ in-dist $\to1.06$ extrapolated; baseline $2.5\times10^{3}$. Stability: 5/5 vs 0/5;
max $7.1\times10^{-5}$. Mass: $4.5\times10^{-16}$ vs baseline $1.09$. Metric note: absolute
nMSE primary (smooth, non-chaotic; oracle ratios inflate at $10^{-10}$ base). \decisive{}.

\subsection*{Fluid second law and the viscosity counterfactual}
Second law: decay-rate error $1.8\times10^{-8}$; no-port control error $1.0$ (port necessary);
violator flagged 5/5; wind injection factor $17.4=$ truth. Viscosity transfer
$\nu:0.5\to0.25$, error $2.9\times10^{-8}$; anti-friction clamps to exactly 0 (guarantee).
Diagnostic on fluctuation energy (constant $\tfrac12gh_0^2L$ background excluded, lesson).
\decisive{}.

\subsection*{Fluid discovery: functional content + unified port}
$M(h)$ slope $1.000$ ($R^2=1.000$); $V'(h)$ slope $1.000$; effective coefficient
$1.000\pm5\times10^{-6}$ per unit supplied $g$ (gravity supplied, not discovered).
Unified (condition-index only): $\sigma_{\rm inviscid}=1.4\times10^{-4}$,
$\sigma_{\rm viscous}=0.50004$; sign-mode control shows no spurious anti-dissipation
($\approx-6\times10^{-6}$). The reported agreement concerns the tested class and support, not a general identifiability or optimization guarantee.
\decisive{}.

\subsection*{The fluid refusal test, certified}
Pressure exponent $\gamma$ inside $h^\gamma$; trained $\gamma=2$ (in-dist $1.0\times10^{-10}$):
refused at $\gamma=1.7$ ($0.75$; $7.4\times10^{9}\times$) and $\gamma=2.3$ ($0.27$;
$2.7\times10^{9}\times$). Learnability certified by the $\gamma$-oracle (same class trained on
each refused law): $1.3\times10^{-6}/3.7\times10^{-7}$, $>5$ orders below the refusal.
Registration scope: direction-level pre-registration, with the oracle control added later (dated).
\decisive{}.

\subsection*{The $|q|$ boundary is grid-specific}
Cavity, $B_0$, 5 seeds, horizon 1000 ($5\times$ the grid's): $|q|=2$ speed error
$0.027\%@25\to1.9\%@1000$; energy-vs-truth $0.12\%$; model drift $0.57\%$; trajectory
decorrelates to nMSE $0.45$ (bounded, phase-only; comparable to in-envelope $|q|=1.5$).
Widening training to $|q|\in\{1,1.5\}$ at equal volume: no change ($1.00\times$), nothing to
fix. Grid-side mechanism (grid-scale coupling at speed) remains a hypothesis, stated.

\subsection*{The grid refusal test (directional)}
Smoothing width $n_{\rm smooth}:3\to5$; 3 seeds. Structured refusal $7.0\times$
in-distribution (pre-registered bar $10\times$: NOT met, \directional{}); oracle-kernel
control $0.34\times$ (learnable, certified); baseline broken ($2\times10^{8}\times$). Milder
than the cavity's $111\times$ (gentler generator perturbation); farther width untested.

\subsection*{Identifiability, and a corrected attribution}
Clamp/profile: retrain everything else at $\sigma_{\rm diss}\in\{0.25,0.4,0.5,0.6,0.75\}$,
fit loss $0.012/0.004/\mathbf{0.0004}/0.008/0.015$: clean minimum at truth; conservative energy
uncorrupted. $\lambda$-sweep: at $\lambda=0$, $\sigma=0.50005$; at the default,
$0.4997$, so L1 contributes $-0.08$pp; the unified model's $\sim4\%$ bias is protocol-specific (mechanism
open); its caveat corrected (dated). \decisive{}.

\subsection*{The stability tail characterized, and an attribution refuted}
Diagnostic across 400 episodes: closest-approach vs worst-case deviation $r=+0.13$ (predicted
$\le-0.5$: wrong sign); attracting pairs reaching separation $\approx0$: $r=-0.06$; strongest
correlate: initial kinetic energy ($+0.24$). Lever: $67\times$ close-encounter oversampling
(fraction $0.087$ vs natural $0.0013$): worst-case $0.124\pm0.079$ vs $0.091\pm0.030$, no
closure; no fit/counterfactual regression. Both pre-registered predictions resolved against the
hypothesis; the attribution corrected program-wide with dated markers. The tail is a
high-initial-energy content boundary, carried.

\subsection*{Scaling probe 1}
$100\times$ FNO ladder (10k/100k/1M), certified 3-LR/20-epoch protocol. No rung competent
(rollouts diverge everywhere; one-step MSE worsens up the ladder under fixed epochs, so
the pre-registered fit-improves prediction failed); flip 6/10--2/10--4/10 (uninformative:
powerless metric on broken models). Self-retracted headline; superseded by the
competence-first design. Lesson: competence is the power condition.

\subsection*{Label-free regime identification}
Frozen unified core; self-supervised $\hat\sigma$ encoder (100-step window; loss = short-rollout
match; no rate labels). Separation: all held-out conservative $\hat\sigma$ ($0.070\pm0.008$)
below all dissipative ($0.449\pm0.008$), 5/5 seeds, $z=7.2$. Mid-rollout switch re-identified
($0.078\to0.457$). the second-law test with $\hat\sigma$: $9.8\%$ mean (per-seed $[9.7,11.3,8.3,9.5,10.0]\%$,
disclosed straddle). Boundaries: intermediate tracking $r=0.90$ (gate $0.95$),
\directional{}; $\hat\sigma_{\rm cons}$ floor $\approx0.07$ (finite-window limit), so
conservative rollouts degrade ($0.60$ vs $0.005$); controls reduce but don't zero the signal
(K=5: $z=1.3$; shuffled: sign-cancelling, one seed retains $z=3.6$). Gate text generated from
the pre-registered constants (an earlier misprint corrected, on record).

\subsection*{The perception boundary mapped}
\label{app:perception}
Lossy 1-D rendering (position blob + field traces; velocity never rendered); corruption sweep;
best-case estimator per level; frozen core. Single frame: position $0.034$ rad; $v_x$ $0.165$
(floor $0.290$), $v_y$ $0.242$ (floor $0.261$, at floor); energy-shell miss $8.8\%$, kinetic
($0.082$) not field ($0.025$). Three strided frames: $v_x$ $0.052$, $v_y$ $0.043$; shell
$3.5\%$; continuation $22.8\to4.5$ (oracle $0.013$). Field-trace ablation: $v_x$
$0.165\to0.297$ (to floor), so the single-frame signal is the field imprint. Per-component
floors are reported here; an earlier combined floor of 0.391 overstated the margins and is corrected.
\boundarylbl{}, measured; from-pixels binding open.

\subsection*{Fresh-data replication, with one gate that fails}
Three fresh generation seeds, disjoint. Flip: 15/15 seeds (pooled 479/480; controls 5/15,
1/15). $\sigma$: orders preserved on every draw ($\sigma_{\rm diss}=0.4988$, $0.2\%$;
corroborates the identifiability result); strict per-value gate $\sigma_{\rm cons}<10^{-4}$ FAILS marginally (7/9;
max $1.07\times10^{-4}$); fresh inference $\sim7\times$ looser than shared-data (worst-seed
implied 50k loss $5.2\%$, over the 5\% bound). An intermediate draft relabeled the recorded FAIL
as a bug; reversed, the boolean was right (ledger). Boundary kept.

\subsection*{Scaling probe 2: does model capacity buy the counterfactual?}
\label{app:scaling-probe}

\begin{figure}[t]
  \centering
  \includegraphics[width=\linewidth]{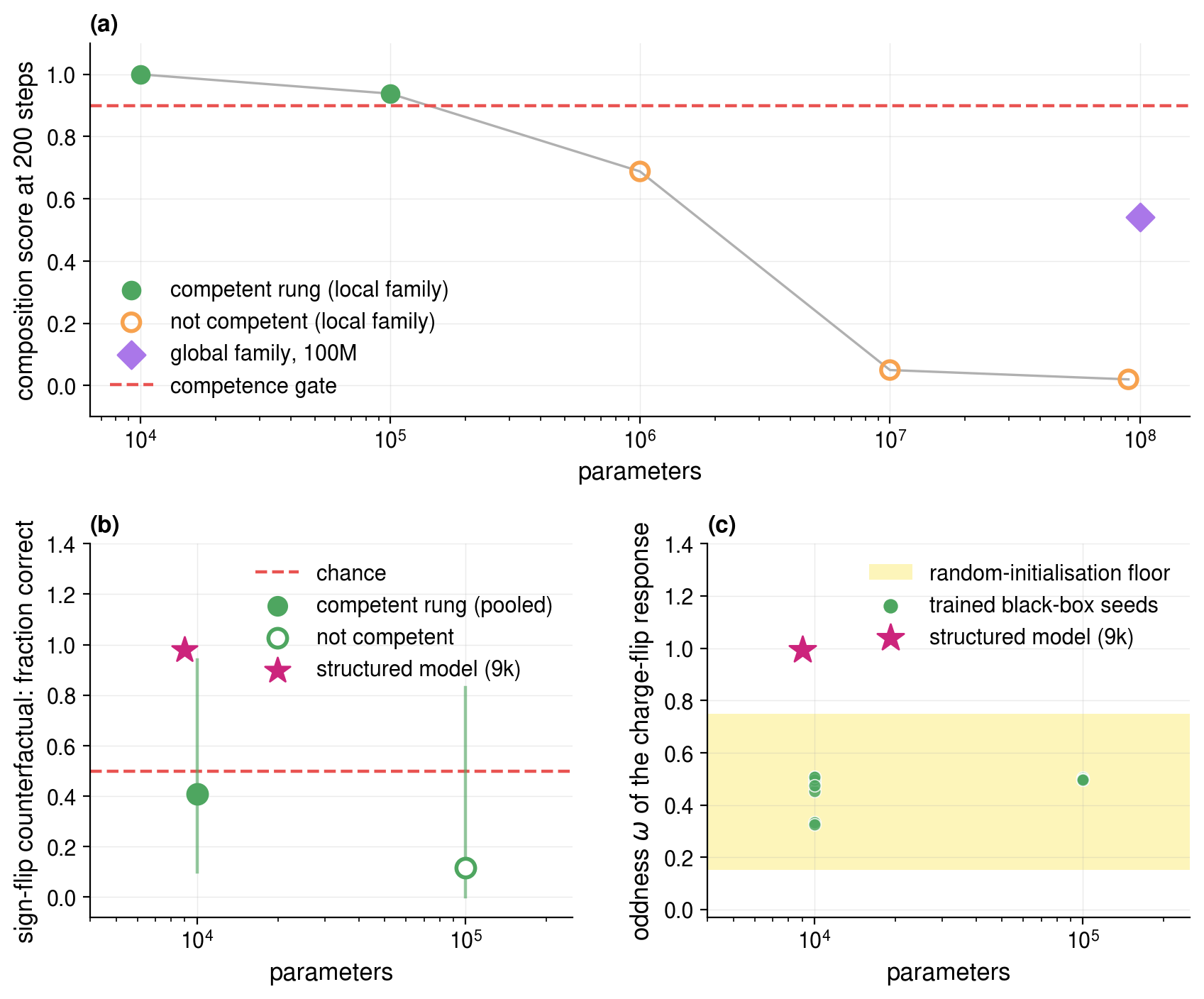}
  \caption{Scale does not replace structure (supporting study). (a) Composition score
    against parameter count for the local convolutional family (filled green: competent;
    open orange: not) and a global family (purple diamond), with the competence gate
    (dashed). (b) Fraction of correct sign-flip counterfactuals at the two competent rungs,
    with chance (dashed) and the small structured model (magenta star). (c) Oddness
    $\omega$ of the charge-flip response for trained black-box seeds (green) and the
    structured model (star), against the random-initialisation band (yellow). The reading
    is given in this appendix.}
  \label{fig:scale-scissors}
\end{figure}

Four families $\times$ five rungs ($10^4$--$10^8$), plateau training, a pre-registered three-branch
verdict, MIN\_COMPETENT $=2$. 100M-first gate: FNO/100M composition $0.542$ (decorrelated;
worse than freezing the state at $0.562$); CNN/100M intractable ($\sim$1.7h/epoch), unstable;
10k local anchor composition $1.000$ $\to$ competence is bought by locality; the study
then pivoted to the tractable local ladder (dated amendment). First test, composition
within the CNN family: $1.000\to0.938\to0.688\to$ divergence; an unlimited-data probe at 1M
improves the fit ($0.014\to0.004$) but not composition ($0.646$), a scale or optimisation
effect rather than data starvation (compute cap disclosed). Second test (3-seed): both
competent rungs below chance
($46/144$, $0/144$), giving a provisional pass reported as ``pending escalation.'' \textbf{Escalation
(10 seeds, fresh data): OVERTURNED}, since 100k trained-sign reliability $0.60$--$1.00$ (mean
$0.867<0.90$; drops out); 10k flip pooled $196/480=0.41$ (``intermediate''); verdict
\textsc{inconclusive-single-anchor}; guard stays up. Structured reference, same harness:
$141/144=0.979$. Incompetent-rung artifacts handled by the power gate (a $10^8$ FNO at $0.76$
pooled with reliability $0.53$ is a powerless metric, excluded by design). Directional signal
kept: competent-seed flips $0.00$--$0.17$ (lock-in); competence-by-locality; erosion with
scale. The main text's scaling claim is exactly what survives.

\subsection*{Odd/even response decomposition}
Evaluation-only on the twenty trained CNN flip seeds. Feeding identical states at $\pm q$ (charge
feature flipped, fields and kinematics held) and splitting the particle velocity update into odd and
even parts in $q$: the oddness fraction is $\omega=0.33$--$0.51$ across black-box seeds against
$0.994$ for the 9k structured model (metric validated). Per-seed $\omega$ does not predict per-seed
flip rate, Spearman $\rho=-0.25$ ($p=0.29$) over twenty seeds; anchor-only $\rho=0.57$ ($p=0.08$,
not significant). Pre-registered outcome (a): the three high-flip anchor seeds are sampling
variation, not an implicit odd-in-$q$ bias. A random-init control shows
$\omega\approx0.5$ is the no-preferred-parity floor (an untrained CNN gives mean $\omega=0.45$,
seeds $0.03$--$0.94$) and trained CNNs sit at it ($10^4$ $0.45$, $10^5$ $0.50$), not near
$0.994$: training acquires \emph{none} of the symmetry, it only concentrates $\omega$ at the floor.
$\omega$ is thus a decisive response-level separator (structured $0.994$ vs every \emph{trained}
black box $0.33$--$0.51$, no overlap) where the rollout-based scaling study is inconclusive.
This closes the open question flagged in \S\ref{app:scaling-probe} and supplies the
mechanism-level half of the evidence that scale does not substitute for structure.

\subsection*{Quartic fit-gap experiment}
The one-degree-of-freedom quadratic-plus-bounded implementation is fitted
to $V(x)=\tfrac12x^2+\tfrac14x^4$ and compared with a quadratic control.
Its reported fit error is $2.7\times10^{-3}$ versus $9.0\times10^{-8}$,
and fitted curvature at $x=2$ is $-4.5$ versus the reference $13.0$.
The independent $\beta$ sweep reports fit-floor ratios
$1,16.8,150,2040,4.6\times10^4,1.8\times10^5,9.1\times10^5$
for $\beta=0,0.1,0.25,0.5,1,2,4$.
These are empirical fitting failures. Bounded corrections can match quartic
behaviour on compact domains, so the result does not prove that the
admissible class cannot represent the observed truth. Coercivity also
does not forbid negative local curvature. The relative contributions of
capacity, correction bounds, and optimization remain unresolved.

\subsection*{Discovery rung 2: selecting the general form}
Candidate forms $\{H,\ H{+}\text{port},\ H{+}\text{port}{+}\text{drive}\}$ in the lumped $1$-DoF
harness (Hamiltonian core, non-negative dissipation port, optional explicit drive; learned inertia,
stiffness, $\sigma$), scored by a pre-specified selection statistic, $\mathrm{BIC}=N\ln\mathrm{MSE}+p\ln N$.
On each generating world the minimum-BIC candidate is the true form. Hardened protocol (amendment of 2026-07-31):
eight worlds on a $\sigma\!\to\!0$ and drive$\to\!0$ ladder, five seeds, $1\%$ noise, $40/40$
correct; a non-nested wrong form (dissipation on the wrong variable) never selected; correct-form
one-step error separates from the next-simpler candidate $1.1$--$3.1\times$ (penalty-assisted only at
$\sigma\!\lesssim\!0.05$, a characterized near-degenerate boundary). The general model's learned
diagnostics agree ($\hat\sigma=0.305$ vs true $0.30$; the drive switches on only when driven). The
same harness recovers the lumped constants to $1$--$3\%$ and does label-free dissipation discovery
($\hat\sigma=0.002$ conservative vs $0.252$ dissipative), the \S\ref{sec:discovery} capability in the
one-degree-of-freedom form. Scope: synthetic worlds with the truth inside the candidate set, so the
claim is that the statistic recovers the correct form when it is present; $\sigma$ is weakly
identified under a strong drive and near the conservative boundary, disclosed.

\emph{Flagship lift, with its five-seed escalation.} The identical BIC statistic
was applied to the paper's central system, the $16$-mode, two-particle EM cavity
(\S\ref{sec:systems}, f64 symplectic core), with the $8.9$k-parameter energy network
\emph{shared} across all three candidate forms, so BIC penalizes only the structural hypothesis,
not capacity. Three worlds (conservative/dissipative/driven), $400$ epochs, per-dimension
standardized one-step fit. At three seeds: $3/3$, every runner-up $\ge278$ BIC worse. The
pre-committed five-seed escalation \emph{softened} the conservative cell: dissipative and driven
remain decisive and seed-stable ($H{+}$port by $\Delta$BIC $2257$; $H{+}$port${+}$drive by
$1054$), but the conservative world flips to $H{+}$port by $29$ BIC, with
$\hat\sigma=0.0126$; the underlying fit difference is
$1\%$ of wMSE, swung by one seed. The pre-registered five-seed verdict is $2/3$, a boundary case, superseding the
three-seed $3/3$; ledgered in Appendix~\ref{app:integrity}. The port form recovers the
resistivity it was not told on every world: $\hat\sigma=0.013$ conservative ($\approx\!0$),
$0.502$ dissipative, $0.502$ driven (true $0.5$). Two reproducibility notes: the fit metric must
be the per-dimension standardized (1/std-residual) weighted MSE, since the port's per-step signal is
$\sim\!0.1\%$ of a state's residual, so the raw full-state MSE is diluted and selects the simplest
form regardless; and the conservative $H$-vs-$H{+}$port tie-break is degenerate at the fit but
\emph{not} at evaluation horizons: $\hat\sigma=0.0126$ sheds $6\%$ of the field energy over the
$500$-step window yet $99.8\%$ over the $50$k-step stability horizon ($E\propto e^{-2\hat\sigma t}$),
so the flipped selection carries drastically wrong long-horizon physics, the near-degeneracy
characterized on the $1$-DoF ladder, now demonstrated to flip a selection on the distributed
system, and the reason removing this residual is registered as a follow-up study.

\subsection*{The selection loop and its guard}
No new runs: the min-BIC form selector presupposes the true form is in its candidate set; the misspecification alarm's
reference-free detector (structured-vs-capacity-matched-unstructured floor, inverting at
$\beta\ge0.25$ with no ground truth) is the fallible diagnostic for that presupposition. Composed
(\S\ref{sec:discovery}, Fig.~\ref{iclr:fig-loop}): propose forms $\to$ fit under one protocol
$\to$ guard (alarm $\Rightarrow$ investigate fit failure and candidate adequacy) $\to$ select $\to$ recover the
form and rates. Every edge is backed by a recorded run: the selector at $40/40$ on the $1$-DoF
stress ladder, and on the flagship cavity decisive wherever dissipative/driven structure exists
($\Delta$BIC $\ge1054$ at five seeds; the conservative tie-break's five-seed flip, at $2/3$,
is the disclosed boundary, superseding the earlier three-seed $3/3$). Scope: this is a first
version on synthetic worlds with small candidate sets, and loop closure that auto-proposes new
forms is not yet built.

\subsection*{Energy-auxiliary unification: a pre-registered negative result}
The misfit-residual hypothesis, that the residual
conservative $\hat\sigma$ absorbs one-step objective error, curable by the unified model's energy-trajectory
auxiliary (which reached $\hat\sigma_{\rm cons}=1.27\times10^{-5}$ on this same cavity), was
pre-registered and \emph{failed under its pre-registered thresholds}: with the
exact energy meter and the unified model's window protocol, $\hat\sigma_{\rm cons}$ moved $0.0126\to0.0107$
($1.2\times$ against a $10\times$ bar), every fit degraded, and the conservative flip widened to
$750$ BIC. The unified model's L1-on-$\sigma$ was deliberately omitted (it would confound the test); how its
$10^{-5}$ decomposes among its levers (L1, energy aux, joint multi-world training) is now an open
attribution question, stated as such. The one-step $\sigma$ floor is an objective-level
identifiability boundary. It suggests a \emph{trusted-meter principle}: an energy supervisor
must be exact or computed from grounded data only, since a learned functional supervising its
own trajectories opens a self-consistency escape instead of constraining the physics.

\subsection*{Decomposition of the training procedure}
The L1 arm ($\lambda=10^{-5}$, the unified model's
value): selection \emph{restored} to $3/3$ (conservative $\to H$ by $280$ BIC; the paired
  profile-likelihood clamp tie-break, per-seed $H$-vs-$H{+}$port wMSE differences,
$-2.0\times10^{-3}\pm4.4\times10^{-3}$, CI includes $0$, independently selects $H$), with the
identifiable rate unbiased ($\hat\sigma_{\rm diss}=0.502$; identifiability reconfirmed in-harness). But
$\hat\sigma_{\rm cons}=0.0126$ is unmoved, failing the pre-registered eval-horizon gate
($\le10^{-4}$): the L1 stops the port from converting this residual into a \emph{fit advantage}
without removing the residual itself. The pre-registered verdict for this arm is FAIL on that gate, since selection integrity
is repaired while long-horizon negligibility-when-selected is still open. The second
arm (L1 $+$ aux at the lowered $W=0.01$) is strictly worse, since every fit degrades $3$--$4\times$
and even the dissipative selection misses ($2/3$) with $\hat\sigma_{\rm cons}$ still unmoved,
confirming the energy-aux negative at a tenth of the weight.

\subsection*{The residual near-zero rate}
(a) The pre-registered
$\sigma^{*}$ prediction, the static regression of the trained $H$ core's one-step error onto
the port decay direction, gives $4.1\times10^{-4}$, $30\times$ below the $400$-epoch residual rate:
this was not predicted, because the residual rate at protocol budget is not a property of the converged core's residual
field. (b) The misfit ladder \emph{refutes the infeasibility framing we expected}:
$\hat\sigma_{\rm cons}$ drains as ${\sim}1/\text{epochs}$ ($0.069/0.0126/0.0035/0.00026$ at
$40/400/1600/6400$; log-log slope $-1.055$), reaching the $10^{-4}$ eval-horizon bound at an
extrapolated ${\sim}2.7\times10^{4}$ epochs, feasible, a transient rather than a wall. Post
hoc synthesis, flagged: the asymptote matches $\sigma^{*}$'s magnitude, consistent with the
residual being a \emph{co-training transient} that relaxes toward the converged core's
static-regression floor, explaining at once why the energy aux could not drain it (wrong
lever), why L1 restores selection without draining (it suppresses the transient's fit advantage,
not the transient), and why training alone eventually does. This also dissolves the unified model's
attribution question: no special lever is needed, only budget.

\subsection*{Removing the residual rate with longer training}
A pre-registered continuous run to the budgeted horizon:
$\hat\sigma_{\rm cons}(27{,}500)=5.1/4.1\times10^{-6}$ (2 seeds; bar $10^{-4}$, P-cap-1
CONFIRMED, the crossing at $8.0$--$8.5$k epochs, so the residual-rate price was $\sim3\times$ conservative), and
the measured decay steepens to $\approx t^{-2.5}$ (slopes $-2.49/-2.55$, outside the
  pre-registered $1/t$ band, P-cap-2 NOT CONFIRMED as pre-specified; the ladder's early points were not
asymptotic and its exponent is retired). The mechanism's floor signature held cleanly:
$\sigma^{*}(t)$ stays at $|\sigma^{*}|\le6\times10^{-5}$, essentially zero, from epoch
$5{,}000$ onward (one seed's $2{,}500$ checkpoint reads $1.64\times10^{-4}$, still $12\times$
under the pre-registered gate) while $\hat\sigma(t)$ relaxes monotonically down to it, terminating
\emph{below} the unified model's $1.27\times10^{-5}$: the question is closed empirically. Nothing about
the conservative cell remains open.

\subsection*{Wave 1, the ODE zoo (8 systems, full pass cards)}
Pre-registered card variants with the pre-registered constants; every verdict printed by the harness.
The pendulum PASS (one-step competence $2.2\times10^{5}\times$; worst energy bound $0.93\%$; $\hat g$ to $0.3\%$;
  geometry refusal test refuses at $1.5\times10^{6}\times$ with a hand-set truth anchor at
$0.03\times$ base). The double pendulum PASS\textsuperscript{b} (climate 3/3, counterfactual climate 10/10 seeds; one-step-competence sprint
$13.4\times$ under the $10^{3}$ gate, carried). The RLC circuit PASS\textsuperscript{b} ($\hat\sigma=0.2999$ vs
$0.300$; $\hat\omega^{2}=1.0004$; the counterfactual test $6.5\times$ carried). The drag projectile PASS (clean twice; violator
prints DIVERGED and is flagged a fortiori). The spinning top PASS under the gauge-free gates (one-step competence
$5\times10^{11}\times$; Casimir machine-exact; $\kappa$-differences in-tol; uniform-scale counterfactual).
Lorenz-63 PASS\textsuperscript{b} (climate 3/3, the counterfactual test 10/10; one-step competence $0.38\times$ carried, the split buys the long
game). Dry friction PASS\textsuperscript{b} (guard $3.5\times10^{4}\times$; $\hat\mu N$ to $0.012\%$; material
  refusal test passes by truth anchor; counterfactual carried on a stick-degenerate target world, mechanism
recorded). The DC motor PASS (power balance $5.6\times10^{-17}$ architectural; saturation refusal test).

\subsection*{Wave 2, 1D fields (4 systems; exact-propagator references)}
All references calibrate at machine precision against analytic mode solutions
($10^{-16}$-level; mass drift $4\times10^{-14}$), so gate failures are the model's, never the
integrator's. The spring lattice, the heat rod, advection--diffusion, and the plucked string all PASS. Highlights: the heat rod's flux-form refusal test refuses at
$1.6\times10^{7}\times$ with the bond-form mobility oracle separating at $7600\times$; advection--diffusion's
mass conservation is architectural ($1.8\times10^{-14}$ on model rollouts) and its two
instrument bugs (a sign convention the model exposed by fitting $-w$ exactly; a projection
instrument that conflates transport with decay, replaced by the FFT magnitude) are on record;
the plucked string's locality refusal test is certified by the curvature-bond class that represents the
biharmonic exactly.

\subsection*{The instrument findings}
(1) the drive-time bug (batch-mean drive times poisoned every driven system while sparing the
undriven ones); (2) the Lie--Poisson $\kappa$ gauge (Casimir multiples are flow-null; measured
as $5\times10^{11}$ one-step competence beside a $6\times10^{8}$ axis-wise counterfactual failure); (3) certified floor-fitting
procedure (a hand-set truth floor at $2.8\times10^{-10}$ vs a fixed-LR stall at
$1.4\times10^{-2}$); (4) the amendment-5 refusal instrument (within-class refusal + oracle
separation or truth anchor; common-base ratios retired as truncation-confounded); (5) oracle
classes must provably contain their refusal-test worlds (pointwise vs flux-form mobility;
pad-2 stencils vs curvature bonds); (6) refusal-test worlds can be degenerate ($\mu_s$ at the
drive amplitude sticks forever); (7) violator meters need a divergence branch; (8) event-time
quantization bounds every instrument near events (calibration, model substeps, and the
one-step floors of \emph{all} classes including alarm references); (9) reference-grade
thermostats are part of the instrument (BAOAB; Boris rotations for workless magnetic forces);
(10) architectural guarantees relocate violators from the class to the world ($\hat e<1$,
$\sigma\ge0$).

\subsection*{The four exceptions}
bouncing contact CLOSED (guard-first fires vs the soft-wall base; contact recovered to machine precision:
$\hat g=1-3\times10^{-14}$, $\hat e=0.8-3\times10^{-12}$, oracle floor $10^{-24}$). The interacting thermal bath
characterized boundary (all discoveries and both FDT meters pass, incl.\ the sector-asymmetric
  violator at $66\%$ and the magnetic sector via Boris; stability and equipartition-B fail through the
single quantified $\hat T$-bias chain). The shallow-water dam break closed on the characterized-limit arm (guards
  refuse $36$--$40\times$; the localized port finds the fronts, since localization passes with rich
features at the measured cost of the $\hat g$ discovery gauge; coarse-stencil limit measured).
Bare gravity: the softening ladder is the deliverable, using an adaptive reference with
per-rung numerical calibration.

\subsection*{The selection matrix}
One menu across every world; the frozen selector in its final form (structural-channel BIC + L1 +
drained-channel collapse + CI-window clamp tie-break + the budget flag). Every
declared bracket variant is selected correctly (Lie--Poisson, metriplectic, skew-transport,
port-Hamiltonian); both out-of-menu controls fire; bare gravity's correct-but-undrained row self-flags.
The pre-registered $\ge14/15$ gate is \emph{missed} at $12/15$ registry systems: the three misses
(dry friction, bouncing contact, the shallow-water dam break) share one mechanism, since floor-ratio alarms require dense misfit and a competent
reference, and impulsive/shocked worlds offer neither; their misspecification was caught by
the guard-first protocol at card level instead. The division of labor (alarms for dense
misfit, matched-structure guards for events, budget flags for fake-conservative calls) is the
matrix's methodological yield, and the three selector passes with their per-pass results are
all in the record.

\section{Extended related work and the novelty audit}
\label{app:related}

\begin{table}[h]
  \centering\footnotesize\setlength{\tabcolsep}{3pt}
  \caption{Closest structured methods: what each is supplied, what it learns, what changes at
    test time, and what is actually guaranteed. The present work's advance is the explicit
    supplied-versus-learned ledger with estimation error carried into downstream predictions,
    not any ingredient individually.}
  \label{tab:related-compare}
  \begin{tabular}{@{}p{0.19\linewidth}p{0.2\linewidth}p{0.19\linewidth}p{0.17\linewidth}p{0.19\linewidth}@{}}
    \toprule
    Method & Supplied & Learned & Changes at test & Guaranteed \\
    \midrule
    HNN / SymODEN \citep{greydanus2019hnn,zhong2020symoden} & symplectic form & $H$ & evaluated inputs & Hamiltonian-form continuous dynamics; solver-dependent discrete behaviour \\
    D-HNN \citep{sosanya2022dhnn} & $H$ + Rayleigh split & $H$, dissipation & friction coefficient (evaluated) & none on the sign \\
    Port-HNN \citep{desai2021phnn} & port-Hamiltonian form & $H$, forcing, dissipation & inputs & passivity of the port \\
    Compositional pH \citep{neary2023compositional} & pH interconnection & subsystem $H$, ports & composition & pH structure \\
    GFINNs \citep{lee2021gfinn} & GENERIC with degeneracy & $L,M,E,S$ & none & degeneracy laws \\
    Comparative evaluation \citep{zhong2021benchmark} & various & various & none & benchmark only \\
    This work & discretization, field kernels, coupling channel, confining class & constitutive functions, fitted rates, inferred context & factored parameter values & quadratic-energy dissipation; marginal conformal coverage \\
    \bottomrule
  \end{tabular}
\end{table}

This appendix records the extended related-work comparison, so the main text's differentiator
claims are checkable line by line.

\subsection{The nearest neighbors, adversarially read}
\textbf{Dissipative HNNs (D-HNN)} \citep{sosanya2022dhnn} learn an additive
Hamiltonian-plus-Rayleigh decomposition and can, on the system they are fit to,
exhibit a small dissipative component on conservative data. This is the
phenomenon-level precedent closest to our $\sigma\to0$ result, which we discuss
openly here. The differences that matter are as follows. D-HNN fits one system at a
time, with no cross-condition discovery on shared content, no condition-index
protocol, and no label-free variant. Its dissipative head is unconstrained in sign,
which means it has no second-law guarantee. This is precisely the degree of freedom
our sign-free control shows drifting negative on ambiguous data. D-HNN does
evaluate prediction at unseen friction coefficients, so it is not correct to say it
lacks counterfactual evaluation. What it does not include is the long-horizon,
violator-control, and label-free context-inference program, together with the
explicit supplied-versus-learned accounting and downstream error propagation that
accompany the decomposition here (Table~\ref{tab:related-compare}).

\textbf{Dissipative SymODEN} \citep{zhong2020dissipative} and \textbf{port-Hamiltonian NNs}
\citep{desai2021phnn} learn dissipation and control terms with physically-motivated structure
(PSD damping and port interconnection, respectively). Both share the same
single-system scope as D-HNN: they perform identification rather than discovery,
include no counterfactual criterion, and run no measurement program.

\textbf{GENERIC/metriplectic learning} \citep{hernandez2021,lee2021gfinn,gruber2023} imposes
the same reversible/irreversible skeleton we use, with degeneracy conditions enforced
architecturally, the closest structural ancestry for our port. These works establish relevant structure-preserving learning methods.
Our comparison concerns the combination of shared-content rate estimation,
supplied-versus-learned accounting, and recovery-linked prediction tests,
not absence of individual capabilities from all earlier work.

\subsection{Adjacent lines}
Learned simulators (GNS, MeshGraphNets, FNO) supply our scale frontier and the locality lesson;
system identification (SINDy and variants) supply the clean-data competitor we position
against; DeLaN and pH system identification supply the robotics-side
friction-learning precedent (single-system, torque-supervised); and
multi-environment/context adaptation lines (e.g.\ CoDA-style) share the ``shared content,
per-condition parameter'' protocol shape without physical structure or behavioural targets.

\subsection{The central claim, restated}
The paper has one central claim, the one its main text is organized around: \emph{within
explicitly supplied physical representations, data recover useful constitutive functions
and dissipation parameters, and the recovered content, its identifiability limits, and its
estimation error are carried into the downstream predictions that use it}. The
cross-condition dissipation decision, the quadratic-energy dissipation guarantee, and the
behavioural program are supporting evidence for that claim, not a second contribution of
record; the guarantee is scoped to the quadratic diagnostic energy
(Prop.~\ref{prop:discrete2law}), and the label-free inference carries its measured rate
floor into every prediction it feeds.

\end{document}